\documentclass[11pt,letterpaper]{article}

\usepackage[margin=1in]{geometry}
\usepackage{cite}
\usepackage{hyperref}
\usepackage{colortbl}
\usepackage{paralist}
\usepackage{pdfpages}
\usepackage{enumitem}
\usepackage{float}
\usepackage{booktabs}
\usepackage[override]{cmtt}
\usepackage{color,xcolor}
\usepackage{graphicx,eso-pic}
\usepackage{boxedminipage}
\usepackage{url}
\usepackage{caption}
\usepackage{subcaption}
\usepackage{xspace}
\usepackage{multirow}
\usepackage{amsmath,amsthm,amstext,amssymb,amsfonts,latexsym}
\usepackage{wrapfig}
\usepackage[linesnumbered,ruled,vlined]{algorithm2e}
\usepackage[noend]{algpseudocode}
\usepackage{algorithmicx}
\usepackage{epstopdf}
\usepackage{mdframed}
\usepackage{cases}
\usepackage[capitalize]{cleveref}
\usepackage[compact]{titlesec}

\usepackage{diagbox}

\newcounter{note}[section]

\newcommand{\yuetodo}[1]{{\large\color{green}[Yue todo: #1]}}

\definecolor{yue}{rgb}{0.7, 0, 0}

\renewcommand{\yuetodo}[1]{}

\newcommand{\mcal}[1]{\ensuremath{\mathcal {#1}}}

\definecolor{darkgreen}{rgb}{0,0.5,0}
\definecolor{lightblue}{RGB}{0,176,240}
\definecolor{darkblue}{RGB}{0,112,192}
\definecolor{lightpurple}{RGB}{124, 66, 168}
\definecolor{grey}{RGB}{139, 137, 137}
\definecolor{maroon}{RGB}{178, 34, 34}
\definecolor{green}{RGB}{34, 139, 34}
\definecolor{types}{RGB}{72, 61, 139}
\definecolor{gold}{rgb}{0.8, 0.33, 0.0}

\definecolor{darkgray}{gray}{0.3}

\definecolor{darkred}{rgb}{0.5, 0, 0}
\definecolor{darkgreen}{rgb}{0, 0.5, 0}
\definecolor{darkblue}{rgb}{0,0,0.5}

\newcommand\markx[2]{}

\newcommand{\R}{\mathbb{R}}
\newcommand{\I}{\mathbb{I}}

\newcommand{\ignore}[1]{}

\newcounter{task}

\newtheorem{theorem}{Theorem}[section]

\newtheorem{corollary}[theorem]{Corollary}
\newtheorem{fact}[theorem]{Fact}
\newtheorem{lemma}[theorem]{Lemma}
\newtheorem{assumption}[theorem]{Assumption}

{
\theoremstyle{definition}
\newtheorem{definition}[theorem]{Definition}
\newtheorem{remark}[theorem]{Remark}
}

\newcounter{cnt:challenge}

\newcommand*\samethanks[1][\value{footnote}]{\footnotemark[#1]}

\newcommand{\renyi}{R\'{e}nyi\xspace}
\newcommand{\D}{\mathsf{D}}
\newcommand{\Dz}{\mathsf{D^z}}
\newcommand{\RDP}{\ensuremath{\mathsf{RDP}}}
\newcommand{\zCDP}{\ensuremath{\mathsf{zCDP}}}
\newcommand{\E}{\mathbf{E}}

\newcommand{\supp}{\ensuremath{\mathsf{supp}}}
\newcommand{\W}{\ensuremath{\mathsf{W}}}

\newcommand{\M}{\mathcal{M}}

\newcommand{\ts}{\textstyle}

\newcommand{\f}{\mathfrak{f}}

\SetKwComment{Comment}{//}{}

\begin{document}
\title{Privacy Amplification by Iteration for ADMM with (Strongly) Convex Objective Functions}

\author{T-H. Hubert Chan\thanks{Department of Computer Science, the University of Hong Kong.} \and Hao Xie\samethanks \and Mengshi Zhao\samethanks}

\date{}

\begin{titlepage}

\maketitle

\begin{abstract}
We examine a private ADMM variant for (strongly) convex objectives which is a primal-dual iterative method. Each iteration has a user with a private function used to update the primal variable, masked by Gaussian noise for local privacy, without directly adding noise to the dual variable. Privacy amplification by iteration explores if noises from later iterations can enhance the privacy guarantee when releasing final variables after the last iteration.

Cyffers et al. [ICML 2023] explored privacy amplification by iteration for the proximal ADMM variant, where a user's entire private function is accessed and noise is added to the primal variable. In contrast, we examine a private ADMM variant requiring just one gradient access to a user's function, but both primal and dual variables must be passed between successive iterations.

To apply Balle et al.'s [NeurIPS 2019] coupling framework to the gradient ADMM variant, we tackle technical challenges with novel ideas. First, we address the non-expansive mapping issue in ADMM iterations by using a customized norm. Second, because the dual variables are not masked with any noise directly, their privacy guarantees are achieved by treating two consecutive noisy ADMM iterations as a Markov operator.

Our main result is that the privacy guarantee for the gradient ADMM variant can be amplified proportionally to the number of iterations. For strongly convex objective functions, this amplification exponentially increases with the number of iterations. These amplification results align with the previously studied special case of stochastic gradient descent.

\end{abstract}

\thispagestyle{empty}
\end{titlepage}

\section{Introduction}
\label{sec:intro}

\emph{Alternating direction method of multipliers}~\cite{Gabay1976ADA} (ADMM) has been
 designed
for convex programs whose objective functions
can be decomposed as the sum $\f(x) + g(y)$ of two convex 
functions\footnote{It will be clear soon
why we use a different font for $\f$.},
where the \emph{primal} variables $x$ and $y$ are restricted
by some linear constraint $Ax + By = c$.

Decomposing the objective function into the sum of two convex functions offers several advantages. Firstly, different optimization algorithms can be applied to each part of the function. Secondly, in a distributed learning setting, the function $\f $ refers to the \emph{loss functions} from various users which can be optimized in parallel, while the function $g$ refers to a \emph{regularizer} term that can typically be optimized by a central server.

In ADMM, a \emph{dual} variable $\lambda$ keeps track
of how much the linear constraint is violated.
The method is an iterative procedure that minimizes some
\emph{Lagrangian} function $\mcal{L}(x, y, \lambda)$, which 
is defined in terms\ of $\f$ and $g$.  In each iteration, the three variables $x$, $y$ and $\lambda$ are updated in sequential order.
While the primal variables $x$ and $y$ are each updated
(while keeping other variables constant) to minimize $\mcal{L}$,
the dual variable $\lambda$ is updated to encourage the feasibility
of the linear constraint.
The alternating nature of variable updates makes
the method widely adaptable in large-scale
distributed contexts~\cite{DBLP:journals/ftml/BoydPCPE11}.

In this paper, we focus on a \emph{stochastic} version of ADMM proposed by~\cite{DBLP:conf/icml/OuyangHTG13}, in which the function $\f$
can be viewed as an expectation of functions sampled from
some distribution $\mcal{D}$. 
Each iteration~$t$ is associated with some user,
whose (private) data is some function $f_t$ sampled from $\mcal{D}$.
Instead of directly accessing $\f$, each iteration~$t$ only has access
to the corresponding user's function $f_t$.
While the sequence $f_1, f_2, \ldots, f_T$ of
sampled functions arises from user data,
the function $g$ is publicly known.
In one ADMM iteration, $f_t$ is only needed for updating the $x$ variable.
In the \emph{proximal} variant, the whole function $f_t$ is used
in some optimization step to update~$x$.
Instead, we will focus on the more computationally efficient
\emph{gradient} variant that uses the first order
approximation of $f_t$~\cite{DBLP:journals/siamis/OuyangCLP15, DBLP:journals/jscic/LiL19}, where
only one access to the gradient oracle $\nabla f_t(\cdot)$ is sufficient.

All variants of \emph{differential privacy}~\cite{DBLP:conf/icalp/Dwork06,BunS16,Miro17} are based on the principle that a mechanism or procedure achieves its privacy guarantee through the incorporation of randomness.
\ignore{
are based on the premise that the privacy guarantee of a mechanism or procedure
is achieved via randomness.  
Specifically, given two input scenarios that are ``close'',
the corresponding two distributions of outputs (that 
can be observed by the \emph{adversary}) are also ``close''.  Each variant of differential privacy formally defines a symmetric binary \emph{neighboring} relation
to capture close inputs and uses some notion of \emph{divergence} to quantify
the closeness of output distributions,
where a smaller value for the divergence means stronger privacy.  
}
Private variants
of ADMM have been considered by adding noises to the variables~\cite{DBLP:conf/ccs/ZhangZ16}.
To apply this privacy framework to ADMM, the function~$f_t$ 
is considered as the private input of the user in iteration~$t$.
Hence, one possible way~\cite{DBLP:journals/tifs/ShangXLLSG21} to achieve \emph{local privacy} (against an adversary that can observe the variables after each iteration)
for the user~$t$ is to sample some noise~$N_t$, which
is added to the result of the gradient oracle oracle $\nabla f_t(\cdot)$.
A popular choice for sampling $N_t$ is Gaussian noise, for which the \renyi~$\D_\alpha$ and 
zero-concentrated~$\Dz$ divergences (formally explained in Section~\ref{sec:privacy_prelim}) are suitable to measure the
closeness of the resulting output distributions.

In the literature, \emph{privacy amplification} loosely refers to the
improvement of privacy analysis for a user using extra sources of randomness
other than the noise used for achieving its local privacy.  An example
is the randomness used in sampling data~\cite{DBLP:conf/crypto/ChaudhuriM06,DBLP:conf/nips/BalleBG18}; for ADMM, this can refer to the randomness in sampling each $f_t$ from $\mcal{D}$.  In applications where data from different users can be
processed in any arbitrary order, extra randomness from
\emph{shuffling} users' data have been considered~\cite{DBLP:conf/soda/ErlingssonFMRTT19, DBLP:conf/eurocrypt/CheuSUZZ19, DBLP:conf/crypto/BalleBGN19};
for ADMM, this can mean that the order of the users in the iterative process is randomly permuted. Privacy amplification \emph{by iteration}~\cite{DBLP:conf/focs/FeldmanMTT18}
has been proposed to analyze an iterative procedure in which
some noise is sampled in each iteration to achieve local privacy
for the user in that iteration.  The improved privacy analysis
is from the perspective of the user from the \textbf{first} iteration.
The intuition is that by exploiting the extra randomness generated in subsequent iterations,
the privacy guarantee against an adversary that observes
only the result at the end of the final iteration can be improved.
In this paper, we consider privacy amplification by iteration for ADMM;
in other words, we consider a deterministic sequence $f_1, f_2, \ldots, f_T$ of functions,
where the function $f_t$ is used in iteration~$t$ of ADMM, and
the only source of randomness is the noise $N_t$ sampled in each iteration~$t$,
which is used to mask only the $x$ variable (during access to the gradient oracle).

Loosely speaking, each iteration in the iterative process considered
in~\cite{DBLP:conf/focs/FeldmanMTT18,BalleBGG19}
corresponds to a \emph{non-expansive} mapping,
and an independent copy of Gaussian noise is added to the
result of each iteration before passing to the next iteration.
From the perspective of the user from the first iteration,
the privacy guarantee of the final output after $T$ iterations,
when measured with the $\Dz$-divergence\footnote{The results in~\cite{DBLP:conf/focs/FeldmanMTT18,BalleBGG19} are stated equivalently in terms of \renyi divergence.},
is improved by a multiplicative factor of $T$.

\ignore{
The factor~$T$ comes from the intuition that
a Gaussian noise with a variance of $\sigma^2$
leads to a $\Dz$-divergence proportional to $\frac{1}{\sigma^2}$ (see Fact~\ref{fact:zcdp_gaussian});
hence, for the special case where each non-expansive function
is the identity function, the sum of $T$ independent
Gaussian noises has a variance of $T \sigma^2$,
which explains a shrinking factor of $T$ in the divergence.
}

A recent work~\cite{DBLP:conf/icml/CyffersBB23}
employed this framework to examine privacy amplification by iteration
in the \textbf{proximal} variant of ADMM, 
for the purpose of analyzing privacy leakage to both the adversary and among different users.

\noindent \textbf{Our Contribution.} The main purpose
of this paper is to apply the approaches in~\cite{DBLP:conf/focs/FeldmanMTT18,BalleBGG19} to achieve privacy amplification by iteration for the
\textbf{gradient} variant (that uses only the gradient oracle to update the variable~$x$).
When one iteration of ADMM is considered,
the proximal variant as considered in~\cite{DBLP:conf/icml/CyffersBB23}
needs to pass only one variable between successive iterations,
while the gradient variant needs
to pass both the $x$ and $\lambda$ variables.
However, when one analyzes the transition
of variables in the $(x, \lambda)$-space,
there turns out to be two major technical hurdles,
which we give high levels ideas for how we resolve them
(where more details are described in Section~\ref{sec:challenges}).

\begin{compactitem}
\item \emph{Non-expansive iteration.}  It is crucial
in~\cite{DBLP:conf/focs/FeldmanMTT18,BalleBGG19} that before adding noise,
each iteration corresponds to a non-expansive mapping acting on the variable space.
While it is possible to rephrase one ADMM iteration (in Algorithm~\ref{alg:one_iteration}) as a transition in the $(x, \lambda)$-space,
it can be shown (in Section~\ref{sec:challenges}) that this transition
may correspond to a strictly expanding mapping under the usual norm.

Our novel idea is to design a customized norm in the $(x, \lambda)$-space 
that (i)~is suitable for analyzing 
the privacy of ADMM and (ii)~satisfies the condition
that one ADMM iteration corresponds to a non-expansive
mapping under this customized norm.
For strongly convex objective functions, we further refine our customized norm
under which each iteration becomes a strictly contractive mapping.

\item \emph{One-step privacy.}
In~\cite{DBLP:conf/focs/FeldmanMTT18,BalleBGG19},
the variable produced in each iteration is totally masked
at every coordinate with Gaussian noise before passing
to the next iteration.  Hence, it is somehow straightforward
(also using the aforementioned non-expansive property)
to achieve some privacy guarantee 
for one iteration in terms of $\Dz$-divergence.

However, the case for ADMM is more complicated.
As aforementioned, in each iteration, the sampled noise 
is used to mask only the $x$ variable.  This means that for
the variable in the $(x, \lambda)$-space returned in
one iteration, the $\lambda$-component receives no
noise and is totally exposed.  Therefore, no matter how much
noise is used to mask $x$ variable, the resulting privacy analysis
for one ADMM iteration will still give a $\Dz$-divergence of $+\infty$.

Our innovative idea is to consider one step as consisting of two noisy ADMM
iterations. The very informal intuition is that the two copies of
independent noises from two iterations can each be used to mask one component of $(x, \lambda)$ in the result at the end of the two iterations.
However, the formal argument is more involved.  To avoid complicated
integral calculations, we perform the relevant $\Dz$-divergence analysis
using the tools of adaptive composition of private mechanisms.  The most creative
part of this argument is to find a suitable intermediate variable
that can both (i)~facilitate the privacy composition proof and 
(ii)~utilize the aforementioned customized norm under which one ADMM
iteration is non-expansive.
\end{compactitem}

\noindent \textbf{Our Informal Statements.}  We show that from the perspective of the user
from the first iteration, the final variables after $T$ noisy ADMM iterations
achieve privacy amplification in the sense that the $\Dz$-divergence
is proportional to $\frac{1}{T}$;
for strongly convex objective functions,
the privacy amplification is improved to $\frac{L^T}{T}$, for some $0 < L < 1$.
The formal results for the general convex case are
in Theorem~\ref{th:ADMM_privacy} and Corollary~\ref{cor:first_user}.
The formal statements for the strongly convex case
are given in Section~\ref{sec:strongly_convex}.

\noindent \textbf{Privacy for Other Users.}
We analyze the privacy guarantee from the perspective of the first user to make the presentation clearer.
As pointed out in~\cite{DBLP:conf/focs/FeldmanMTT18}, very simple techniques can extend the privacy
guarantees to all users: (1) random permutation of all users; or
(2) \emph{random stopping}: if there are $n$ users, stop after a random number $R \in [1..n]$ of iterations. (Hence, with constant probability,
a user is either not included in the sample, or the number of iterations after it is $\Omega(n)$.)  We give the details in Section~\ref{sec:other_users}.

\noindent \textbf{Convergence Rates.}  We emphasize that our contribution is to analyze privacy amplification 
for private variants of ADMM that have already appeared in the literature~\cite{DBLP:conf/ccs/ZhangZ16,DBLP:journals/tifs/ShangXLLSG21}, whose applications and convergence rates have already been analyzed. 
However, for completeness, we present the tradeoff between utility (measured
by the convergence rate) and privacy (measured by the
variance of privacy noise) in Section~\ref{sec:utility}.


\noindent \textbf{Experimental Results.}
Despite being primarily theoretical, we conduct experiments on a general Lasso problem. Specifically, we empirically examine the effects of strong convexity and privacy noise magnitude on convergence rates.
The details are given in Section~\ref{sec:experiment}.

\noindent \textbf{Limitations or Social Concerns.}
Since we operate under the same privacy and threat model as in~\cite{DBLP:conf/focs/FeldmanMTT18,BalleBGG19},
we inherit the same limitations and social concerns 
as the previous works.  For instance,
our privacy amplification results assume that
subsequent users generate the randomness honestly and do not collude with the adversary.

\noindent \textbf{Paper Organization.}
While the most relevant works are mentioned in this section, further details
on related work are given in Section~\ref{sec:related}.
Background on ADMM is given in Section~\ref{sec:admm_prelim}
and formal privacy notions are given in Section~\ref{sec:privacy_prelim}.
In Section~\ref{sec:coupling}, we give a review of the previous 
coupling approach~\cite{BalleBGG19} that achieves privacy amplification
by iteration.
We explain the technical hurdles in Section~\ref{sec:challenges} when applying it to ADMM.
The details for the general convex case are given in Section~\ref{sec:technical};
the strongly convex case is given in Section~\ref{sec:strongly_convex}.
In Section~\ref{sec:other_users}, we apply the techniques in \cite{DBLP:conf/focs/FeldmanMTT18} to extend the privacy guarantees to all users.
The trade-off between privacy and utility is given in Section~\ref{sec:utility}.
We refer to Section~\ref{sec:experiment} for a numerical illustration of our algorithms on a general Lasso problem.
We have empirically confirmed that, as predicted theoretically, both the contraction factor and noise variance indeed affect the algorithm's convergence rates.
In Section~\ref{sec:improvement},
we discuss potential improvements of parameters
in our bounds.

\section{More Related Work}
\label{sec:related}

\ignore{
\begin{itemize}
            \item (1jtH, Xzxx) \cite{DBLP:conf/icml/CyffersBB23} does not consider the first order approximation of $f$.
            \item (Xzxx) \cite{DBLP:journals/tsp/MaoYHGSY20, DBLP:journals/siamis/OuyangCLP15} consider first order approximation.
            \item (D5qa) \cite{DBLP:conf/nips/ChourasiaYS21, DBLP:conf/nips/0001S22, DBLP:journals/corr/abs-2204-01585, DBLP:journals/corr/abs-2212-12629}, about Langevin framework, multiple epochs and mini-batches.
            \item (Xzxx) \cite{DBLP:journals/tifs/LiuGSALZF22} has the idea of Laplacian smoothing.
        \end{itemize}
				}

\noindent \emph{ADMM Background.}  More detailed explanation
of ADMM can be found in the book~\cite{DBLP:journals/ftml/BoydPCPE11}.
As many machine learning problems can be cast as the minimization
of some loss function,
many works have applied ADMM as a toolkit to solve the corresponding problems efficiently, e.g. \cite{DBLP:conf/icml/SiahkamariALGSK22, DBLP:conf/icml/0007YCSMJRT0W21, DBLP:conf/nips/HolmesZHW21, DBLP:conf/nips/LiYWX21, DBLP:conf/nips/VariciSST21, DBLP:conf/aaai/RenS021, DBLP:conf/aaai/HeA22}.


The convergence rate of ADMM has also been well investigated. For instance, based on the scheme of variational inequality, a clear proof for ADMM convergence has been give 
in~\cite{DBLP:journals/siamnum/HeY12}. Similar optimality conditions
used in the proof have also inspired us to analyze the non-expansive
inequality in Section~\ref{sec:non-expansion}.
\ignore{
When the objective functions $f$ and $g$ are assumed to be strongly convex, better convergence rates can be achieved\cite{DBLP:journals/siamis/GoldsteinOSB14}. 
}

\ignore{
Some works focus on the variants of the ADMM, such as simplified constraint\cite{DBLP:journals/moc/TianY19}, reducing the number of variables\cite{He2010OnTA}, and adding constraints to the matrices or objective functions\cite{DBLP:journals/jscic/DengY16}. 

\cite{DBLP:journals/jscic/DengY16} provides some linear convergence rate result for ADMM under some assumptions like strong convexity.
}
The gradient variant of ADMM is also focused in this paper,
which uses the first-order approximation
of the function $f(\cdot)$ such
that only access to the gradient oracle $\nabla f(\cdot)$ is sufficient
 during updates~\cite{DBLP:journals/siamis/OuyangCLP15, DBLP:journals/jscic/LiL19}. 
The advantage is that we can obtain a closed form for updating the~$x$~variable.
Recently, stochastic ADMM has attracted wide interests~\cite{DBLP:conf/icml/OuyangHTG13, DBLP:conf/ijcai/0004K16, DBLP:conf/icml/ZhongK14}. The main
insight is that by adding a zero-mean noise with bounded support or variance
to the result of the aforementioned gradient oracle,
convergence results can still be achieved. 
This enables the randomized framework of differential privacy~\cite{DBLP:conf/icalp/Dwork06,BunS16,Miro17} to be readily adopted, which typically adds
zero-mean noises to certain carefully selected variables in a procedure.


\noindent \emph{Private ADMM.}
Various privacy settings for ADMM have been considered in~\cite{DBLP:conf/ccs/ZhangZ16},
where noises are added to both the primal and the dual variables.
The privacy model where each iteration~$t$ uses some function $f_t$ from a user has been studied in~\cite{DBLP:conf/bigdataconf/DingZCXZP19, DBLP:conf/codaspy/ChenL20, DBLP:journals/tifs/ShangXLLSG21},
where local privacy can be achieved by adding noise to the $x$ variable.
For noisy gradient variant of ADMM~\cite{DBLP:journals/tifs/ShangXLLSG21}
adds noise to the result of the gradient oracle.
In Section~\ref{sec:admm_prelim}, we explain the technical issues for why we still add noise to the $x$ variable, as opposed to the gradient, even when we
focus on the gradient variant.

\ignore{
When the smoothness of $f(\cdot)$ is assumed, the noise used to perturbe the gradient or to perturbe the primal variables only differ in a constant magnitude if linear approximation of $f(\cdot)$ is used in minimizing $f(\cdot)$, which gives that these two perturbations are equivalent when the assmuption holds.
}


\noindent \emph{Privacy Amplification.}
To paraphrase~\cite{BalleBGG19},
privacy amplification refers to improving the privacy guarantees
by utilizing \emph{sources of randomness that are not accounted for
by standard composition rules},
where the ``standard'' noises are often interpreted
as those directly used to achieve local privacy for individual users.
For privacy amplification by iteration,
the amplification factor proportional to the number $T$ of iterations
 comes from the intuition that
a Gaussian noise with a variance of $\sigma^2$
leads to a $\Dz$-divergence proportional to $\frac{1}{\sigma^2}$ (see Fact~\ref{fact:zcdp_gaussian}).
Hence, for the special case where each non-expansive function
is the identity function, the sum of $T$ independent
Gaussian noises has a variance of $T \sigma^2$,
which explains a shrinking factor of $T$ in the divergence.

In addition to
privacy amplification by iteration~\cite{DBLP:conf/focs/FeldmanMTT18, BalleBGG19}
that is the focus of this paper, 
other approaches to privacy amplification have been investigated as follows.

\ignore{
The unavoidable trade-off between privacy and utility has promoted the development of further algorithm design to tight the privacy bounds. Privacy amplification refers to find ways to improve the privacy not depend on  standard composition\cite{BalleBGG19}. 
}

For methods that approximate the objective functions
by sampling user data,
privacy amplification by subsampling
utilizes the randomness involved in
sampling the data~\cite{DBLP:conf/crypto/ChaudhuriM06, DBLP:conf/focs/KasiviswanathanLNRS08, DBLP:conf/ccs/LiQS12, DBLP:conf/innovations/BeimelNS13, DBLP:conf/focs/BunNSV15, DBLP:conf/nips/BalleBG18, DBLP:conf/aistats/WangBK19}. It is shown that privacy
guarantees can be amplified proportionally with the square root of the size of dataset~\cite{DBLP:conf/nips/BalleBG18}. 

For aggregation methods that can 
process users' data in any arbitrary order,
extra randomness can be utilized to permute users' data
(that has already been perturbed by local noise) before sending
the shuffled data to the aggregator.
In the framework of privacy amplification by shuffling~\cite{DBLP:conf/soda/ErlingssonFMRTT19, DBLP:conf/eurocrypt/CheuSUZZ19, DBLP:conf/crypto/BalleBGN19},
the privacy guarantees for the shuffled randomized reports received by the aggregator 
can be amplified by a factor proportional to the square root of the number of users.

\ignore{
proposed . Usually, it require data to satisfy local differential privacy, and the transcripts between local agent and central are anonymized. It ensures that the reports from any local differential privacy protocol  will  also  satisfy  central  differential  privacy  at  a  per-report privacy-cost bound, which is also the square root of the size of dataset lower than the local differential privacy  bound\cite{DBLP:conf/crypto/BalleBGN19}.
}

\ignore{
While non-expansive mappings have been considered in~\cite{DBLP:conf/focs/FeldmanMTT18},
the privacy guarantees have been improved when each 
iteration corresponds to a strictly contracting mapping~\cite{BalleBGG19},
which is applicable when the objective functions involved
in stochastic gradient descent are strongly convex.
}

\ignore{
Privacy amplification by iteration mainly consider a sequence of post-processing functions, and show that during the consecutive operator there are a natural amplification phenomenon\cite{DBLP:conf/focs/FeldmanMTT18}.
}

\noindent \emph{Privacy Amplification for the Proximal ADMM Variant.}
Cyffers et al.~\cite{DBLP:conf/icml/CyffersBB23}
considered privacy amplification by both subsampling
and iteration for the proximal variant of ADMM.
They employed privacy amplification by iteration to analyze
the privacy leakage of one user's private data when another user observes the intermediate variables after a number of iterations.
In the proximal variant, it is possible to pass just
one variable $w = \lambda - \beta A x$ between successive iterations.
Indeed, they have shown that this can be captured by an abstract fixed-point iteration,
to which the framework of Feldman et al.~\cite{DBLP:conf/focs/FeldmanMTT18}
can be readily applied.
In contrast, the gradient variant of ADMM requires two variables
to be passed between successive iterations.

\noindent \emph{Other ADMM Interpretations.}
Instead of having one user per iteration,
some works consider the case that each iteration involves
multiple number of users~\cite{DBLP:conf/ccs/ZhangZ16, DBLP:conf/icml/ZhangKL18, DBLP:journals/tifs/HuangHGCG20, DBLP:conf/securecomm/DingEZGPH19,DBLP:journals/corr/abs-2207-10226} and the purpose
is for all the users to learn some common model parameters.  In this case, the
the coordinates of the variable $x$ are partitioned among the users,
each of which has a local copy incorporated into~$x$.  The linear constraint $A x + B y = c$ has a special form that essentially states that all copies should be the same as some global copy represented by the variable~$y$.  Hence, ADMM in this case can be interpreted as a distributed system
that promotes the consensus of the users' parameters towards a global agreement.

\noindent \emph{Analyzing Noisy Iterative Procedures 
with Langevin Diffusion.}
Some works have analyzed the randomness in an iterative procedure such 
as stochastic gradient descent
using Langevin diffusion~\cite{DBLP:conf/nips/ChourasiaYS21, DBLP:conf/nips/0001S22, DBLP:journals/corr/abs-2204-01585, DBLP:journals/corr/abs-2212-12629}.
Both privacy amplification by iteration frameworks~\cite{DBLP:conf/focs/FeldmanMTT18, BalleBGG19} use a coupling approach to analyze the divergence
associated with Gaussian noise.  By using stochastic differential equations,
the Langevin diffusion approach offers more flexibility and
allows more nuanced analysis for multiple epochs and mini-batches, where 
a user's data may be used multiple number of times in different iterations.
As commented in~\cite{DBLP:conf/nips/0001S22}, 
the coupling approach of Balle et al.~\cite{BalleBGG19} can
analyze multiple epochs via the straightforward
privacy composition, but may not give the best privacy guarantees.
Therefore, it will be interesting future work to analyze the noisy ADMM variants
under the Langevin diffusion framework.

\section{Preliminaries}
\label{sec:admm_prelim}

\subsection{Basic ADMM Framework}


\noindent \textbf{ADMM Convex Program.}
Suppose for some positive
integers $n$ and $\ell$,
we have convex functions $\f: \R^n \rightarrow \R$
and $g: \R^\ell \rightarrow \R$.
Suppose for some positive integer $m$,
we have linear transformations (also viewed as matrices)
$A: \R^n \rightarrow \R^m$
and $B: \R^\ell \rightarrow \R^m$,
and a vector $c \in \R^m$.
The method ADMM is designed
to tackle convex programs of the form:

\begin{subequations}
\begin{align}
		\min_{x, y} \quad &   \f(x) + g(y) \label{objective} \\
		\text{s.t.} \quad & Ax + By = c \in \R^m \label{prob:modified} \\
		\quad & x \in \R^n, y \in \R^\ell
\end{align}
\end{subequations}

\noindent \textbf{Function $\f$ as an expectation functions.}
In some learning applications, the function $\f$ is derived from
a distribution $\mcal{D}$ of functions $f : \R^n \rightarrow \R$.
We use $f \in \mcal{D}$ to mean that $f$ is in the support of $\mcal{D}$
and $f \gets \mcal{D}$ to mean sampling $f$ from $\mcal{D}$.
Then, the function $\f$ has the form
$\f(x) = \E_{f \gets \mcal{D}}[f(x)]$.

In this work, we focus on the case that the functions $f \in \mcal{D}$
are differentiable.
In the basic version,
we assume that the algorithm has oracle access to the gradient $\nabla \f(\cdot)$.
However, in the stochastic version described
in Section~\ref{sec:stochastic_ADMM}, there are $T$ i.i.d. samples $(f_1, f_2, \ldots, f_T)$ from
$\mcal{D}$ and the algorithm only has oracle
access to the gradient $\nabla f_t(\cdot)$ for each $t \in [T]$.

\noindent \textbf{Notation.}
We use $\langle \cdot, \cdot \rangle$ to represent the standard
inner product operation and $\| x \| := \sqrt{\langle x, x \rangle}$
to mean the usual Euclidean norm.  We use
$\I$ to denote the identity map (or matrix) in the appropriate space.  Since a linear transformation~$A$
can be interpreted as a matrix multiplication,
we use the transpose notation \mbox{$A^\top: \R^m \rightarrow \R^n$} to denote the adjoint
of $A$.  We also consider the \emph{operator norm} $\| A \| := \sup_{x \neq y} \frac{\| A x - A y \|}{\|x - y\|}$.

\noindent \textbf{Smoothness Assumption.}
We assume that every function $f$ in the support of $\mcal{D}$ is differentiable and $\nu$-smooth for some $\nu > 0$;
in other words, $\nabla f(\cdot)$ is $\nu$-Lipschitz,
i.e., for all $x$ and $x'$,
we have $\|\nabla f(x) - \nabla f(x')\| \leq \nu \|x - x'\|$.
To avoid too many parameters, we will
mostly set $\eta = \frac{1}{\nu}$ for the general case.
(For the strongly convex case in Section~\ref{sec:strongly_convex}, we will set $\eta$ differently.)

\noindent \textbf{Augmented Lagrangian Function.}
Recall that we have primal variables $x \in \R^n$ and  $y \in \R^\ell$,
and the dual variable $\lambda \in \R^m$
corresponds to the feasibility constraint~(\ref{prob:modified}).
For some parameter~$\beta > 0$, the following
augmented Lagrangian function is considered in the literature:

\begin{align}
\f(x) + g(y) - \langle \lambda, A x + B y - c \rangle
+ \frac{\beta}{2} \|Ax + By - c\|^2.
\label{eq:basic_L}
\end{align}

The parameter $\beta > 0$ is chosen to offer a tradeoff between the approximations of the original objective
function $\f(x) + g(y)$ versus the feasibility constraint $Ax + By = c$,
where a larger value of $\beta$ means that more importance is placed
on the feasibility constraint.

On a high level, ADMM is an iterative method. In the literature, the description of one iteration consists of three steps in the following order:

\begin{compactitem}
\item Keeping $y$ and $\lambda$ fixed, update $x$ to minimize the Lagrangian function in (\ref{eq:basic_L}).
\item Keeping $x$ and $\lambda$ fixed, update $y$ to minimize the above Lagrangian function
in~(\ref{eq:basic_L}).
\item Update $\lambda$ accordingly to ``encourage'' the feasibility of $(x,y)$.
\end{compactitem}

However, since we later need to consider privacy leakage in each iteration,
in our formal description, it will be more convenient to start each iteration with the second bullet.
This is an equivalent description, because these three steps are executed in a round-robin
fashion.

\noindent \textbf{Augmented Lagrangian Function with First Order Approximation for Differentiable~$f$.}  As we shall see later,
in the stochastic version, the algorithm may have
access to only samples $f$ from $\mcal{D}$,
and eventually, we will need to add noise to preserve privacy.
Instead of assuming that the algorithm has full knowledge of the sampled $f \in \mcal{D}$,
the first order approximation of~$f$
(with respect to some current~$\widehat{x}$) is considered
in the literature~\cite{DBLP:journals/siamis/OuyangCLP15, DBLP:journals/jscic/LiL19}
as follows, where $\eta > 0$ is typically chosen
to be $\eta = \frac{1}{L}$ according to the smoothness parameter~$L$
as in the standard gradient descent method.
The advantage is that the algorithm only needs one access to
the gradient oracle $\nabla f(\cdot)$ at the point $\widehat{x}$.

\begin{subequations}
\begin{align}
\mcal{L}^{f}_{\widehat{x}}(x,y,\lambda) & :=  f(\widehat{x}) + \langle \nabla f(\widehat{x}), x - \widehat{x} \rangle
 + \mcal{H}(x,y,\lambda) + \frac{1}{2 \eta} \|x - \widehat{x}\|^2  \\
\mcal{H}(x,y,\lambda) & :=  g(y) - \langle \lambda, A x + B y - c \rangle
+ \frac{\beta}{2} \|Ax + By - c\|^2
\end{align}
\end{subequations}

\noindent \textbf{One ADMM Iteration.}  As mentioned above,
our description of each ADMM iteration starts
with updating the variable~$y$
in each ADMM iteration.  In Algorithm~\ref{alg:one_iteration},
the input to each iteration $t+1 \in [T]$ is $(x_{t}, \lambda_{t})$
from the previous iteration,
where $(x_0, \lambda_0)$ is initialized arbitrarily.
Each iteration uses some function $f_{t+1}$ (which is the
same as $\f$ in the basic version).
Observe that only oracle access to the gradient $\nabla f_{t+1} (\cdot)$ is sufficient,
but we assume that the algorithm has total knowledge of $g$.
Moreover, given $(x_{t+1}, \lambda_{t+1})$, we can recover
$y_{t+1} \gets \mcal{G}({\lambda}_{t+1} - \beta A x_{t+1} )$ deterministically as in Lemma~\ref{lemma:step2}.
Hence, we only need to pass variables in the  $(x, \lambda)$-space between
consecutive iterations and treat $y$ as an intermediate variable within each iteration.

\begin{algorithm}[H]
\caption{One ADMM Iteration}
\label{alg:one_iteration}

\KwIn{Previous $(x_{t}, \lambda_{t}) \in \R^n \times \R^m$
and function $f_{t+1} : \R^n \rightarrow \R$.}

\KwOut{$(x_{t+1}, 
\lambda_{t+1} 
)$}

	$y_{t} \gets \mcal{G}({\lambda}_{t} - \beta A x_{t})$
	\hfill	\Comment{pick canonical minimizer (see Lemma~\ref{lemma:step2})}

	$\lambda_{t+1} \gets {\lambda}_{t} - \beta (A x_{t} + B y_{t} - c)$

	$x_{t+1} \gets \mcal{F}^{f_{t+1}}(x_{t}, y_{t},{\lambda}_{t+1})$ \label{ln:x}
	\hfill	\Comment{oracle access to $\nabla f_t(\cdot)$ (see Lemma~\ref{lemma:step1})}


\KwRet $(x_{t+1}, \lambda_{t+1})$
\hfill	\Comment{$y_{t+1} \gets \mcal{G}({\lambda}_{t+1} - \beta A x_{t+1} )$ can be recovered from $(x_{t+1}, \lambda_{t+1})$}

\end{algorithm}

\begin{lemma}[Local Optimization for $g$]
\label{lemma:step2}
There exists $\mcal{G}: \R^m \rightarrow \R^\ell$ such that
for all $x$ and $\lambda$,
$\mcal{G}(\lambda - \beta A x) = \arg \min_{y} \mcal{H}(x,y,\lambda)$.\footnote{
Note that the minimizer might not be unique.  In practice,
some deterministic method can pick a canonical value,
or alternatively, we can invoke the Axiom of Choice such that $\mcal{G}$ returns
only one value in~$\R^\ell$.}

Moreover, for any $x \in \R^n$, $\lambda \in \R^m$,
 $y_1 = \mcal{G}(\lambda - \beta A x)$
and $y \in \R^\ell$, we have

$g(y_1) - g(y) \leq \langle \lambda - \beta (Ax + B y_1 - c), B(y_1 - y) \rangle$.
\end{lemma}

\begin{proof}
Observe that we can express
$\mcal{H}(x,y,\lambda) = \varphi(y) - \langle \lambda - \beta A x, y \rangle
+ \vartheta(x, \lambda)$
for some functions $\varphi(y)$ and $\vartheta(x, \lambda)$.
Observe that the variables $(x, \lambda)$ and $y$ only interact
in the middle inner product term.
Therefore, fixing $(x, \lambda)$,
$\arg \min_y \mcal{H}(x,y,\lambda)$ is a function of $\lambda - \beta A x$.

The optimality of $y_1$
implies that $0 \in \partial_y \mcal{H}(x, y_1, \lambda) = \partial g(y_1) - B^{\top} \lambda + \beta B^{\top} (A x + B y_1 - c)$.

Hence, we have $s_1 := B^{\top} (\lambda - \beta (Ax + By_1 - c)) \in \partial g(y_1)$.
The convexity of $g$ implies that for all $y \in \R^\ell$,
$g(y_1) - g(y) \leq \langle s_1, y_1 - y \rangle
= \langle \lambda - \beta (Ax + B y_1 - c), B(y_1 - y) \rangle$.

\end{proof}

\begin{lemma}[Local Optimization for $f$]
\label{lemma:step1}
Given differentiable and convex $f: \R^n \rightarrow \R$,
define $\mcal{F}^{f}: \R^n \times \R^\ell \times
 \R^m \rightarrow \R^n$
by
$\mcal{F}^{f}(\widehat{x}, y, \lambda)
:= \arg \min_{x} \mcal{L}^{f}_{\widehat{x}}(x,y,\lambda)$.
Then, it follows that the minimum is uniquely attained by

$$\mcal{F}^{f}(\widehat{x}, y, \lambda) = (\I + \eta \beta A^{\top} A)^{-1}
\{\widehat{x} - \eta \cdot [\nabla f(\widehat{x}) + A^{\top} (\beta(B y - c) - \lambda )] \}.$$

Moreover, if $x_1 = \mcal{F}^{f}(\widehat{x}, y, \lambda)$,
then we have:

$$\nabla f(\widehat{x}) = A^{\top} (\lambda - \beta(A x_1 + By - c)) + \frac{1}{\eta} \cdot (\widehat{x} - x_1).$$

\end{lemma}

\begin{proof}
One can check that

$$\nabla_x \mcal{L}^f_{\widehat{x}}(x,y,\lambda) =
\nabla f(\widehat{x}) - A^{\top} \lambda
+ \beta A^{\top} (A x + B y - c) + \frac{1}{\eta} \cdot (x - \widehat{x}).$$

Setting $\nabla_x \mcal{L}^f_{\widehat{x}}(x_1,y,\lambda) = 0$
and observing that $(\I + \eta \beta A^\top A)$ has only eigenvalues at least 1 give the result.
\end{proof}

\subsection{Private Stochastic ADMM}
\label{sec:stochastic_ADMM}

\noindent \textbf{Private vs Public Information.}
As aforementioned, the algorithm does not always have direct access
to the function $\f = \E_{f \gets \mcal{D}}[f]$.  Instead, we consider
the scenario with $T$ users, where each user~$t \in [T]$ samples
some function $f_t$ from $\mcal{D}$ independently.  Each user~$t$
considers its function~$f_t$ as private information
and as seen in Algorithm~\ref{alg:one_iteration},
when a user participates in one each iteration of ADMM,
only oracle access to $\nabla f_t(\cdot)$ is sufficient,
but the resulting information may leak private information about $f_t$.
On the other hand, all other objects such as $g$, $A$, $B$, $c$
and the initialization $(x_0, \lambda_0)$ are considered public information.

\noindent \textbf{Randomness in Privacy Model.}  In the literature,
the privacy for the functions $(f_t : t \in [T])$
is studied from the \emph{sampling} perspective~\cite{DBLP:conf/crypto/ChaudhuriM06, DBLP:conf/focs/KasiviswanathanLNRS08, DBLP:conf/ccs/LiQS12, DBLP:conf/innovations/BeimelNS13, DBLP:conf/focs/BunNSV15, DBLP:conf/nips/BalleBG18, DBLP:conf/aistats/WangBK19},
as there is randomness in sampling each $f_t$ from $\mcal{D}$
(which is also needed for analyzing the quality of the
final solution $(x_T, y_T)$ with respect
to the original objective function $\f(\cdot) + g(\cdot)$).
Moreover, since $f_t$ is sampled i.i.d. from $\mcal{D}$,
one can also potentially consider a random permutation
of the users' private information under the \emph{shuffling} perspective~\cite{DBLP:conf/soda/ErlingssonFMRTT19, DBLP:conf/eurocrypt/CheuSUZZ19, DBLP:conf/crypto/BalleBGN19}
and exploit this randomness for privacy amplification.
However, in this work, our privacy model does not consider
the randomness involved in sampling the functions or shuffling the users.
Instead,
we assume that the (private) sequence of functions $(f_t: t \in [T])$ is fixed,
which also determines that in iteration $t \in [T]$ of ADMM,
the function $f_t$ from user~$t$ will be accessed (via the gradient oracle).

\noindent \textbf{Neighboring Notion.}  Under the differential privacy
framework, one needs to specify the \emph{neighboring} notion.

\begin{definition}[Neighboring Functions]
\label{defn:neighbor_functions}
For $\Delta \geq 0$, define a (symmetric) neighboring relation~$\sim_\Delta$
on functions in $\mcal{D}$ such that
two functions $f \sim_\Delta f'$ are neighboring
if for all $x \in \R^n$,

$\|\nabla f (x) - \nabla f'(x)\| \leq \Delta$.
\end{definition}


\noindent \textbf{Noisy ADMM for Local Privacy.} \emph{Where should noised be added?}
In iteration~$t+1 \in [T]$
for Algorithm~\ref{alg:one_iteration},
the private information $f_{t+1}$ of user~$t+1$
is accessed only in the computation of $x_{t+1}$ in
line~\ref{ln:x} via the gradient oracle $\nabla f_{t+1}(\cdot)$.

\noindent \emph{Informal definition of local privacy.}
Suppose some $(x_{t}, \lambda_{t})$ is returned at
the end of the iteration~$t$.  Consider two neighboring
functions $f_{t+1} \sim_{\Delta} f'_{t+1}$, which are used in two scenarios
of executing iteration~$t+1$ with the same input
$(x_{t}, \lambda_{t})$.  Local privacy for
user~$t+1$ means that as long as the two functions $f_{t+1}$ and $f'_{t+1}$ are neighboring,
the corresponding (random) $\widetilde{x}_{t+1}$ and $\widetilde{x}'_{t+1}$
from the two scenarios will have ``close'' distributions.

In Section~\ref{sec:privacy_prelim},
the closeness between two distributions will be formally quantified
using \emph{divergence}.
A standard way to achieve local privacy (with
respect to the neighboring notion defined above)
for user~$t+1$ is to sample some noise $N_{t+1} \in \R^n$, e.g.,
Gaussian noise $\mcal{N}(0, \sigma^2 \I)$ with some appropriate variance~$\sigma^2$.
We shall see that the privacy guarantee
will be determined by $\sigma^2$ and the parameter~$\Delta$
defined for the neighboring relation~$\sim_\Delta$ on the functions.
There are two possible ways to apply this noise.

\begin{compactitem}

\item The noise $N_{t+1}$ is added immediately to the result
of the gradient oracle.  Hence, when the query point is
$\widehat{x}$, the gradient oracle returns $\nabla f_{t+1}(\widehat{x}) + N_{t+1}$.

This is the most standard way noise is added to achieve local privacy in the literature~\cite{DBLP:journals/jmlr/ChaudhuriMS11, DBLP:journals/tifs/ShangXLLSG21}.
However, according to Lemma~\ref{lemma:step1},
the noise $N_{t+1}$ will be subjected to the transformation
$(\I + \eta \beta A^{\top} A)^{-1}$, which makes the calculation
slightly more complicated, because even though the transformation
has eigenvalues in the range $(0,1]$,
the ratio of the largest to the smallest eigenvalues may be large.

\item For simpler analysis,
we shall first compute $x_{t+1}$ without any noise.
Then, the generated noise $N_{t+1}$ is used to
return the masked value
$\widetilde{x}_{t+1} \gets x_{t+1} + N_{t+1}$.
In this paper, we will refer to this
as the \emph{noisy variant} of Algorithm~\ref{alg:one_iteration},
or \emph{noisy ADMM}.

The drawback is that in the analysis we might use noise with
a slightly larger variance than needed to achieve local privacy.
Observe that if we choose the parameters $\beta$ and $\eta$ such that
$\eta \beta \|A^\top A\| = O(1)$,
then the privacy guarantees for both approaches
are the same up to a constant multiplicative factor.
\end{compactitem}

Observe that as far as local privacy is concerned,
there is no need to mask $y_{t}$ or $\lambda_{t+1}$,
whose computation does not involve the private function $f_{t+1}$.

\noindent \textbf{Privacy Amplification by Iteration.}
Observe that in each iteration~$t \in [T]$,
some noise~$N_t$ is sampled to mask the value~$x_t$
to achieve local privacy for user~$t$.
Privacy amplification by iteration refers to the privacy analysis
from the perspective of the \textbf{first user} (i.e. $t = 1)$.  Since there is so much
randomness generated in all $T$ iterations,
will the privacy guarantee for the finally returned $(\widetilde{x}_T, y_T, \lambda_T)$
be amplified with respect to the first user?
The challenge here is that we are analyzing the ``by-product'' (for the benefit
of the first user) of an iterative process that is locally private for the user in each iteration.  In particular, observe that in each iteration~$t$,
noise is added only to the computation of the $x$ variable,
but not to the $y$ and $\lambda$ variables.

Finally, to achieve privacy amplification for ADMM,
we need to transform the problem instance into an appropriate form,
whose significance will be apparent in Section~\ref{sec:one_step_privacy}.

\begin{remark}[Transformation of the Linear Constraints]
\label{remark:transform}
By Gaussian elimination,
we may assume without loss of generality that $m \leq n$
and the matrix $A$ has the form $A = \left [\I_m \, | \, D \right ]$ for some $m \times (n-m)$ matrix $D$.  (The reason we need this assumption is that if $N$ is sampled from
the multivariate Gaussian distribution $\mcal{N}(0, \sigma^2 \I_n))$,
then the linear transformation $A N$ still contains a
fresh copy of $\mcal{N}(0, \sigma^2 \I_m)$.)

However,
after the process of Gaussian elimination, we may have
extra linear constraints of the form $\widehat{B} y = \widehat{c}$,
which can be absorbed into a modified convex function $\widehat{g}$ in the following:

$$ \widehat{g}(y) := \begin{cases}
g(y), & \text{if } \widehat{B} y = \widehat{c}; \\
+ \infty, & \text{otherwise.}
\end{cases} $$

Observe this transformation does not change $n$ and $\ell$.  However,
if initially $m$ is strictly greater than $n$,
then the transformation ensures that $m \leq n$ afterwards.
\end{remark}

\section{\renyi and Zero-Concentrated Differential Privacy Background}
\label{sec:privacy_prelim}

Just like differential privacy~\cite{DBLP:conf/icalp/Dwork06}, \renyi~\cite{Miro17} 
and zero-concentrated~\cite{BunS16} differential privacy
have a similar premise.  Suppose $\mcal{I}$ is the collection 
of private inputs.
Given some (randomized) mechanism~$M$ and input $I \in \mcal{I}$,
we use $\mcal{A}^M(I)$ to denote the random object from 
some set~$\mcal{O}$ of views that can
be observed by the adversary~$\mcal{A}$ when the mechanism~$M$ is run on
input~$I$.
The input set~$\mcal{I}$
is equipped with some 
symmetric binary relation~$\sim$ known as \emph{neighboring}
such that the adversary~$\mcal{A}$ cannot distinguish with certainty between neighboring inputs when mechanism~$M$ is run on them.
Intuitively, for any
neighboring inputs $I \sim I'$,
the distributions of the random objects $\mcal{A}^M(I)$
and $\mcal{A}^M(I')$ are ``close''.  
The closeness of two distributions $P$ and $Q$ is quantified
formally by some notion of \emph{divergence}~$\D(P\|Q)$.
\ignore{
Even though a divergence~$\D$ may not be symmetric in general,
we can consider the symmetric variant $\max\{ \D(P\|Q), \D(Q\|P)\}$,
which is already implicitly considered because the neighboring 
relation~$\sim$ is symmetric.
}

\ignore{
\noindent \textbf{Standard DP in terms of divergence.}
For the standard differential privacy~\cite{DBLP:conf/icalp/Dwork06},
consider the $\delta$-max divergence:
 
$$\D^{\delta}_{\max} (P || Q) := \ln \sup_{S \subseteq \mcal{O}} \frac{P(S) - \delta}{Q(S)}.$$

\noindent A mechanism $M$ is $(\epsilon, \delta)$-DP
if for all neighboring inputs
$I \sim I'$, 
$\D^{\delta}_{\max} (\mcal{A}^M(I) || \mcal{A}^M(I')) \leq \epsilon$.

Another example is
the \renyi divergence.
}

\begin{definition}[\renyi Divergence~\cite{renyi1961measures}]
\label{defn:renyi}
Given distributions $P$ and $Q$ over some sample space~$\mcal{O}$,
the \renyi divergence of order $\alpha > 1$ between them is:
\begin{equation}
\D_\alpha(P \| Q) := \frac{1}{\alpha - 1} \ln \E_{x \leftarrow Q} \left(\frac{P(x)}{Q(x)} \right)^\alpha.
\end{equation}
\end{definition}

\ignore{
\begin{definition}[\renyi Differential Privacy~\cite{Miro17}]
For $\epsilon \geq 0$ and $\alpha > 1$,
a mechanism~$M$ is $\epsilon$-$\RDP_\alpha$
($\epsilon$-\renyi differentially private\footnote{We prefer this notation
over the original $(\alpha, \epsilon)$-$\RDP$ in~\cite{Miro17}.} of order~$\alpha$) against
adversary~$\mcal{A}$ if for all neighboring inputs $I \sim I'$,
we have
$\D_\alpha(\mcal{A}^M(I) \| \mcal{A}^M(I') ) \leq \epsilon.$
\end{definition}

\begin{remark}
The case $\alpha = \infty$ reduces to the usual differential privacy.
\end{remark}
}

\noindent \textbf{Gaussian Distribution.}
For the finite dimensional Euclidean space $\R^n$,
we use $\mcal{N}(x, \sigma^2 \I_n)$
to denote the standard Gaussian distribution with mean~$x \in \R^n$
and variance~$\sigma^2$.
The following fact shows that
\renyi divergence is suitable for analyzing Gaussian noise.

\begin{fact}[\renyi Divergence between Gaussian Noises~\cite{BunS16}]
\label{fact:rdp_gaussian}
For $\alpha > 1$ and vectors $x, x' \in \R^n$,

$\D_\alpha(\mcal{N}(x, \sigma^2 \I) \| \mcal{N}(x', \sigma^2 \I) ) = \frac{\alpha \|x - x'\|^2}{2 \sigma^2}.$
\end{fact}

Observe that in Fact~\ref{fact:rdp_gaussian}, it is possible to
divide both sides of equation to define a notion of divergence
without the extra $\alpha$ parameter.

\noindent \textbf{Zero-Concentrated Divergence.}  
To define zero-concentrated differential privacy (zCDP),
we consider the following divergence:

$$\Dz(P||Q) := \sup_{\alpha > 1} \frac{1}{\alpha} \cdot \D_\alpha(P \| Q).$$

From Fact~\ref{fact:rdp_gaussian},
we can get an equation without explicitly including the
$\alpha$ parameter.  The following fact allows
the analysis $\Dz$ divergence without explicitly considering
the probability density functions of Gaussian noises.

\begin{fact}[$\Dz$-Divergence between Gaussian Noises~\cite{BunS16}]
\label{fact:zcdp_gaussian}
For any vectors $x, x' \in \R^n$,

$\Dz(\mcal{N}(x, \sigma^2 \I_n) \| \mcal{N}(x', \sigma^2 \I_n) ) = \frac{\|x - x'\|^2}{2 \sigma^2}.$
\end{fact}

Therefore, we focus on zero-concentrated differential privacy in this paper.

\begin{definition}[Zero-Concentrated Differential Privacy~\cite{BunS16}]
\label{defn:zcdp}
For $\epsilon \geq 0$,
a mechanism~$M$ is $\epsilon$-$\zCDP$
against adversary~$\mcal{A}$ if for all neighboring inputs $I \sim I'$,
we have
$\Dz(\mcal{A}^M(I) \| \mcal{A}^M(I') ) \leq \epsilon.$ 
\end{definition}

\noindent \textbf{Remark.}  Instead of using the
terminology ``$\epsilon$-$\zCDP$'' (which implicitly
requires some notion of neighboring inputs),
sometimes it is more convenient to directly use the inequality
in Definition~\ref{defn:zcdp}
and  express $\epsilon$ as a function
of some ``distance notion'' between the inputs $I$ and $I'$.

The following fact states the properties of $\Dz$-divergence
that we need.

\begin{fact}[Properties for $\Dz$-Divergence~{\cite[Lemma 2.2]{BunS16}}\footnote{The
original results~\cite[Lemma 2.2]{BunS16}
have been stated in terms of the \renyi divergence $\D_\alpha$,
but they have also shown that they can be readily generalized
to $\Dz$.}]
\label{fact:zCDP}
Suppose $(X,Y)$ and $(X',Y')$ are joint distributions
with the same support.  Then, we have the following conclusions.

\begin{compactitem}
\item[(a)] \emph{Data processing inequality.} 
It holds that $\Dz(Y \| Y') \leq \Dz((X,Y) \| (X',Y'))$.

\item[(b)] \emph{Uniform to Average Bounds (a.k.a. Quasi-convexity).}
Suppose there exists $\epsilon_2 \geq 0$ such that
for any~$x$ in the support of $X$ and $X'$,  the conditional distributions
satisfy:  

$\Dz((Y|X=x) \| (Y'|X'=x)) \leq \epsilon_2$.

Then, the marginal distributions satisfy $\Dz(Y \| Y') \leq \epsilon_2$.

\item[(c)] \emph{Adaptive composition.}  Suppose,
in addition to the condition in~(b),
there exists $\epsilon_1 \geq 0$ such that
$\Dz(X \| X') \leq \epsilon_1$.

Then, it holds that
$\Dz((X,Y) \| (X', Y')) \leq \epsilon_1 + \epsilon_2$.

\end{compactitem}

\end{fact}

\section{Amplification by Iteration via Coupling Framework}
\label{sec:coupling}

We review the coupling framework in~\cite{BalleBGG19}
used to analyze privacy amplification in an iterative process
and outline the major technical challenges
that will be encountered when we apply it to analyze ADMM.

\noindent \textbf{Iteration interpreted as a Markov operator.}
Suppose the information passed between consecutive
iterations is an element in some sample space $\Omega$
(which is also equipped with some norm\footnote{We actually
just need the linearity property $\|a z \|^2 = a^2 \|z\|^2$
for $a \in \R$ and $z \in \Omega$.} $\| \cdot \|$).
Since each iteration uses fresh randomness,
it can be represented as a Markov operator $K: \Omega \rightarrow \mcal{P}(\Omega)$.
If $z \in \Omega$ is the input to an iteration,
then $K(z)$ represents the distribution of the output returned
by that iteration.  We use $\mcal{K}$ to denote
the collection of \emph{iteration} Markov operators,
each of which corresponds to some iteration.
For each $K \in \mcal{K}$ and $z \in \Omega$, we also view that
the process samples some fresh randomness $N$ (independent of $z$) from an appropriate
distribution, and the output is a deterministic
function of $(z, N)$.   (We will use this randomness $N$ when we consider
the notion of \emph{coupling} described below.)

\noindent \textbf{Technical Assumption.}
We assume that for any $K \in \mcal{K}$ and for any $z, z' \in \Omega$,
the distributions $K(z)$ and $K(z')$ have the same
support.  Observe this is certainly true
if we mask with Gaussian noise, which has the whole ambiance space as the support.
 We shall see that when
adapting the framework to ADMM, this turns out to be a crucial property.

\noindent \emph{Example.} For noisy gradient descent,
suppose some (private) function $f$ is used in an iteration.
Then, given input~$z$, some fresh Gaussian noise $N$ is sampled,
and the output $z - \eta \cdot \nabla f(z) + N$ is returned.

\noindent \textbf{Notation.}
Observe that we can naturally extend a Markov operator
to $K: \mcal{P}(\Omega) \rightarrow \mcal{P}(\Omega)$.
  If $\mu \in \mcal{P}(\Omega)$ is a distribution, then
$K(\mu) \in \Omega$ is the distribution corresponding to the following sampling process: (i) first, sample $z \in \Omega$ from $\mu$, (ii) second, return a sample
from the distribution $K(z)$.

Differing slightly from the notation in~\cite{BalleBGG19},
we denote the composition of Markov operators using
the function notation, i.e.,
$(K_2 \circ K_1)(\mu) = K_2(K_1(\mu))$.

\noindent \textbf{Coupling.}  Given two distributions
$\mu, \nu \in \mcal{P}(\Omega)$,
a coupling $\pi$ from $\mu$ to $\nu$
is a joint distribution on $\Omega \times \Omega$
such that marginal distributions for the
two components are $\mu$ and $\nu$, respectively.

When the same iterator operator $K \in \mcal{K}$
is applied in two different scenarios,
the \emph{natural coupling} refers to using
the same aforementioned randomness $N$ sampled within the iteration process
for both scenarios.

The following notion gives a uniform bound on the distance between
two distributions.

\begin{definition}[$\infty$-Wasserstein Distance]
\label{defn:wasserstein}
Given distributions $\mu$ and $\nu$ on
some normed space $\Omega$ and $\Delta \geq 0$,
a coupling $\pi$ from $\mu$ to $\nu$ is a witness 
that the (infinity) Wasserstein distance $\W(\mu, \nu) \leq \Delta$
if for all $(z,z') \in \supp(\pi)$, $\|z - z'\| \leq \Delta$.

The distance $\W(\mu, \nu)$ is the infimum of the
collection of $\Delta$ for which such a witness $\pi$ exists.
\end{definition}

\subsection{Technical Details}

The following theorem paraphrases the privacy
amplification result in~\cite[Theorem 4]{BalleBGG19}.

\begin{theorem}[Privacy Amplification by Iteration~\cite{BalleBGG19}]
\label{th:privacy_amp}
Suppose the collection $\mcal{K}$ of iteration 
operators satisfies the following conditions.

\begin{compactitem}

\item[(A)] \emph{Non-Expansion.} There is some $0<L\leq 1$ such that
for any $K \in \mcal{K}$ and $z, z' \in \Omega$, there
exists a witness (e.g., the natural coupling)
for
$\W(K(z), K(z')) \leq L \|z - z'\|$.

\item[(B)] \emph{One-step Privacy.}  There exists
a constant $C_{\mcal{K}} > 0$ (depending on $\mcal{K}$) such that
for any $K \in \mcal{K}$ and $z, z' \in \Omega$, it holds that:
$\Dz(K(z) \| K(z')) \leq C_{\mcal{K}} \cdot \|z - z'\|^2$.

\end{compactitem}

Then, given for any $T \geq 1$ iterator operators $K_1, K_2, \ldots, K_T \in \mcal{K}$ and $z, z' \in \Omega$, it holds that

$\Dz( (K_T \circ \cdots \circ K_1)(z) \| (K_T \circ \cdots \circ K_1)(z'))
\leq \frac{C_{\mcal{K}}L^{T-1}}{T} \cdot \|z - z' \|^2$.
\end{theorem}

\begin{remark}
\label{remark:first_iteration}
Using Fact~\ref{fact:zCDP}(b),
the conclusion of Theorem~\ref{th:privacy_amp}
generalizes readily from a pair $(z, z')$ 
of points to a pair $(\mu, \mu')$ of distributions,
where $\|z - z'\|$ is replaced by $\W(\mu, \mu')$.

Another technical issue is that when we consider
privacy amplification for an iterative procedure,
the initial solution $z_0$ is the same
for both scenarios, but the first iteration
uses different neighboring functions $f_1 \sim_\Delta f'_1$ (as
in Definition~\ref{defn:neighbor_functions}),
which will lead to different Markov operators $K_1$ and $K_1'$.
There are two ways to resolve this inconsistency with Theorem~\ref{th:privacy_amp},
which considers different starting $z$ and $z'$ in the two scenarios, but uses
the same sequence of Markov operators for both scenarios (including $K_1)$.

\begin{compactitem}

\item We can start with the distributions $\mu = K_1(z_0)$ and $\mu' = K_1'(z_0)$
and use the natural coupling (induced by using the same randomness $N$)
as a witness for $\W(\mu, \mu') \leq \eta \|\nabla f(z_0) - \nabla f'(z_0)\| \leq \eta \Delta$.  However, this means that we do not
exploit the randomness of $K_1$ and $K'_1$, leading to a slight loss of privacy guarantee
in the analysis.

\item An alternative is to consider
$z = z_0 - \eta \nabla f(z_0)$ and $z' = z_0 - \eta \nabla f'(z_0)$
with $\|z - z'\| \leq \eta \Delta$.
Then, we can use the same Markov operator $K_1$ for the first iteration
that is specially defined by $(z, N) \mapsto z + N$, where $N$ is the sampled randomness.
\end{compactitem}

For simplicity, we will use the first approach and do not need to redefine the Markov operator for the first iteration.

\end{remark}

\noindent \textbf{Example.} Continuing 
with the example of noisy gradient descent,
suppose each iterator operator $K \in \mcal{K}$
corresponds to some convex function $f$
that is $\frac{1}{\eta}$-smooth (i.e., $\nabla f(\cdot)$
has an operator norm of at most $\frac{1}{\eta}$).  Suppose that
the Markov operator samples fresh $N$ from Gaussian distribution
$\mcal{N}(0, \sigma^2 \I)$ and corresponds to
$(z, N) \mapsto z - \eta \cdot \nabla f(z) + N$.

Condition~(A) of non-expansion follows by considering the natural coupling
because the mapping $z \mapsto z - \eta \cdot \nabla f(z)$
is a non-expansion (see a proof in~\cite{DBLP:conf/focs/FeldmanMTT18}
and Fact~\ref{fact:gd});
for strongly convex functions,
the parameter $\eta$ is refined such that
the corresponding mapping is strictly contractive
(see~\cite[Lemma 18]{BalleBGG19}).

From
Fact~\ref{fact:zcdp_gaussian},
it follows that condition~(B) of one-step privacy
holds with $C_{\mcal{K}} = \frac{1}{2 \sigma^2}$

\noindent \textbf{Proof Overview of Theorem~\ref{th:privacy_amp}.}
Since we will also use the same proof structure for analyzing
privacy amplification for ADMM,
we will outline the proof in~\cite{BalleBGG19} for completeness
and skip most of the algebraic calculation.
The proof is by induction on $T$.
As in~\cite{BalleBGG19}, we will actually show a more precise upper bound
as follows:

$$\Dz( (K_T \circ \cdots \circ K_1)(z) \| (K_T \circ \cdots \circ K_1)(z'))
\leq C_{\mcal{K}} \cdot \|z - z' \|^2 \cdot T L^{T-1} \cdot \phi_T(L)^2,$$

where the function $\phi_T(L) := \frac{L^{-\frac{1}{2}} - L^{\frac{1}{2}}}{L^{-\frac{T}{2}} - L^{\frac{T}{2}}}$,
for $0 < L < 1$; and we set $\phi_T(1) := \lim_{L \rightarrow 1} \phi_T(L) = \frac{1}{T}$.
Then, the required result follows because $\phi_T(L) \leq \frac{1}{T}$ for all $L \in (0, 1]$.

\noindent \textbf{Base case.} The case $T=1$ is already covered by condition~(B).

\noindent \textbf{Inductive step.}
For $T \geq 2$,
the goal is to show that 
the distributions $(K_T \circ \cdots \circ K_1)(z)$
and $(K_T \circ \cdots \circ K_1)(z')$ are ``close''
with respect to the $\Dz$-divergence.

The idea is to 
define some intermediate $\widehat{z} := (1 - \lambda_T(L)) \cdot z + \lambda_T(L) \cdot z'$
for some appropriate $\lambda_T(L) \in (0,1)$.
For $0 < L < 1$, we set $\lambda_T(L) := \frac{1-L}{1-L^T}$;
for $L = 1$, we set $\lambda_T(1) := \lim_{L \rightarrow 1} \lambda_T(L) = \frac{1}{T}$.

Then, we consider
a triangle inequality~\cite[Theorem 2]{BalleBGG19} involving
the Markov operator $K = K_T$ and
three distributions
$\mu = (K_{T-1} \circ \cdots \circ K_1)(z)$,
$\omega = (K_{T-1} \circ \cdots \circ K_1)(\widehat{z})$
and $\nu = (K_{T-1} \circ \cdots \circ K_1)(z')$ 
that we paraphrase as follows.

\begin{theorem}[Markov Triangle Inequality~\cite{BalleBGG19}]
\label{th:triangle}
Suppose $\mu$, $\omega$ and $\nu$ are distributions
on $\Omega$ such that $\omega$ and $\nu$ have the same support and $K: \Omega \rightarrow \mcal{P}(\Omega)$
is a Markov operator.
Furthermore,
suppose $\pi$ is a coupling from $\mu$ to $\omega$.
Then, the following inequality holds:

$$\Dz(K(\mu) \| K(\nu)) \leq \sup_{(u,w) \in \supp(\pi)} \Dz(K(u)\|K(w)) +  \Dz(\omega \| \nu).$$
\end{theorem}

\begin{remark}
The supremum in the original inequality in~\cite{BalleBGG19}
is slightly more obscure.  The supremum is taken over $w \in \supp(\nu)$,
and $u$ is replaced by the conditional distribution on $\mu$ (with
respect to the coupling $\pi$) given that $w$ is sampled from $\omega$.
Our alternate formulation makes the ``triangle flavor'' of the inequality
more obvious.  Observe that if the
subset $\supp(\omega) \setminus \supp(\nu)$ has non-zero measure,
we will have $\Dz(\omega \| \nu) = + \infty$ and the inequality becomes trivial.
\end{remark}

\noindent \emph{Completion of inductive step.}
Observing that $\|\widehat{z} - z'\| = (1 - \lambda_T(L)) \cdot \|z - z'\|$,
the induction hypothesis gives
$\Dz(\omega \| \nu) \leq C_{\mcal{K}} \cdot (T-1) L^{T-2} \cdot \phi_{T-1}(L)^2 \cdot
(1 - \lambda_T(L))^2 \cdot \|z - z'\|^2$.

Next, observe that $\|z - \widehat{z}\| = \lambda_T(L) \cdot \|z - z'\|$.
Consider the coupling $\pi$ (induced by condition~(A)) from $\mu$ to $\omega$ for applying
$(K_{T-1} \circ \cdots \circ K_1)$ to the two different 
scenarios $z$ and $\widehat{z}$.
Then, applications of condition~(A) for $T-1$ times
implies that for any $(u, w) \in \supp(\pi)$,
$\|u - w \| \leq L^{T-1} \|z - \widehat{z}\|$.
Therefore, condition~(B) implies
that $\Dz(K(u) \| K(w)) \leq C_{\mcal{K}} \cdot L^{2(T-1)} \cdot \lambda_T(L)^2
\cdot \|z - z'\|^2$.  Observe that we have only used
the linearity of the norm $\| \cdot \|$ and do not
need any triangle inequality for the norm.

Applying the Markov triangle inequality in Theorem~\ref{th:triangle},
the inductive step
is completed by observing that
the sum of the two upper bounds
is  $C_{\mcal{K}} \cdot \gamma_T(L) \cdot \|z - z\|^2$,
where  $\gamma_T(L)$ is the following expression:
$\gamma_T(L) :=  (T-1) \cdot L^{T-2} \cdot \phi_{T-1}(L)^2 \cdot
(1 - \lambda_T(L))^2 + L^{2(T-1)} \cdot \lambda_T(L)^2$.

For $L = 1$, we have
$\gamma_T(1) = (T-1) \cdot (\frac{1}{T-1})^2 \cdot (1 - \frac{1}{T})^2 + \frac{1}{T^2} = \frac{1}{T}$.

For $L < 1$, it suffices
to check the calculation that
$\gamma_T(L) = T L^{T-1} \cdot \phi_T(L)^2$,
which is already implicitly done in~\cite{BalleBGG19}.  \qed

\subsection{Technical Challenges for Applying the Framework to ADMM}
\label{sec:challenges}

In view of Theorem~\ref{th:privacy_amp},
it suffices to achieve variants of (A)~non-expansion and (B)~one-step privacy
for ADMM.  We outline our approaches for achieving these variants for ADMM.

\noindent \textbf{(A) Non-expansion Property.}  
Indeed, it is observed~\cite{DBLP:conf/focs/FeldmanMTT18}
that the gradient descent update mapping $x \mapsto x - \eta \nabla f(x)$
is non-expansive, when $f$ is a $\frac{1}{\eta}$-smooth
convex function; in other words,
the mapping is non-expansive as long as $\eta$ is small enough.  This condition on the parameter $\eta$
and the smoothness of $f$ is also used in the standard convergence
proof of gradient descent.  However, the situation for ADMM is not
so straightforward.

\noindent \emph{Counter-example for Non-expansion.}
Consider a special case for Algorithm~\ref{alg:one_iteration}
in which both $g$ and $B$ are zero, and so we can ignore the
variable $y$.  Moreover, we set $A = \I_n$ and $\beta = 1$,
and imagine that $\eta > 0$ is close to 0.  Even though
this gives a trivial optimization problem, the purpose is to show
that each iteration may violate the non-expansion property under the
usual norm.

Consider the two different inputs $(x_0, \lambda_0 = 0)$
and $(x'_0, \lambda'_0 = 0)$
such that only the corresponding first components $x_0 \neq x'_0$ differ.
Because $\eta$ is close to 0,
Lemma~\ref{lemma:step1} implies that
the corresponding outputs satisfy $\|x_1 - x'_1 \| \approx \|x_0 - x'_0\|$.
However, on the other hand,
we have $\|\lambda_1 - \lambda'_1 \| = \beta \|x_0 - x'_0\|$.
Therefore, it follows that the non-expansion property is violated for one iteration under the usual norm. Because $\beta = 1$, we actually have
$\|x_1 - x'_1\|^2 + \|\lambda_1 - \lambda'_1 \| ^2 > 
\frac{3}{2} \cdot (\|x_0 - x'_0\|^2 + \|\lambda_0 - \lambda'_0 \| ^2)$.

\noindent \emph{Customized Norm.}
It is somehow counter-intuitive that
a converging iterative process
can have an iteration that corresponds to a strictly
expanding mapping.
Since there are many known convergence proofs for ADMM~\cite{DBLP:journals/siamnum/HeY12, DBLP:journals/jscic/DengY16, DBLP:conf/icml/OuyangHTG13, DBLP:conf/ijcai/0004K16, DBLP:conf/icml/ZhongK14},
in Section~\ref{sec:non-expansion},
we have shown that
similar conditions
can lead to some non-expansion inequality for each ADMM iteration,
with the surprising twist that it holds 
for a special customized norm that resolves the aforementioned paradox.
For strongly convex objective functions,
we will further refine our customized norm in Section~\ref{sec:strongly_convex} such that
each ADMM iteration corresponds to a strictly contractive mapping.

\noindent \textbf{(B) One-step Privacy.} Adapting one-step privacy
to an ADMM iteration turns out to be more challenging.
Recall that the purpose of privacy amplification by iteration
is to analyze whether the user from the first iteration
can enjoy amplified privacy guarantees from the randomness
used for achieving local privacy for users in subsequent iterations.

Even though there are $x$, $y$ and $\lambda$ variables
involved in ADMM, in Algorithm~\ref{alg:one_iteration},
we see that it is possible to pass only the pair $(x, \lambda)$
between consecutive iterations, because
the variable $y$ can be deterministically recovered from $(x, \lambda)$.  However,
in order to preserve local privacy
for the user in iteration~$t+1$ (whose sensitive data is the function $f_{t+1}$),
it suffices to add noise only in line~\ref{ln:x} to mask the $x$ variable.
This means that given input $(\widetilde{x}_{t}, \lambda_{t})$,
only the first component of the output $(\widetilde{x}_{t+1}, \lambda_{t+1})$
has masking noise, while $\lambda_{t+1}$ is actually
a deterministic function of $(\widetilde{x}_{t}, \lambda_{t})$.
This means that given two different inputs,
the corresponding two output distributions on $(x, \lambda)$ can still have 
a $\Dz$-divergence of $+\infty$, no matter how much noise is added
to mask the $x$ variable.   One might suggest adding noise to the $\lambda$ variable
in every iteration as well to resolve this issue, but this would be considering ``cheating'', because this extra noise is unnecessary for achieving local privacy
for the user in each iteration.

 Hence, at first glance, it seems that
amplification by iteration is impossible for ADMM because the $\lambda$ (and also the $y$)
variables do not receive any masking noise, thereby potentially
leaking information on the private function $f_1$ from the first user.

Indeed, this intuition would be right if the number $m$ of rows in $A$ is much
larger than the dimension~$n$ of $x$.  In some sense, the
linear transformation $A \widetilde{x}_1$ corresponds to many
snapshots of the masked vector~$\widetilde{x}_1$ that
are somehow encoded into the $\lambda$ variable.  In the extreme case
when $m$ is even much larger than the number $T$ of iterations,
it is perceivable that the masked vector~$\widetilde{x}_1$
may be recovered from the final $\lambda_T$ (together
with knowledge of the subsequent functions $f_2, \ldots, f_T$) with little noise, thereby defeating the goal of privacy amplification for the first user.
This is the reason why we first perform Gaussian elimination
in Remark~\ref{remark:transform} to ensure that $m \leq n$ for the matrix~$A$.

\noindent \emph{Incorporating two ADMM iterations into a single Markov operator.}  An innovative idea to resolve the infinity $\Dz$-divergence issue is to consider two ADMM iterations together in a single Markov operator $K \in \mcal{K}$.
Suppose we have some input $(\widetilde{x}_{t}, \lambda_{t})$
to iteration~$t+1$ for the noisy variant of Algorithm~\ref{alg:one_iteration},
but we consider both iterations $t+1$ and $t+2$ together.
We have already mentioned that in iteration~$t+1$,
we only have noise for the masked $\widetilde{x}_{t+1}$,
but $\lambda_{t+1}$ is a deterministic function of $(\widetilde{x}_{t}, \lambda_{t})$.

However, in iteration~$t+2$, since $\lambda_{t+2}$ 
is a deterministic function of $(\widetilde{x}_{t+1}, \lambda_{t+1})$,
the noise from $\widetilde{x}_{t+1}$ can actually be sufficient
to provide a mask for $\lambda_{t+2}$, assuming
that matrix $A$ has the right form as in Remark~\ref{remark:transform}.
Moreover, the fresh randomness from iteration~$t+2$ can mask $\widetilde{x}_{t+2}$.
In Section~\ref{sec:one_step_privacy}, we shall see that the privacy analysis of the Markov operator
(corresponding to both iterations $t+1$ and $t+2$) producing
$(\widetilde{x}_{t+2}, \lambda_{t+2})$ is reminiscent
of a privacy composition proof.

\section{Achieving Privacy Amplification by Iteration for ADMM}
\label{sec:technical}

We give the technical details for
applying the coupling framework described
in Section~\ref{sec:coupling} to achieve
privacy amplification for ADMM.
Recall that the high level goal
is to amplify the privacy guarantee
for the user in the \textbf{first} iteration via
the randomness in subsequent iterations.
As mentioned in Remark~\ref{remark:first_iteration},
we will \textbf{not} exploit the masking randomness (for the
variable $x$) generated in the first iteration in the privacy amplification
analysis.
Therefore, for the purpose of privacy amplification,
the starting points for the two scenarios
are two inputs $(x_0, \lambda_0)$ and $(x'_0, \lambda_0)$,
where $x_0 \neq x'_0$ and the $\lambda$ components are the same.
Below is the main technical result.

\begin{theorem}[Privacy Amplification by Iteration for ADMM]
\label{th:ADMM_privacy}
Suppose given two input scenarios $(x_0, \lambda_0)$
and $(x'_0, \lambda_0)$, a total of $2T$ \emph{noisy} ADMM  iterations
in Algorithm~\ref{alg:one_iteration} are applied
to each input scenario, where for each iteration $t \in [2T]$,
the same function $f_t$ is used in both scenarios
and fresh randomness $N_t$ drawn from Gaussian distribution
$\mcal{N}(0, \sigma^2 \I_n)$ is used to produce
masked $\widetilde{x}_t \gets x_t + N_t$ (that is passed to the next iteration
together with $\lambda_t$).
Then, the corresponding output distributions from the two scenarios satisfy:

$$\Dz((\widetilde{x}_{2T}, \lambda_{2T}) \| (\widetilde{x}'_{2T}, \lambda'_{2T})) \leq \frac{1}{2 \sigma^2}  \frac{C}{T} \cdot \|x_0 - x'_0 \|^2,$$

where $C := \max\{2, \frac{3}{\beta \eta} \} \cdot (1+  \beta \eta \cdot \|A\|^2)$.
\end{theorem}

\ignore{
\noindent \textbf{Remark.}  Observe that in Theorem~\ref{th:ADMM_privacy},
the parameter $C$ has a dependency of $O(\frac{1}{\beta \eta})$ on
$\beta$ and $\eta$.  We will discuss potential improvements
in Section~\ref{sec:conclusion}.
}

\begin{corollary}[Privacy Amplification for the First User]
\label{cor:first_user}
Consider two scenarios for running
$2T+1$ \emph{noisy} ADMM iterations (using
Gaussian noise $\mcal{N}(0, \sigma^2 \I)$ in each iteration to
mask only the $x$ variable),
where the only difference
in the two scenarios
is that the corresponding functions used in the first iteration
can be different neighboring functions $f_1 \sim_\Delta f'_1$ (as
in Definition~\ref{defn:neighbor_functions}).
In other words, the initial $(x_0, \lambda_0)$
and the functions $f_t$ in subsequent iterations $t \geq 2$ are identical
in the two scenarios.

Then, the solutions from the two scenarios
satisfy the following.

\begin{compactitem}

\item \emph{Local Privacy.} After the first iteration, we have

$$\Dz(\widetilde{x}_1 \| \widetilde{x}'_1) \leq \frac{\eta^2 \Delta^2}{2 \sigma^2}.$$

\item \emph{Privacy Amplification.}  After the final iteration $2T+1$, we have

$$\Dz((\widetilde{x}_{2T+1}, \lambda_{2T+1}) \| (\widetilde{x}'_{2T+1}, \lambda'_{2T+1})) \leq \frac{C}{T} \cdot\frac{\eta^2 \Delta^2}{2 \sigma^2},$$

where $C := \max\{2, \frac{3}{\beta \eta} \} \cdot (1+  \beta \eta \cdot \|A\|^2)$.
\end{compactitem}
\end{corollary}

\begin{proof}
For local privacy,
observe that by Lemma~\ref{lemma:step1}, we have

 $\|x_1 - x'_1 \| =
\eta \| (\I + \eta \beta A^\top A)^{-1} (\nabla f_1(x_0) - \nabla f'_1(x_0))\|
\leq \eta \|\nabla f_1(x_0) - \nabla f'_1(x_0)\|
\leq \eta \Delta$,

\noindent where the first inequality follows because all eignevalues of
$(\I + \eta \beta A^\top A)^{-1}$ are positive and at most~1.
The last inequality
follows from the assumption that $f_1 \sim_{\Delta} f'_1$.

The final result follows from Theorem~\ref{th:ADMM_privacy}.
\ignore{
Since we add Gaussian noise in $\mcal{N}(0, \sigma^2 \I_n)$, Fact~\ref{fact:gd} implies that
$\Dz(\widetilde{x}_1 \| \widetilde{x}'_1) \leq \frac{\eta^2 \Delta^2}{2 \sigma^2}$.

For privacy amplification, observe that
the natural coupling in the first iteration implies that
$\W_*((\widetilde{x}_1, \lambda_1), (\widetilde{x}'_1, \lambda'_1))
\leq \eta \Delta$.  Theorem~\ref{th:ADMM_privacy}
gives the required result.
}
\end{proof}


As described in Section~\ref{sec:challenges},
the main technical challenges
are how to achieve (A) non-expansion (in Section~\ref{sec:non-expansion})
and (B) one-step privacy (in Section~\ref{sec:one_step_privacy}).
After achieving those two key properties,
we will show how everything fits together
to achieve Theorem~\ref{th:ADMM_privacy}
in Section~\ref{sec:combining}.

\subsection{Achieving Non-expansion via Customized Norm}
\label{sec:non-expansion}

As discussed in Section~\ref{sec:challenges},
one ADMM iteration as in Algorithm~\ref{alg:one_iteration}
may produce a strictly non-expanding mapping under the usual norm.
We consider the following specialized norm.

\begin{definition}[Customized Norm]
\label{defn:norm}
Using the ADMM parameters $\eta$ and $\beta$
from Section~\ref{sec:admm_prelim}, we define a customized norm.
For $(x, \lambda) \in \R^n \times \R^m$,

$$\|(x, \lambda)\|^2_* := \|x\|^2 + \frac{\eta}{\beta} \cdot \| \lambda - \beta A x \|^2.$$
\end{definition}

\noindent \textbf{Remark.} Even though
we use the term ``norm'', we only need the
linearity property, i.e., for all $a \in \R$,
$\|(a x, a \lambda)\|^2_* = a^2 \cdot \|(x, \lambda)\|^2_*$.
Note that we do not need any triangle inequality
for the customized norm.

\noindent \textbf{Smoothness Assumption.}
As in the normal gradient descent,
the parameter $\eta > 0$ is related to the smoothness
of the function~$f$ used in each ADMM iteration.
Recall that a differentiable function~$f$ is $L$-smooth
if for all $x, x'$, $\|\nabla f(x) - \nabla f(x') \| \leq L \|x - x'\|$.
As in~\cite{DBLP:conf/focs/FeldmanMTT18}, we also use the following well-known fact.

\begin{fact}[Gradient Descent Update is Non-expansive]
\label{fact:gd}
Suppose $f$ is a differentiable
convex function that is $\frac{1}{\eta}$-smooth.
Then, the mapping $\varphi(x) := x - \eta \nabla f(x)$
is non-expansive, i.e., for all $x, x'$,
$\|\varphi(x) - \varphi(x')\| \leq \|x - x'\|$.
\end{fact}

\begin{proof}
The proof follows readily from the following
inequality~\cite[Theorem 4.2.2]{hiriart2013convex}
that holds because $f$ is $\frac{1}{\eta}$-smooth and convex:
$
\langle\nabla f(x)-\nabla(x'),x-x'\rangle\geq \eta \|\nabla f(x)-\nabla f(x')\|^2. $

Applying this inequality below, we have:
\begin{align}
&\|\varphi(x)-\varphi(x')\|^2=\|x-x'\|^2-2\eta\langle \nabla f(x)-\nabla f(x'),
x-x'\rangle + \eta^2\|\nabla f(x)-\nabla f(x')\|^2 \notag \\
& \leq \|x-x'\|^2 - \eta^2 \|\nabla f(x)-\nabla f(x')\|^2\leq \|x-x'\|^2, \notag
\end{align}
as required.
\end{proof}

\begin{lemma}[ADMM Iteration is Non-expansive with Customized Norm]
\label{lemma:ADMM_exp}
Suppose one ADMM iteration in Algorithm~\ref{alg:one_iteration}
is applied to two different inputs $(x_t, \lambda_t)$
and $(x'_t, \lambda'_t)$ with the same
function $f$ that is convex and $\frac{1}{\eta}$-smooth.
Then, the corresponding two outputs
$(x_{t+1}, \lambda_{t+1})$
and $(x'_{t+1}, \lambda'_{t+1})$
satisfy:

$$ \|(x_{t+1} - x'_{t+1}, \lambda_{t+1} - \lambda'_{t+1})\|^2_*
\leq
\|(x_{t} - x'_{t}, \lambda_{t} - \lambda'_{t})\|^2_*.$$
\end{lemma}

\begin{proof}
    During this proof we will often consider the differences between two collections of variables
    that are distinguished using superscripts, e.g., $z$ vs $z'$.
    To simplify the proof, we use $\mathfrak{d} z := z - z'$ to indicate the difference between the
    original variable and the one with the superscript.
    For instance, when we have two variables $x_t$ and $x'_t$, we write $\mathfrak{d} x_t = x_t - x'_t$.
    Also Fact~\ref{fact:gd} can be rewritten as $\|\mathfrak{d} \varphi(x)\|\leq \|\mathfrak{d}x\|$.

For the input scenario $(x_t, \lambda_t)$,
recall that Algorithm~\ref{alg:one_iteration} using
function~$f$ consists of the following calculation.

\begin{align}
\ts    y_{t} :=& \arg\min_{y \in \R^\ell} \left\{g(y) + \frac{\beta}{2}\left\| Ax_{t} + By - c - \frac{1}{\beta}\lambda_t \right\|^2 \right\} \label{eq:inequalities:update_y}\\
 \ts   \lambda_{t+1} :=& \lambda_t - \beta \left( Ax_{t} + By_{t} - c \right) \label{eq:inequalities:update_lambda}\\
 \ts   x_{t+1} :=& \arg\min_{x \in \R^n} \left\{ f(x_t) + \langle \nabla f(x_t), x - x_t \rangle + \frac{1}{2\eta} \left\| x - x_t \right\|^2 + \frac{\beta}{2}\left\| Ax + By_t - c - \frac{1}{\beta}\lambda_{t+1} \right\|^2 \right\} \label{eq:inequalities:update_x}
\end{align}

Similarly, quantities $y'_t$, $\lambda'_{t+1}$ and $x'_{t+1}$
are defined for the input scenario $(x'_t, \lambda'_t)$.

\noindent \textbf{Inequalities from the Optimality of $y_t$.}

As in Lemma~\ref{lemma:step2}, the optimality for choosing $y_t$ in \eqref{eq:inequalities:update_y} implies the following.

\begin{align}
    0 \in \partial g(y_t) + \beta B^\top\left( Ax_{t} + By_t - c - \frac{1}{\beta}\lambda_t \right)\notag\\
    0 \in \partial g(y_t') + \beta B^\top\left( Ax_{t}' + By_t' - c - \frac{1}{\beta}\lambda_t' \right) \notag
\end{align}

Since $g$ is convex,
we have the monotonicity of subgradient,
i.e., $s \in \partial g(y_t)$
and $s' \in \partial g(y'_t)$ implies that
$\langle s - s', y_t - y'_t \rangle \geq 0$.
After rearranging and using \eqref{eq:inequalities:update_lambda}, we have

\begin{align}
\left\langle \mathfrak{d}\lambda_{t+1}, B\mathfrak{d}y_t \right\rangle =
\left\langle \mathfrak{d}\lambda_t - \beta(A\mathfrak{d}x_t + B\mathfrak{d}y_t), B\mathfrak{d}y_t \right\rangle \ge 0 \label{eq:inequalities:subgradient_mononicity}
\end{align}

\noindent \textbf{Inequalities from the Optimality of $x_{t+1}$.}

From Lemma~\ref{lemma:step1},
the optimality of $x_{t+1}$ from \eqref{eq:inequalities:update_x} implies
the following.

\begin{align}
    \left\langle \nabla f(x_t) + \frac{1}{\eta} (x_{t+1} - x_t) + \beta A^\top \left( Ax_{t+1} + By_t - c - \frac{1}{\beta}\lambda_{t+1} \right), x - x_{t+1} \right\rangle = 0, & \forall x \in \R^n  \label{eq:inequalities:first-order-x}\\
    \left\langle \nabla f(x_t') + \frac{1}{\eta} (x_{t+1}' - x_t') + \beta A^\top \left( Ax_{t+1}' + By_t' - c - \frac{1}{\beta}\lambda_{t+1}' \right), x - x_{t+1}' \right\rangle = 0, & \forall x \in \R^n  \label{eq:inequalities:first-order-x'}
\end{align}

Substituting $x = x_{t+1}'$ in \eqref{eq:inequalities:first-order-x} and $x = x_{t+1}$ in \eqref{eq:inequalities:first-order-x'}, \eqref{eq:inequalities:first-order-x} $+$ \eqref{eq:inequalities:first-order-x'} and rearranging give:

\begin{align}
    \begin{split}
        \left\langle \mathfrak{d}x_t - \eta \mathfrak{d}\nabla f(x_t) , \mathfrak{d}x_{t+1} \right\rangle - \left\| \mathfrak{d}x_{t+1} \right\|^2
        + \eta \beta \left\langle \frac{1}{\beta} \mathfrak{d}\lambda_{t+1} - \mathfrak{d}Ax_{t+1} - \mathfrak{d}By_t, A\mathfrak{d}x_{t+1}\right\rangle = 0
    \end{split}\label{eq:inequalities:update-x-add-update-x'}
\end{align}

We first analyze the third term from \eqref{eq:inequalities:update-x-add-update-x'}.

\begin{align}
    & \left\langle \frac{1}{\beta} \mathfrak{d}\lambda_{t+1} - \mathfrak{d}Ax_{t+1} - \mathfrak{d}By_t, A\mathfrak{d} x_{t+1} \right\rangle \notag \\
    =& \left\langle \frac{1}{\beta} \mathfrak{d}\lambda_{t+1} - \mathfrak{d}Ax_{t+1} , \mathfrak{d}Ax_{t+1} + \mathfrak{d}By_t\right\rangle \notag
    - \left\langle \frac{1}{\beta} \mathfrak{d}\lambda_{t+1}, \mathfrak{d}By_t\right\rangle \notag \\
    \le& \left\langle \frac{1}{\beta} \mathfrak{d}\lambda_{t+1} - \mathfrak{d}Ax_{t+1}, \mathfrak{d}Ax_{t+1} + \mathfrak{d}By_t \right\rangle \tag{using \eqref{eq:inequalities:subgradient_mononicity}} \\
    =& \left\langle \frac{1}{\beta} \mathfrak{d}\lambda_{t+1}- \mathfrak{d}Ax_{t+1}, \frac{1}{\beta} \mathfrak{d}\lambda_{t+1}  + \mathfrak{d}By_t\right\rangle \notag
    - \left\|\frac{1}{\beta} \mathfrak{d}\lambda_{t+1} - \mathfrak{d}Ax_{t+1}\right\|^2 \notag\\
    =& \left\langle \frac{1}{\beta} \mathfrak{d}\lambda_{t+1} - \mathfrak{d}Ax_{t+1}, \frac{1}{\beta} \mathfrak{d}\lambda_{t} - \mathfrak{d}Ax_{t}\right\rangle \tag{using \eqref{eq:inequalities:update_lambda}}
    - \left\|\frac{1}{\beta} \mathfrak{d}\lambda_{t+1}- \mathfrak{d}Ax_{t+1}\right\|^2 \notag
\end{align}

Together with \eqref{eq:inequalities:update-x-add-update-x'}, we have
\begin{align}
    &\left\langle \mathfrak{d}x_t -  \eta\mathfrak{d}\nabla f(x_t) , \mathfrak{d}x_{t+1}\right\rangle \notag
    +  \eta \beta \left\langle \frac{1}{\beta} \mathfrak{d}\lambda_{t+1} - \mathfrak{d}Ax_{t+1}, \frac{1}{\beta} \mathfrak{d}\lambda_{t} - \mathfrak{d}Ax_{t}\right\rangle \\
    &\ge \left\| \mathfrak{d}x_{t+1}\right\|^2 + \eta \beta \left\|\frac{1}{\beta} \mathfrak{d}\lambda_{t+1} - \mathfrak{d}Ax_{t+1}\right\|^2 \label{eq:int2}
\end{align}

Using the Cauchy-Schwarz Inequality and Fact~\ref{fact:gd} that $\left\| \mathfrak{d}x_t - \eta \mathfrak{d}\nabla f(x_t) \right\| \le \left\| \mathfrak{d}x_t\right\|$ (because
$f$ is $\frac{1}{\eta}$-smooth), we have an upper bound
for the first term in \eqref{eq:int2}:
\begin{align}
    \left\| \mathfrak{d}x_{t+1} \right\|^2 + \eta\beta \left\|\frac{1}{\beta} \mathfrak{d}\lambda_{t+1}- \mathfrak{d}Ax_{t+1}\right\|^2
    \le \left\| \mathfrak{d}x_{t}\right\| \cdot \left\| \mathfrak{d}x_{t+1}\right\| + \eta\beta \left\|\frac{1}{\beta} \mathfrak{d}\lambda_{t} - \mathfrak{d}Ax_{t}\right\| \cdot \left\|\frac{1}{\beta} \mathfrak{d}\lambda_{t+1} - \mathfrak{d}Ax_{t+1}\right\| \notag
\end{align}

\noindent Denoting $P = \left \| \mathfrak{d}x_{t+1}\right\|^2$,
$Q = \eta\beta \left\|\frac{1}{\beta} \mathfrak{d}\lambda_{t+1} - \mathfrak{d}Ax_{t+1}\right\|^2$,
$R = \left\| \mathfrak{d}x_{t}\right\|^2$, and
$S = \eta\beta \left\|\frac{1}{\beta} \mathfrak{d}\lambda_{t} - \mathfrak{d}Ax_{t}\right\|^2$,
the above inequality becomes
$$P + Q \leq \sqrt{P} \cdot \sqrt{R} + \sqrt{Q} \cdot \sqrt{S}.$$

Finally, applying the Cauchy-Schwarz inequality again,
we have

$$\sqrt{P} \cdot \sqrt{R} + \sqrt{Q} \cdot \sqrt{S} \leq
\sqrt{P + Q}  \cdot \sqrt{R + S}.$$

This gives the required inequality:

$$ \left\|(\mathfrak{d}x_{t+1} , \mathfrak{d}\lambda_{t+1} ) \right\|^2_* =
P+Q
\leq R + S =
\left\|(\mathfrak{d}x_{t}, \mathfrak{d}\lambda_{t})\right\|^2_*.$$

\ignore{
Using Cauchy Schwarz Inequality again, we obtain one of the most important inequalities in this paper:

\begin{align}
    \left\| x_{t+1}' - x_{t+1} \right\|^2 + \eta\beta \left\|\frac{1}{\beta} \lambda_{t+1}' - Ax_{t+1}' - \frac{1}{\beta} \lambda_{t+1} + Ax_{t+1}\right\|^2 \le \left\| x_{t}' - x_{t} \right\|^2 + \eta\beta \left\|\frac{1}{\beta} \lambda_{t}' - Ax_{t}' - \frac{1}{\beta} \lambda_{t} + Ax_{t}\right\|^2 \label{eq:inequalities:upper-bound-updating-x-lambda}
\end{align}
}

\end{proof}

\subsection{Achieving One-Step Privacy}
\label{sec:one_step_privacy}

As described in Section~\ref{sec:challenges},
the randomness in one noisy ADMM iteration is not sufficient
to achieve one-step privacy (condition~(B) in
Theorem~\ref{th:privacy_amp})
because only the $x$ component of the output
of Algorithm~\ref{alg:one_iteration} is masked with noise,
while the $\lambda$ component is totally exposed.

\noindent \textbf{Incorporating two noisy ADMM iterations
into a single Markov Operator.} Our novel idea
is to let each Markov operator represent two
ADMM iterations.  For instance, an operator $K \in \mcal{K}$
in the collection corresponds to
iterations~$t+1$ and $t+2$,
which use two $\frac{1}{\eta}$-smooth convex functions $f_{t+1}$ and
$f_{t+2}$, respectively.  Given some input $(\widetilde{x}_t, \lambda_t)$,
the application $K(\widetilde{x}_t, \lambda_t)$ of the operator $K$
may be described with the following randomized process
(whose source of randomness is two independent copies of Gaussian noise with some appropriate variance $\sigma^2$).

\begin{compactitem}

\item[1.] Apply Algorithm~\ref{alg:one_iteration}
with input $(\widetilde{x}_t, \lambda_t)$ and function $f_{t+1}$
to produce $(x_{t+1}, \lambda_{t+1})$.

\item[2.] Sample fresh $N_{t+1}$ from Gaussian $\mcal{N}(0, \sigma^2 \I_n)$
to produce masked $\widetilde{x}_{t+1} \gets x_{t+1} + N_{t+1}$.

\item[3.] Apply Algorithm~\ref{alg:one_iteration}
with input $(\widetilde{x}_{t+1}, \lambda_{t+1})$ and function $f_{t+2}$
to produce $(x_{t+2}, \lambda_{t+2})$.

\item[4.] Sample fresh $N_{t+2}$ from Gaussian $\mcal{N}(0, \sigma^2 \I_n)$
to produce masked $\widetilde{x}_{t+2} \gets x_{t+2} + N_{t+2}$.

\item[5.] Return the pair $(\widetilde{x}_{t+2}, \lambda_{t+2})$.

\end{compactitem}

Recall when we compare two random processes that both sample
$(N_t, N_{t+1})$ from the same distribution,
the natural coupling refers to sharing the same sampled $(N_t, N_{t+1})$
in the two processes.

\noindent \textbf{Proof Setup.}
Given two input scenarios $(\widetilde{x}_t, \lambda_t)$
and $(\widetilde{x}'_t, \lambda'_t)$,
our goal is to derive an upperbound for
the divergence $\Dz(K(\widetilde{x}_t, \lambda_t) \| K(\widetilde{x}'_t, \lambda'_t))$
that is defined in Section~\ref{sec:privacy_prelim}.

Recall that in Remark~\ref{remark:transform},
we have transformed the problem by Gaussian elimination such that
$A = [\I_m D]$ for some $m \times (n-m)$ matrix $D$.
In our privacy analysis, we do not actually need to use the randomness
of all $n$ coordinates of $N_{t+1}$ (but we
still need all $n$ coordinates of $N_{t+2}$).
We use $U_{t+1} \sim \mcal{N}(0, \sigma^2 \I_m)$ to represent
the first $m$ coordinates of $N_{t+1}$.
In both scenarios, we will fix the last $n-m$ coordinates
of $N_{t+1}$ and denote this common part as $\mathfrak{z}  \in \R^{n-m}$
(which is no longer random).
By Fact~\ref{fact:zCDP}(b), any uniform upperbound on the $\Dz$-divergence
after conditioning on $\mathfrak{z}$ will also be an upperbound
for the original divergence.
Observe that $A N_{t+1} = U_{t+1} + D \mathfrak{z}$.

\noindent \textbf{Expressing Markov Operator $K$ as
an Adaptive Composition of Two Private Mechanisms.}
Instead of directly working with probability density
function of $(\widetilde{x}_{t+2}, \lambda_{t+2})$  in the analysis of
$\Dz$-divergence, we will use properties of
$\Dz$-divergence in Fact~\ref{fact:zCDP},
whose proofs in the literature have already incorporated
the technical manipulation of integrals.
Our approach is to analyze the divergence in the language
of adaptive composition of private mechanisms
with which most readers have some familiarity.

\noindent \emph{How to decompose $(\widetilde{x}_{t+2}, \lambda_{t+2})$?}
No matter whether one wants to directly analyze the probability
density functions or make use of adaptive composition,
one technical hurdle is that given one component of the pair
$(\widetilde{x}_{t+2}, \lambda_{t+2})$,
the conditional distribution of the other component is not easy to analyze.
In the language of mechanism composition,
this means that it is complicated to describe a randomized mechanism
that takes one given component and returns the other component (even with
access to the original input $(\widetilde{x}_{t}, \lambda_{t})$).

From the input to $\mcal{G}$ in Lemma~\ref{lemma:step2} and
the customized norm of Definition~\ref{defn:norm},
one might guess that it is convenient
to work with another solution space via a bijective mapping $(x, \lambda) \leftrightarrow (x, w)$, where $\lambda$ and $w$ are linked by $w = \lambda - \beta A x$.
Indeed, it is possible to rephrase an ADMM iteration in Algorithm~\ref{alg:one_iteration}
using the space of $(x, w)$ variables.
However, in this case, in one noisy ADMM iteration taking input
$(\widetilde{x}_t, \widetilde{w}_t)$,
both components in the output
$(\widetilde{x}_{t+1}, \widetilde{w}_{t+1})$ will
share the sampled randomness $N_{t+1}$.
Conceptually, it would be slightly more indirect to explain
why one noisy ADMM iteration is not sufficient to achieve one-step privacy.

Therefore, we decide to mainly work with variables in the $(x, \lambda)$ space.
For the composition, it actually suffices
to consider an intermediate variable ${w}_{t+1} = \lambda_{t+1} - \beta A x_{t+1}$
and its masked variant
$\widetilde{w}_{t+1} = \lambda_{t+1} - \beta A (x_{t+1} + N_{t+1})$.
We shall see in Algorithm~\ref{alg:m2}
that $\lambda_{t+2}$ is a deterministic
function of $\widetilde{w}_{t+1}$.
Hence, by the data processing inequality in Fact~\ref{fact:zCDP}(a),
it suffices to analyze a (randomized) composition that
takes input $(\widetilde{x}_t, \lambda_t)$ and returns
the pair $(\widetilde{w}_{t+1}, \widetilde{x}_{t+2})$.

\noindent \emph{Adaptive Composition.}  Recall
that we are analyzing the Markov operator $K$
while conditioning on the last $n-m$ coordinates
$\mathfrak{z}$ of $N_{t+1}$.  We will paraphrase $K$
as an adaptive composition
of two (randomized) mechanisms
$\mcal{M}_1: \R^n \times \R^m \rightarrow \R^m$
and $\mcal{M}_2: (\R^n \times \R^m) \times \R^m \rightarrow \R^n$.
Given some input $(\widetilde{x}_t, \lambda_t)$, the composition
$(\mcal{M}_2 \circ \mcal{M}_1)(\widetilde{x}_t, \lambda_t)$ is interpreted as the following random
process.

\begin{compactitem}
\item[1.] In Algorithm~\ref{alg:m1}, sample fresh randomness $U_{t+1}$ in
$\mcal{M}_1(\widetilde{x}_t, \lambda_t)$ to get $\widetilde{w}_{t+1}$.

\item[2.] In Algorithm~\ref{alg:m2}, sample fresh randomness $N_{t+2}$ in $\mcal{M}_2(\widetilde{x}_t, \lambda_t, \widetilde{w}_{t+1})$ to get $\widetilde{x}_{t+2}$.

\item[3.] Return the pair $(\widetilde{w}_{t+1}, \widetilde{x}_{t+2})$.
\end{compactitem}

\begin{lemma}[Equivalence of Random Processes]
\label{lemma:equiv}
For the common input~$(\widetilde{x}_t, \lambda_t)$,
by considering the natural coupling (i.e.,
using the same randomness $U_{t+1}$ and $N_{t+2}$),
the output $(\widetilde{w}_{t+1}, \widetilde{x}_{t+2})$
returned by the composition
$(\mcal{M}_2 \circ \mcal{M}_1)(\widetilde{x}_t, \lambda_t)$
can be transformed deterministically
to the output $(\widetilde{x}_{t+2}, \lambda_{t+2})$
returned by $K(\widetilde{x}_t, \lambda_t)$.
\end{lemma}

\begin{proof}
Observe that Algorithm~\ref{alg:m2} simulates two ADMM iterations starting
from input $(\widetilde{x}_t, \lambda_t)$,
except that $U_{t+1}$ is determined by $\widetilde{w}_{t+1}$,
instead of sampling from $\mcal{N}(0, \sigma^2 \I_m)$.
On the other hand, $N_{t+2}$ is sampled from $\mcal{N}(0, \sigma^2 \I_n)$
in the same way as the operator $K$.

However, observe that the component $\widetilde{w}_{t+1}$
of the input to Algorithm~\ref{alg:m2} is generated
by Algorithm~\ref{alg:m1}, in which the fresh copy $U_{t+1}$
is sampled from $\mcal{N}(0, \I_m)$.
Therefore, it suffices to check that line~\ref{ln:rand_reconstruct}
of Algorithm~\ref{alg:m2} indeed ``reconstructs''
the randomness $U_{t+1}$ generated in Algorithm~\ref{alg:m1} correctly.
Note that the first 4 lines of Algorithms~\ref{alg:m1}
and~\ref{alg:m2} are identical, which means
that the 4 associated variables are also the same across
both algorithms.

Line~\ref{ln:output_m1} of Algorithm~\ref{alg:m1}
implies that the sampled randomness $U_{t+1}$ satisfies:

$\widetilde{w}_{t+1} = w_{t+1} - \beta D \mathfrak{z} -  \beta U_{t+1}$,
which is consistent with line~\ref{ln:rand_reconstruct}
of Algorithm~\ref{alg:m2}.

Therefore, by considering the natural coupling via $(U_{t+1}, N_{t+2})$,
it follows that
$(\mcal{M}_2 \circ \mcal{M}_1)(\widetilde{x}_t, \lambda_t)$
can simulate $K(\widetilde{x}_t, \lambda_t)$.
\end{proof}

\noindent \emph{Neighboring Inputs.}  For readers
that are more familiar with privacy composition proofs,
it suffices to consider $(\widetilde{x}_t, \lambda_t)$
and $(\widetilde{x}'_t, \lambda'_t)$ as the only possible neighboring inputs
when we analyze the composition $\mcal{M}_2 \circ \mcal{M}_1$.

\begin{algorithm}
    \caption{Mechanism $\M_1$}
    \label{alg:m1}

    \KwIn{
        $(\widetilde{x}_t, \lambda_t)$ and fixing the last
				$n-m$ coordinates of $N_{t+1}$ to be $\mathfrak{z}$
	
    }

    \KwOut{$\widetilde{w}_{t+1}$
    }

        $y_{t} \gets \mcal{G}(\lambda_{t} - \beta A \widetilde{x}_{t})$

        $\lambda_{t+1} \gets \lambda_{t} - \beta (A \widetilde{x}_{t} + B y_{t} - c)$

        $x_{t+1} \gets \mcal{F}^{\nabla f_{t+1}}(\widetilde{x}_{t}, y_{t},\lambda_{t+1})$
        \hfill	\Comment{first 3 lines same as Algorithm~\ref{alg:one_iteration}}


        $w_{t+1} \gets \lambda_{t+1} - \beta A{x}_{t+1} $

				Sample fresh $U_{t+1}$ from $\mcal{N}(0, \sigma^2 \I_m)$.

    \KwRet{
        $\widetilde{w}_{t+1} \gets w_{t+1} - \beta D \mathfrak{z} -  \beta U_{t+1}$
    } \label{ln:output_m1}

\end{algorithm}

\begin{algorithm}
    \caption{Mechanism $\M_2$}
    \label{alg:m2}

    \KwIn{
        $(\widetilde{x}_t, \lambda_t, \widetilde{w}_{t+1})$
				and fixing the last
				$n-m$ coordinates of $N_{t+1}$ to be $\mathfrak{z}$
    }

    \KwOut{
        $\widetilde{x}_{t+2}$
    }

				$y_{t} \gets \mcal{G}(\lambda_{t} - \beta A \widetilde{x}_{t})$

        $\lambda_{t+1} \gets \lambda_{t} - \beta (A \widetilde{x}_{t} + B y_{t} - c)$

        $x_{t+1} \gets \mcal{F}^{\nabla f_{t+1}}(\widetilde{x}_{t}, y_{t},\lambda_{t+1})$

        $w_{t+1} \gets \lambda_{t+1} - \beta A{x}_{t+1} $
				\hfill	\Comment{first 4 lines the same as Algorithm~\ref{alg:m1}}

				$U_{t+1} \gets \frac{1}{\beta} (w_{t+1} -  \widetilde{w}_{t+1}) - D \mathfrak{z}$ \label{ln:rand_reconstruct}
				\hfill \Comment{randomness ``reconstruction''; $A = [\I_m \, D]$}

        $\widetilde{x}_{t+1} \gets x_{t+1} +  (U_{t+1},  \mathfrak{z})$
				\label{ln:newx}
        				\hfill \Comment{$N_{t+1} = (U_{t+1}, \mathfrak{z})$;
								$\widetilde{w}_{t+1} = \lambda_{t+1} - \beta A \widetilde{x}_{t+1}$}


        $y_{t+1} \gets \mcal{G}(\widetilde{w}_{t+1})$
        \hfill	\Comment{2nd iteration of ADMM Algorithm~\ref{alg:one_iteration}
				with input $(\widetilde{x}_{t+1}, \lambda_{t+1})$}

        $\lambda_{t+2} \gets \widetilde{w}_{t+1} - \beta (B y_{t+1} - c)$
				\hfill \Comment{$\lambda_{t+2}$ is a deterministic function of $\widetilde{w}_{t+1}$}

        $x_{t+2} \gets \mcal{F}^{\nabla f_{t+2}}(\widetilde{x}_{t+1}, y_{t+1},\lambda_{t+2})$
        \hfill	\Comment{oracle access to $\nabla f_{t+2}(\cdot)$}

        Sample fresh $N_{t+2}$ from $\mcal{N}(0, \sigma^2\I_n)$.

    \KwRet $\widetilde{x}_{t+2} \gets x_{t+2} + N_{t+2}$

\end{algorithm}

\noindent \textbf{Divergence Analysis.}  Recall that
we consider two inputs $(\widetilde{x}_t, \lambda_t)$
and $(\widetilde{x}'_t, \lambda'_t)$.  Observe that we use a superscript to indicate
variables associated with the second input.

\begin{lemma}[Privacy for $\mcal{M}_1$]
    \label{lemma:m1_privacy}
Given inputs~$(\widetilde{x}_t, \lambda_t)$ and $(\widetilde{x}'_t, \lambda'_t)$,
we have:
		
\begin{align}
        \Dz \left(\M_1(\widetilde{x}_t, \lambda_t) \bigl\| \M_1(\widetilde{x}'_t, \lambda'_t) \right) = \frac{1}{2\sigma^2} \left\| \frac{1}{\beta}\lambda_{t+1} - Ax_{t+1} - \frac{1}{\beta}\lambda'_{t+1} +  Ax'_{t+1} \right\|^2, \notag
    \end{align}
		
where the associated variables are defined in the first 4 lines
of Algorithm~\ref{alg:m1}.
\end{lemma}

\begin{proof}
Observe that for the two inputs,
the noise~$\beta U_{t+1}$ has distribution $\mcal{N}(0, \beta^2 \sigma^2 \I_m)$,
which is used to mask the corresponding two values $w_{t+1} - \beta D \mathfrak{z}$
and $w'_{t+1} - \beta D \mathfrak{z}$.

By Fact~\ref{fact:zcdp_gaussian}, we have
    \begin{align}
        \Dz\left(\widetilde{w}_{t+1}\|\widetilde{w}_{t+1}'\right)
				= \frac{1}{2\sigma^2\beta^2} \left\| w_{t+1} -  w'_{t+1} \right\|^2
				= \frac{1}{2\sigma^2\beta^2} \left\| \lambda_{t+1} - \beta Ax_{t+1} - \lambda'_{t+1} + \beta Ax'_{t+1} \right\|^2, \notag
    \end{align}
		
as required.
\end{proof}

\begin{lemma}[Privacy for $\mcal{M}_2$]
    \label{lemma:m2_privacy}
		Given some $\widetilde{w}$ and inputs~$(\widetilde{x}_t, \lambda_t)$ and $(\widetilde{x}'_t, \lambda'_t)$,

    \begin{align}
        \Dz \left(\M_2(\widetilde{x}_t, \lambda_t, \widetilde{w}) \bigl\| \M_2(\widetilde{x}'_t, \lambda'_t, \widetilde{w}) \right) \le \frac{2}{2\sigma^2}\left(\left\| x_{t+1} - x'_{t+1} \right\|^2 + \left\| \frac{1}{\beta}\lambda_{t+1} - Ax_{t+1} - \frac{1}{\beta}\lambda'_{t+1} + Ax'_{t+1} \right\|^2\right). \notag
    \end{align}
\end{lemma}

\begin{proof}
By Fact~\ref{fact:zcdp_gaussian},
we have $\Dz( \widetilde{x}_{t+2} \| \widetilde{x}'_{t+2} )
= \frac{1}{2 \sigma^2} \| {x}_{t+2}  - {x}'_{t+2}  \|^2$,
which we will analyze in the rest of the proof.
Next, observe that $\widetilde{w} = \widetilde{w}_{t+1} =
\widetilde{w}'_{t+1}$ is common to both instances.

Therefore, we have $\lambda_{t+1} - \beta A \widetilde{x}_{t+1}
= \widetilde{w} = \lambda'_{t+1} - \beta A \widetilde{x}'_{t+1}$.
Hence, we have:

\begin{align}
        \left\| x_{t+2} - x'_{t+2}\right\|^2 \le
				\left\| (x_{t+2} - x'_{t+2}, \lambda_{t+2} - \lambda'_{t+2}) \right \|^2_* \le
				\left\| (\widetilde{x}_{t+1} - \widetilde{x}'_{t+1}, \lambda_{t+1} - \lambda'_{t+1}) \right \|^2_* =				
				\left\| \widetilde{x}_{t+1} - \widetilde{x}'_{t+1} \right\|^2,
				\notag
    \end{align}

where the second inequality follows from Lemma~\ref{lemma:ADMM_exp},
because $(x_{t+2}, \lambda_{t+2})$ is the
output of one ADMM iteration Algorithm~\ref{alg:one_iteration}
with input $(\widetilde{x}_{t+1}, \lambda_{t+1})$.

    By line 5-6 of Algorithm~\ref{alg:m2}
		and using the inequality $\|a +  b \|^2 \leq 2 (\|a\|^2 + \|b\|^2)$, we have:
    \begin{align}
        \left\| \widetilde{x}_{t+1} - \widetilde{x}'_{t+1} \right\|^2 &\le 2 \left(\left\| x_{t+1} - x'_{t+1} \right\|^2 + \left\| \frac{1}{\beta} (w_{t+1} - \widetilde{w}) - \frac{1}{\beta} (w'_{t+1} - \widetilde{w}) \right\|^2 \right) \notag \\
        &= 2\left(\left\| x_{t+1} - x'_{t+1} \right\|^2 + \left\| \frac{1}{\beta}\lambda_{t+1} - Ax_{t+1} - \frac{1}{\beta}\lambda'_{t+1} + Ax'_{t+1} \right\|^2\right). \notag
    \end{align}

    Combining the above inequalities, we have
    \begin{align}
      \Dz\left(  \widetilde{x}_{t+2} \bigl\| \widetilde{x}'_{t+2} \right)  \le\frac{2}{2\sigma^2}\left(\left\| x_{t+1} - x'_{t+1} \right\|^2 + \left\| \frac{1}{\beta}\lambda_{t+1} - Ax_{t+1} - \frac{1}{\beta}\lambda'_{t+1} + Ax'_{t+1} \right\|^2\right). \notag
    \end{align}
\end{proof}

\begin{lemma}[Privacy Composition for $\mcal{M}_2 \circ \mcal{M}_1$]\label{lemma:composition_privacy}
Suppose when inputs $(\widetilde{x}_t, \lambda_t)$ and $(\widetilde{x}'_t, \lambda'_t)$
are given to the composition $\mcal{M}_2 \circ \mcal{M}_1$,
intermediate variables $({x}_{t+1}, \lambda_{t+1})$
and $({x}'_{t+1}, \lambda'_{t+1})$ are produced as in
Algorithms~\ref{alg:m1} and~\ref{alg:m2}.
Then, we have the following upper bound for the divergence:

		
        \begin{align}
            \Dz\left(\left(\widetilde{w}_{t+1}, \widetilde{x}_{t+2}\right)\bigl\| \left(\widetilde{w}'_{t+1}, \widetilde{x}'_{t+2}\right)\right) &\le \frac{1}{2\sigma^2}\left(2\left\| x_{t+1} - x'_{t+1} \right\|^2 + 3\left\|\frac{1}{\beta}\lambda_{t+1} - Ax_{t+1} - \frac{1}{\beta}\lambda'_{t+1} +  Ax'_{t+1} \right\|^2\right) \notag\\
        \end{align}
\end{lemma}

\begin{proof}
    By the adaptive composition of private mechanisms
in Fact~\ref{fact:zCDP}(c),
the sum of the upper bounds in Lemmas~\ref{lemma:m1_privacy}
and~\ref{lemma:m2_privacy} produces the result.
\end{proof}

\begin{lemma}[One-Step Privacy of Operator $K$]
    \label{lemma:K_privacy}
    Given inputs $(\widetilde{x}_t, \lambda_t)$ and $(\widetilde{x}'_t, \lambda'_t)$
		and Markov operator $K \in \mcal{K}$ (corresponding to two
		ADMM iterations using convex $\frac{1}{\eta}$-smooth functions
		and $\mcal{N}(0, \sigma^2 \I_n)$ noise),
		we have:
		\begin{align}
        \Dz\left(K(\widetilde{x}_t, \lambda_t) \bigl\| K(\widetilde{x}'_t, \lambda'_t) \right) \le C_{\mcal{K}} \left\| (\widetilde{x}_{t} - \widetilde{x}'_{t},
				\lambda_{t} - \lambda'_{t})	\right\|^2_*, \notag
    \end{align}
		 where $C_{\mcal{K}} = \frac{1}{2 \sigma^2} \max\left(2, \frac{3}{\eta\beta}\right)$.
\end{lemma}

\begin{proof}
    By Lemma~\ref{lemma:equiv},
it suffices to consider the adaptive composition
$(\mcal{M}_2 \circ \mcal{M}_1)$ on the two inputs.
The first inequality below is due to Lemma~\ref{lemma:composition_privacy}:

    \begin{align}
        \Dz\left(\left(\widetilde{w}_{t+1}, \widetilde{x}_{t+2}\right)\bigl\| \left(\widetilde{w}'_{t+1}, \widetilde{x}'_{t+2}\right)\right) &\le \frac{1}{2\sigma^2}\left(2\left\| x_{t+1} - x'_{t+1} \right\|^2 + 3\left\|\frac{1}{\beta}\lambda_{t+1} - Ax_{t+1} - \frac{1}{\beta}\lambda'_{t+1} +  Ax'_{t+1} \right\|^2\right) \notag\\
        &\le C_{\mcal{K}} \cdot \left(\left\| x_{t+1} - x'_{t+1} \right\|^2 + \eta\beta\left\|\frac{1}{\beta}\lambda_{t+1} - Ax_{t+1} - \frac{1}{\beta}\lambda'_{t+1} +  Ax'_{t+1} \right\|^2\right) \notag \\
				& = C_{\mcal{K}}\| \cdot (x_{t+1} - x'_{t+1}, \lambda_{t+1} - \lambda'_{t+1} ) \|^2_* \notag,
    \end{align}
    where the second inequality
		is due to the choice of $C_{\mcal{K}} := \frac{1}{2\sigma^2}\max\left(2, \frac{3}{\eta\beta}\right)$.

 Finally,  observing that $(x_{t+1}, \lambda_{t+1})$
is the output of one ADMM iteration Algorithm~\ref{alg:one_iteration}
with input $(\widetilde{x}_t, \lambda_t)$,
Lemma~\ref{lemma:ADMM_exp}
states that
$\| (x_{t+1} - x'_{t+1}, \lambda_{t+1} - \lambda'_{t+1} ) \|^2_* \leq
\| (\widetilde{x}_{t} - \widetilde{x}'_{t}, \lambda_{t} - \lambda'_{t} ) \|^2_*$, which
 finishes the proof.
\end{proof}

\subsection{Combining Everything Together to Achieve Privacy Amplification by Iteration for ADMM}
\label{sec:combining}

In order to use the coupling framework described in Theorem~\ref{th:privacy_amp},
it suffices to check that
each Markov operator $K \in \mcal{K}$ (which corresponds
to two noisy ADMM iterations) satisfies both the
 (A)~non-expansion property and (B)~one-step privacy.
While condition~(B) is achieved exactly by Lemma~\ref{lemma:K_privacy},
we will convert the non-expansion property for one ADMM iteration
in Lemma~\ref{lemma:ADMM_exp} to condition~(A) expressed by Lemma~\ref{lemma:K_exp}.

\begin{definition}[$\infty$-Wasserstein Distance for Customized Norm]
\label{defn:wasserstein2}
Given distributions $\mu$ and $\nu$ on
the $(x, \lambda)$-space $\R^n \times \R^m$ and $\Delta \geq 0$,
a coupling $\pi$ from $\mu$ to $\nu$ is a witness
that the (infinity) customized Wasserstein distance $\W_*(\mu, \nu) \leq \Delta$
if for all $(z,z') \in \supp(\pi)$, $\|z - z'\|_* \leq \Delta$.

The distance $\W_*(\mu, \nu)$
is the infimum of the
collection of $\Delta$ for which such a witness $\pi$ exists.
\end{definition}

\begin{lemma}[Non-Expansion Property of Operator $K$]
    \label{lemma:K_exp}
    Given inputs $(\widetilde{x}_t, \lambda_t)$ and $(\widetilde{x}'_t, \lambda'_t)$
		and Markov operator $K \in \mcal{K}$ (corresponding to two noisy
		ADMM iterations using convex $\frac{1}{\eta}$-smooth functions
		and $\mcal{N}(0, \sigma^2 \I_n)$ noise),
		the natural coupling is a witness that
		
		$$\W_* \left(K(\widetilde{x}_t, \lambda_t) \| K(\widetilde{x}'_t, \lambda'_t)   \right) \leq \left \| (\widetilde{x}_{t} - \widetilde{x}'_{t}, \lambda_{t} - \lambda'_{t} )      \right \|^2_*.$$		
\end{lemma}

\begin{proof}
By the natural coupling, we are using the same randomness $(N_{t+1}, N_{t+2})$
for the two input scenarios.   Observe that under this coupling, we can treat $K$ as two (deterministic) ADMM iterations, each of which uses a convex $\frac{1}{\eta}$-smooth function.
Therefore, two applications of Lemma~\ref{lemma:ADMM_exp} gives the required result.
\end{proof}

\noindent \textbf{Proof of Theorem~\ref{th:ADMM_privacy}.}
The same proof for Theorem~\ref{th:privacy_amp} can be applied,
because we have already established
the corresponding condition~(A) in Lemma~\ref{lemma:K_exp}
and condition~(B) in Lemma~\ref{lemma:K_privacy}.
However, observe that
$\|(x_0 - x'_0, 0)\|^2_* = \|x_0 - x'_0 \|^2 + \eta \beta \|A (x_0 - x_0') \|^2
\leq (1 + \eta \beta \|A\|^2) \|x_0 - x_0'\|^2$,
which explains the factor $(1 + \eta \beta \|A\|^2)$ in the constant.

The desired bound follows because $2T$ noisy ADMM iterations
correspond to the composition of $T$ Markov operators in the collection~$\mcal{K}$.
\qed 

\section{Extension of Privacy Amplification Results to Strongly Convex Objective Functions}
\label{sec:strongly_convex}

In this section,
we extend the privacy amplification results for ADMM
to the scenario when the objective functions are strongly convex.
On a high level, in the framework described in Section~\ref{sec:coupling},
strong convexity enables us to derive a strictly contractive guarantee
in condition~(A) of Theorem~\ref{th:privacy_amp} (for 
a new customized norm),
thereby leading to an exponential improvement in
the privacy guarantee.


We first state the assumptions on the objective functions.
Recall that a function $f$ is $\mu$-strongly convex with parameter $\mu > 0$
if $x \mapsto f(x) - \frac{\mu}{2} \cdot \|x \|^2$ is convex.

\begin{assumption}[Strongly Convex Objective Functions]
    \label{Assumption_strongly-convex}
    
		We assume that the algorithm knows the following
		parameters $ {\mu} $ and $\mu_g$.\footnote{To simplify
		the notation, we only use subscript for $\mu_g$, which
		appears in only a few places in the proof.
		}
		
    \begin{compactitem}
		
		\item There exists $ {\mu}  > 0$ such that
		in each ADMM iteration~$t$, the function $f_t$ is $ {\mu} $-strongly convex.
		
		\item There exists $\mu_g > 0$ such that
		the function $g$ is $\mu_g$-strongly convex.
		       
    \end{compactitem}
\end{assumption}

\noindent \textbf{Other Assumptions.}
As before, we still assume that each $f_t$ is $\nu$-smooth (recall that
we picked $\eta = \frac{1}{\nu}$ for the general convex case);
for the strongly convex case, we will pick $\eta$ more carefully in Lemma~\ref{lemma:contractive_mapping}. 
Moreover, after the transformation in Remark~\ref{remark:transform},
we can also assume that $A$ has full row rank.

The main result in this section goes as the following:

\begin{theorem}[Privacy Amplification by Iteration for ADMM under Strongly Convex Assumption]
    \label{th:ADMM_privacy_SC}
    Suppose the objective functions are strongly convex as
		stated in Assumption~\ref{Assumption_strongly-convex}. Given two input scenarios $(x_0, \lambda_0)$
    and $(x'_0, \lambda_0)$, a total of $2T$ \emph{noisy} ADMM  iterations
    in Algorithm~\ref{alg:one_iteration} are applied
    to each input scenario, where for each iteration $t \in [2T]$,
    the same function $f_t$ is used in both scenarios
    and fresh randomness $N_t$ drawn from Gaussian distribution
    $\mcal{N}(0, \sigma^2 \I_n)$ is used to produce
    masked $\widetilde{x}_t \gets x_t + N_t$ (that is passed to the next iteration
    together with $\lambda_t$).
		
    Then, for some constants $0 < \mathfrak{L}<1$
		and $0 < \eta < \frac{2}{\nu + \mu}$
		(independent of the inputs),
		the corresponding output distributions from the two scenarios satisfy:
    
    $$\Dz((\widetilde{x}_{2T}, \lambda_{2T}) \| (\widetilde{x}'_{2T}, \lambda'_{2T})) \leq \frac{C}{2 \sigma^2} \cdot  \frac{\mathfrak{L}^{2T-1}}{T} \cdot \|x_0 - x'_0 \|^2,$$
    
    where $C := \max\left(\frac{2}{\left(1-\frac{2 \eta \nu  {\mu} }{\nu+ {\mu} }\right)+\frac{1}{\eta}\left(\frac{2}{\nu+ {\mu} }-\eta\right)}, \frac{3}{\eta\beta}\right)\left[\left(1-\frac{2 \eta \nu  {\mu} }{\nu+ {\mu} }\right)+\frac{1}{\eta}\left(\frac{2}{\nu+ {\mu} }-\eta\right)+\eta \beta\|A\|^2\right]$.
		
\end{theorem}
    
\begin{corollary}[Privacy Amplification for the First User under Strongly Convex Assumption]
\label{cor:first_user_SC}
Consider two scenarios for running
$2T+1$ \emph{noisy} ADMM iterations (using
Gaussian noise $\mcal{N}(0, \sigma^2 \I)$ in each iteration to
mask only the $x$ variable),
where the only difference
in the two scenarios
is that the corresponding functions used in the first iteration
can be different neighboring functions $f_1 \sim_\Delta f'_1$ (as
in Definition~\ref{defn:neighbor_functions}).
In other words, the initial $(x_0, \lambda_0)$
and the functions $f_t$ in subsequent iterations $t \geq 2$ are identical
in the two scenarios.

Then, 
for the same choice of the constants
$\mathfrak{L}, \eta, C > 0$
as in Theorem~\ref{th:ADMM_privacy_SC},
the solutions from the two scenarios
satisfy the following.

\begin{compactitem}

\item \emph{Local Privacy.} After the first iteration, we have

$$\Dz(\widetilde{x}_1 \| \widetilde{x}'_1) \leq \frac{\eta^2 \Delta^2}{2 \sigma^2}.$$

\item \emph{Privacy Amplification.}  After the final iteration $2T+1$, 
we have:

$$\Dz((\widetilde{x}_{2T+1}, \lambda_{2T+1}) \| (\widetilde{x}'_{2T+1}, \lambda'_{2T+1})) \leq \frac{C \mathfrak{L}^{2T-1}}{T} \cdot\frac{\eta^2 \Delta^2}{2 \sigma^2}.$$ 
%
\end{compactitem}
\end{corollary}

\begin{proof}
For local privacy,
observe that by Lemma~\ref{lemma:step1}, we have

$$\|x_1 - x'_1 \| =
\eta \| (\I + \eta \beta A^\top A)^{-1} (\nabla f_1(x_0) - \nabla f'_1(x_0))\|
\leq \eta \|\nabla f_1(x_0) - \nabla f'_1(x_0)\|
\leq \eta \Delta,$$

\noindent where the first inequality comes from the fact that all eignevalues of
$(\I + \eta \beta A^\top A)^{-1}$ are positive and at most~1.
The last inequality 
holds thanks to the assumption that $f_1 \sim_{\Delta} f'_1$.
The result follows from Theorem~\ref{th:ADMM_privacy_SC}.
\ignore{
Since we add Gaussian noise in $\mcal{N}(0, \sigma^2 \I_n)$, Fact~\ref{fact:gd} implies that 
$\Dz(\widetilde{x}_1 \| \widetilde{x}'_1) \leq \frac{\eta^2 \Delta^2}{2 \sigma^2}$.

For privacy amplification, observe that
the natural coupling in the first iteration implies that
$\W_*((\widetilde{x}_1, \lambda_1), (\widetilde{x}'_1, \lambda'_1))
\leq \eta \Delta$.  Theorem~\ref{th:ADMM_privacy_SC}
gives the required result.
}
\end{proof}

\subsection{Achieving Strict Contraction using Strong Convexity and Customized Norm}
\label{sec:contractive}

Similar to Section~\ref{sec:non-expansion}, we will derive the contractive property based on a customized norm.
The definition of the norm is similar to Section~\ref{sec:non-expansion}, except
that there is an extra factor for the first term.

\begin{definition}[Customized Norm for Strongly Convex Case] \label{def:norm_SC}
    Using the ADMM parameters $\beta$
    from Section~\ref{sec:admm_prelim}, $\eta$, and the parameters
		in Assumption~\ref{Assumption_strongly-convex}, we define a customized norm.
    For $(x, \lambda) \in \R^n \times \R^m$,
    
    $$\|(x,\lambda)\|_*^2:=\left[\left(1-\frac{2 \eta \nu  {\mu} }{\nu+ {\mu} }\right)+\frac{1}{\eta}\left(\frac{2}{\nu+ {\mu} }-\eta\right)\right]\left\|x\right\|^2+\frac{\eta}{\beta}\left\| \lambda-\beta Ax  \right\|^2,$$
		
		where $0 < \eta < \frac{2}{\nu + \mu}$ will be decided later.
\end{definition}

For strongly convex functions, we adopt the following fact
that appears as an inequality in the proof~\cite[Lemma 18]{BalleBGG19}.

\begin{fact}\cite{BalleBGG19}\label{fact:SC_bound}
		Suppose the function $f$
		is $\nu$-smooth and $ {\mu} $-strongly convex. Then, for $0 < \eta \leq 2/(\nu + {\mu} )$
		and any inputs $x$ and $x'$ in the domain of $f$, 
    \begin{align}
            \|x -\eta \nabla f(x)- (x'-\eta \nabla f(x'))\|^2
            \leq\left(1-\frac{2 \eta \nu  {\mu} }{\nu+ {\mu} }\right)\|x-x'\|^2+\eta\left(\eta-\frac{2}{\nu+ {\mu} }\right) \| \nabla f(x)-\nabla f(x')\|^2.
    \end{align}
\end{fact}

\begin{lemma}[Strict Contraction]\label{lemma:contractive_mapping}
There exist $0 < \eta < \frac{2}{\nu + \mu}$ and $0 < \mathfrak{L} < 1$ such that
when one ADMM iteration in Algorithm~\ref{alg:one_iteration}
is applied to any two different inputs $(x_t, \lambda_t)$
and $(x'_t, \lambda'_t)$ with the same
function $f$ (satisfying Assumption~\ref{Assumption_strongly-convex}),
the corresponding two outputs
$(x_{t+1}, \lambda_{t+1})$
and $(x'_{t+1}, \lambda'_{t+1})$
satisfy:

%
%

    $$ \|(x_{t+1} - x'_{t+1}, \lambda_{t+1} - \lambda'_{t+1})\|^2_*
    \leq
    \mathfrak{L}\|(x_{t} - x'_{t}, \lambda_{t} - \lambda'_{t})\|^2_*.$$
\end{lemma}

\begin{proof}
		Note that we consider two collections of variables,
		where one collection has variables with superscripts.
		Hence, we use the notation $\mathfrak{d} z := z - z'$ to indicate the difference between the
    original variable and the one with the superscript. 

\noindent \textbf{Inequalities from the Optimality of $y_t$.}
As in Lemma~\ref{lemma:step2}, the optimality of $y_t$ and $y'_t$ in \eqref{eq:inequalities:update_y} implies the following.

\begin{align}
    0 \in \partial g(y_t) + \beta B^\top\left( Ax_{t} + By_t - c - \frac{1}{\beta}\lambda_t \right)\notag\\
    0 \in \partial g(y_t') + \beta B^\top\left( Ax_{t}' + By_t' - c - \frac{1}{\beta}\lambda_t' \right) \notag
\end{align}

Since $g$ is $\mu_g$-strongly convex,
we have the monotonicity of subgradient,
i.e., $s \in \partial g(y_t)$
and $s' \in \partial g(y_t')$ imply that
$\langle s - s', y_t - y'_t \rangle \geq \mu_g \|y_t-y'_t\|^2$.
After rearranging and using \eqref{eq:inequalities:update_lambda}, we have

\begin{align} \label{eq:inequalities:subgradient_mononicity_SC}
\left\langle  \mathfrak{d}\lambda_{t+1}, B\mathfrak{d}y_t \right\rangle =
\left\langle \mathfrak{d}\lambda_t - \beta(A\mathfrak{d}x_t + B\mathfrak{d}y_t), B\mathfrak{d}y_t \right\rangle \ge \mu_g \|\mathfrak{d}y_t\|^2.
\end{align}

\noindent \textbf{Inequalities from the Optimality of $x_{t+1}$.}

From Lemma~\ref{lemma:step1},
the optimality of $x_{t+1}$ from \eqref{eq:inequalities:update_x} implies
the following.

\begin{align}
    \left\langle \nabla f(x_t) + \frac{1}{\eta} (x_{t+1} - x_t) + \beta A^\top \left( Ax_{t+1} + By_t - c - \frac{1}{\beta}\lambda_{t+1} \right), x - x_{t+1} \right\rangle = 0, & \forall x \in \R^n  \label{eq:inequalities:first-order-x}\\
    \left\langle \nabla f(x_t') + \frac{1}{\eta} (x_{t+1}' - x_t') + \beta A^\top \left( Ax_{t+1}' + By_t' - c - \frac{1}{\beta}\lambda_{t+1}' \right), x - x_{t+1}' \right\rangle = 0, & \forall x \in \R^n  \label{eq:inequalities:first-order-x'}
\end{align}

Substituting $x = x_{t+1}'$ in \eqref{eq:inequalities:first-order-x} and $x = x_{t+1}$ in \eqref{eq:inequalities:first-order-x'}, summing \eqref{eq:inequalities:first-order-x} and \eqref{eq:inequalities:first-order-x'} leads to
a lower bound for the following quantity:

\begin{align}\label{eq:bound_by_mu_g}
    &\left\langle\frac{1}{\eta} \mathfrak{d} x_t  -\mathfrak{d} \nabla f\left(x_t\right), \mathfrak{d} x_{t+1}\right\rangle
     =\frac{1}{\eta}\left\|\mathfrak{d} x_{t+1}\right\|^2+\left\langle\beta\left(A \mathfrak{d} x_{t+1}+B \mathfrak{d} y_t\right)-\mathfrak{d} \lambda_{t+1}, A \mathfrak{d} x_{t+1}\right\rangle \notag\\
    & =\frac{1}{\eta}\left\|\mathfrak{d} x_{t+1}\right\|^2+\left\langle\beta A \mathfrak{d} x_{t+1}-\mathfrak{d} \lambda_{t+1}, A \mathfrak{d} x_{t+1}+B \mathfrak{d} y_t\right\rangle+\left\langle\mathfrak{d} \lambda_{t+1}, B \mathfrak{d} y_t\right\rangle \notag\\
    & \geqslant \frac{1}{\eta}\left\|\mathfrak{d} x_{t+1}\right\|^2+\frac{1}{\beta}\left\langle\mathfrak{d}\lambda_{t+1}-A\mathfrak{d}x_{t+1}, \mathfrak{d}\lambda_{t+1}-A\mathfrak{d}x_{t+1}-(\mathfrak{d}\lambda_{t}-A\mathfrak{d}x_{t})\right\rangle+\mu_g\left\|\mathfrak{d} y_t\right\|^2 \tag{using \eqref{eq:inequalities:subgradient_mononicity_SC}}\\
    & \geqslant \frac{1}{\eta}\|\mathfrak{d} x_{t+1}\|^2+\frac{1}{2\beta}\left\|\mathfrak{d}\lambda_{t+1}-A\mathfrak{d}x_{t+1}\right\|^2-\frac{1}{2\beta}\left\|\mathfrak{d}\lambda_{t}-A\mathfrak{d}x_{t}\right\|^2+\mu_g\left\|\mathfrak{d} y_t\right\|^2,   
\end{align}

where the last inequality comes from $\|a\|^2+\|b\|^2\geq 2\langle a,b\rangle$.

We next derive an upper bound for $\left\langle\frac{1}{\eta} \mathfrak{d} x_t  -\mathfrak{d} \nabla f\left(x_t\right), \mathfrak{d} x_{t+1}\right\rangle$.
Using the Cauchy-Schwarz inequalty, this is as most:

\begin{align} \label{ineq:bounded_inner_product}
     &\frac{1}{\eta}\left\|\mathfrak{d} x_t-\eta \mathfrak{d} \nabla f\left(x_t\right)\right\| \cdot  \left\|\mathfrak{d} x_{t+1}\right\| 
    \leqslant \frac{1}{2 \eta}\left(\left\|\mathfrak{d} x_t-\eta \mathfrak{d} \nabla f\left(x_t\right)\right\|^2+\left\|\mathfrak{d} x_{t+1}\right\|^2\right) \notag
    \\
    &\leqslant\frac{1}{2\eta}\left(1-\frac{2\eta \nu  {\mu} }{\nu+ {\mu} }\right)\|\mathfrak{d} x_t\|^2+\frac{1}{2}\left(\eta-\frac{2}{\nu+ {\mu} }\right)\|\mathfrak{d} \nabla f(x_t)\|^2+\frac{1}{2\eta}\left\|\mathfrak{d} x_{t+1}\right\|^2.
\end{align}

where the last inequality comes from Fact~\ref{fact:SC_bound}.

Combining the upper and the lower bounds, we get

\begin{align} 
    &\frac{1}{\eta}\left\|\mathfrak{d} x_{t+1}\right\|^2+\frac{1}{\beta}\left\|\mathfrak{d}\lambda_{t+1}-A\mathfrak{d}x_{t+1}\right\|^2+\left(\frac{2}{\nu+ {\mu} }-\eta\right)\left\|\mathfrak{d} \nabla f\left(x_t\right)\right\|^2+2\mu_g \|\mathfrak{d} y_t\|^2 \notag
    \\ \leq
    &\frac{1}{\eta}\left(1-\frac{2 \eta \nu  {\mu} }{\nu+ {\mu} }\right)\left\|\mathfrak{d} x_{t}\right\|^2+\frac{1}{\beta}\left\|\mathfrak{d}\lambda_{t}-A\mathfrak{d}x_{t}\right\|^2. \label{SC_contractive}
\end{align}

We next give a lower bound for $\|\mathfrak{d} \nabla f(x_t)\|^2$. 

Below we use the inequality:
$\|p + q + r + s \|^2 \geq \frac{1}{4} \|p\|^2 - \|q\|^2 - \|r\|^2 - \|s\|^2$.
Hence, we have:

\begin{align}
    &\left\|\mathfrak{d} \nabla f\left(x_t\right)\right\|^2=\left\|A^{\top}\left[\mathfrak{d} \lambda_{t+1}-\beta\left(A \mathfrak{d} x_{t+1}+B \mathfrak{d} y_t\right)\right]+\frac{1}{\eta}\left(\mathfrak{d} x_t-\mathfrak{d} x_{t+1}\right)\right\|^2 \notag\\
    & =\left\| A^{\top}(\mathfrak{d} \lambda_{t+1}-\beta A \mathfrak{d}x_{t+1})+\frac{1}{\eta}\left(\mathfrak{d} x_t-\mathfrak{d} x_{t+1}\right)-\beta A^{\top} B\mathfrak{d} y_{t}\right\|^2 \notag \\
    &\geq \frac{1}{4}\left\| A^{\top}(\mathfrak{d} \lambda_{t+1}-\beta A \mathfrak{d}x_{t+1})\right\|^2-\frac{1}{\eta^2}\left\|\mathfrak{d} x_t\right\|^2-\frac{1}{\eta^2}\left\|\mathfrak{d} x_{t+1}\right\|^2-\left\|\beta A^{\top} B\mathfrak{d} y_{t}\right\|^2 \\
    &\geq \frac{1}{4}\left\|\mathfrak{d} \lambda_{t+1}-\beta A \mathfrak{d}x_{t+1}\right\|^2-\frac{1}{\eta^2}\left\|\mathfrak{d} x_t\right\|^2-\frac{1}{\eta^2}\left\|\mathfrak{d} x_{t+1}\right\|^2-\beta^2\left\|A^{\top}B\right\|^2\left\|\mathfrak{d} y_t\right\|^2, \notag
\end{align}
where the last inequality
follows from the spectral norm of the operator $A^{\top} B$
and the fact that
the minimum eigenvalue of $A A^{\top} = I + D D^{\top}$ is at least 1.


Apply this inequality to (\ref{SC_contractive}), we have
\begin{align}
    & \left[\frac{1}{ \eta}-\frac{1}{\eta^2}\left(\frac{2}{\nu+ {\mu} }-\eta\right)\right]\left\|\mathfrak{d} x_{t+1}\right\|^2 + 
		\left[\frac{1}{\beta}+\frac{1}{4}\left(\frac{2}{\nu+ {\mu} }-\eta\right)\right]\left\|\mathfrak{d} \lambda_{t+1}-\beta A \mathfrak{d}x_{t+1}\right\|^2 \notag
    \\
    & +2\mu_g\left\|\mathfrak{d} y_t\right\|^2-\beta^2\left\|A^{\top}B\right\|^2 \left(\frac{2}{\nu+ {\mu} }-\eta\right)\left\|\mathfrak{d} y_t\right\|^2 \notag\\
    & \leqslant\left[\frac{1}{ \eta}\left(1-\frac{2 \eta \nu  {\mu} }{\nu+ {\mu} }\right)+\frac{1}{ \eta^2}\left(\frac{2}{\nu+ {\mu} }-\eta\right)\right]\left\|\mathfrak{d} x_t\right\|^2+\frac{1}{\beta}\left\|\mathfrak{d} \lambda_{t}-\beta A \mathfrak{d}x_{t}\right\|^2. \label{ineq:operator_bound}
\end{align}

We first wish to eliminate the term
involving $\left\|\mathfrak{d} y_t\right\|^2$.
It suffices to choose $\eta$ such that the following holds:
$2\mu_g-\beta^2\left\|A^{\top} B\right\|^2\left(\frac{2}{\nu+ {\mu} }-\eta\right) \geq 0$,

which is equivalent to $\eta \geq \frac{2}{\nu+ {\mu} }-\frac{2 \mu_g}{\beta^2\|A^{\top}B\|^2}$.

We next multiply (\ref{ineq:operator_bound}) by $\eta$
to get an inequality of the form:

 $P \cdot \|\mathfrak{d}x_{t+1}\|^2+ Q \cdot \|\mathfrak{d}\lambda_{t+1}-\beta A\mathfrak{d}x_{t+1}\|^2\leq R \cdot \|\mathfrak{d}x_{t}\|^2+ S \cdot \|\mathfrak{d}\lambda_{t}-\beta A\mathfrak{d}x_{t}\|^2$,

where $R := \left(1-\frac{2 \eta \nu  {\mu} }{\nu+ {\mu} }\right)+\frac{1}{\eta}\left(\frac{2}{\nu+ {\mu} }-\eta\right)$ and $S := \frac{\eta}{\beta}$ have the right values,
because according to Definition~\ref{def:norm_SC}, we have
$\|(\mathfrak{d} x_t, \mathfrak{d} \lambda_t)\|^2_* 
=
R \cdot \|\mathfrak{d}x_{t}\|^2+ S \cdot \|\mathfrak{d}\lambda_{t}-\beta A\mathfrak{d}x_{t}\|^2$

Moreover, we have $Q := \frac{\eta}{\beta}+\frac{\eta}{4}\left(\frac{2}{\nu+ {\mu} }-\eta\right)
> S$, whenever $\eta < \frac{2}{\nu + \mu}$.

Finally, we have $P := 1-\frac{1}{\eta}\left(\frac{2}{\nu+ {\mu} }-\eta\right)$,
and we wish to find $\eta$ such that $P > R$.
If we can achieve this,
we can define $\mathfrak{L} := \max\{ \frac{R}{P}, \frac{S}{Q}\} < 1$
such that:

$\|(\mathfrak{d} x_{t+1}, \mathfrak{d} \lambda_{t+1})\|^2_* 
\leq \mathfrak{L} \cdot( 
P \cdot \|\mathfrak{d}x_{t+1}\|^2+ Q \cdot \|\mathfrak{d}\lambda_{t+1}-\beta A\mathfrak{d}x_{t+1}\|^2) \leq
\mathfrak{L} \cdot \|(\mathfrak{d} x_t, \mathfrak{d} \lambda_t)\|^2_* $.

To ensure $P > R$, it suffices to satisfy: $\frac{2\nu  {\mu} }{\nu+ {\mu} }-\frac{2}{\eta^2}\left(\frac{2}{\nu+ {\mu} }-\eta\right)>0$,
which is equivalent to

$\eta >
\frac{4}{\nu+ {\mu} +\sqrt{\left(\nu+ {\mu} \right)^2+8 \nu  {\mu} }}
$.

\ignore{
In order to derive a contractive map from this inequality, it sufficies to let $a>c$ and $b>d$. Hence we require
\begin{align}
    & \begin{array}{l}
        \frac{2\nu  {\mu} }{\nu+ {\mu} }-\frac{2}{\eta^2}\left(\frac{2}{\nu+ {\mu} }-\eta\right)>0 \\
        2\mu_g-\beta^2\left\|A^{\top} B\right\|^2\left(\frac{2}{\nu+ {\mu} }-\eta\right)>0 \notag
        \end{array}
\end{align}
to make \ref{ineq:operator_bound} imply a contractive mapping, 
}

To summarize, we check that we can pick the value of $\eta$ such that:

\begin{align} \label{ineq:eta_range_for_SC}
    \eta \in \left( \max \left\{\frac{4}{\nu+ {\mu} +\sqrt{\left(\nu+ {\mu} \right)^2+8 \nu  {\mu} }}, \frac{2}{\nu+ {\mu} }-\frac{2 \mu_g}{\beta^2\|A^{\top}B\|^2}\right\}, \frac{2}{\nu+  {\mu} }\right).
\end{align}
Since this interval is non-empty, the proof is completed.
\end{proof}

\subsection{Achieving One-Step Privacy under Strongly Convex Assumption}
\label{sec:one_step_privacy_SC}

Based on the same discussion as Lemma~\ref{lemma:K_privacy}, we get a similar result as in previous sections. 
In the following analysis, we assume that $\eta$ satisfies the assumption in Lemma~\ref{lemma:contractive_mapping}.
\begin{lemma}[One-Step Privacy of Operator $K$ under Strongly Convex Assumption]
    \label{lemma:K_privacy_SC}
    Suppose the objective functions are strongly convex as
		stated in Assumption~\ref{Assumption_strongly-convex}. Given inputs $(\widetilde{x}_t, \lambda_t)$, $(\widetilde{x}'_t, \lambda'_t)$
		and Markov operator $K \in \mcal{K}$ (corresponding to two
		ADMM iterations using convex $\nu$-smooth functions
		and $\mcal{N}(0, \sigma^2 \I_n)$ noise),
		we have:
		\begin{align}
        \Dz\left(K(\widetilde{x}_t, \lambda_t) \bigl\| K(\widetilde{x}'_t, \lambda'_t) \right) \le C_{\mcal{K}}\left\| (\widetilde{x}_{t} - \widetilde{x}'_{t}, 
				\lambda_{t} - \lambda'_{t})	\right\|^2_*, \notag
    \end{align}
		 where $C_{\mcal{K}} = \frac{\mathfrak{L}}{2 \sigma^2} \max\left(\frac{2}{\left(1-\frac{2 \eta \nu  {\mu} }{\nu+ {\mu} }\right)+\frac{1}{\eta}\left(\frac{2}{\nu+ {\mu} }-\eta\right)}, \frac{3}{\eta\beta}\right)$, and $\eta > 0$ and $0<\mathfrak{L}<1$ constants
		chosen as in Lemma~\ref{lemma:contractive_mapping}.
\end{lemma} 

\begin{proof}
    Similar to Section~\ref{sec:one_step_privacy}, denote $ w_t=\lambda_t-Ax_t$. 
		Observe that the analysis in
		Section~\ref{sec:one_step_privacy} is still valid for the strongly convex case.
		Hence, we can use the same divergence bound given in Lemma~\ref{lemma:composition_privacy}: 
    \begin{align}
        &\Dz\left(\left(\widetilde{w}_{t+1}, \widetilde{x}_{t+2}\right)\bigl\| \left(\widetilde{w}'_{t+1}, \widetilde{x}'_{t+2}\right)\right) \\
        &\le \frac{1}{2\sigma^2}\left(2\left\| x_{t+1} - x'_{t+1} \right\|^2 + 3\left\|\frac{1}{\beta}\lambda_{t+1} - Ax_{t+1} - \frac{1}{\beta}\lambda'_{t+1} +  Ax'_{t+1} \right\|^2\right) \notag\\
        &\le \frac{C_{\mcal{K}}}{\mathfrak{L}} \left([(1-\frac{2 \eta \nu  {\mu} }{\nu+ {\mu} })+\frac{1}{\eta}(\frac{2}{\nu+ {\mu} }-\eta)]\left\| x_{t+1} - x'_{t+1} \right\|^2 + \eta\beta\left\|\frac{1}{\beta}\lambda_{t+1} - Ax_{t+1} - \frac{1}{\beta}\lambda'_{t+1} +  Ax'_{t+1} \right\|^2\right) \notag \\
				& = \frac{C_{\mcal{K}}}{\mathfrak{L}} \cdot \| (x_{t+1} - x'_{t+1}, \lambda_{t+1} - \lambda'_{t+1} ) \|^2_* \tag{Definition~\ref{def:norm_SC}}\\
                &\le C_{\mcal{K}} \cdot \|\widetilde{x}_{t}-\widetilde{x}'_{t},\lambda_{t}-\lambda'_{t}\|_*^2 \tag{Lemma~\ref{lemma:contractive_mapping}},
    \end{align}
		
where the second inequality comes
from the choice of $C_{\mcal{K}} := \frac{\mathfrak{L}}{2 \sigma^2} \max\left(\frac{2}{\left(1-\frac{2 \eta \nu  {\mu} }{\nu+ {\mu} }\right)+\frac{1}{\eta}\left(\frac{2}{\nu+ {\mu} }-\eta\right)}, \frac{3}{\eta\beta}\right)$. 
    
\end{proof}

\subsection{Combining Everything Together to Achieve Privacy Amplification by Iteration for ADMM under Strongly Convex Assumption}
\label{sec:combining_SC}

To adopt the framework of Theorem~\ref{th:privacy_amp}, we need to show that the Markov operators $K \in \mathcal{K}$ satisfy the contractive property and one-step privacy.
We adopt the same definition of the $\infty$-Wasserstein distance as definition~\ref{defn:wasserstein2}. 
\begin{lemma}[Contractive Property of Operator $K$]
\label{lemma:K_exp_SC}
    Given inputs $(\widetilde{x}_t, \lambda_t)$ and $(\widetilde{x}'_t, \lambda'_t)$
		and Markov operator $K \in \mcal{K}$ (corresponding to two noisy
		ADMM iterations using convex $\nu$-smooth functions
		and $\mcal{N}(0, \sigma^2 \I_n)$ noise),
		the natural coupling is a witness that
		
		$$W_* \left(K(\widetilde{x}_t, \lambda_t) \| K(\widetilde{x}'_t, \lambda'_t)   \right)^2 \leq \mathfrak{L}^2 \left \| (\widetilde{x}_{t} - \widetilde{x}'_{t}, \lambda_{t} - \lambda'_{t} )      \right \|^2_*,$$
        where $0<\mathfrak{L}<1$ is chosen
				as in Lemma~\ref{lemma:contractive_mapping}.		
\end{lemma}

\begin{proof}
    By the natural coupling, we are using the same randomness $(N_{t+1}, N_{t+2})$ for the two input
    scenarios. Observe that under this coupling, we can treat $K$ as two (deterministic) ADMM iterations. Moreover, in each iteration, $f$ is $ {\mu} $-strongly convex and $\nu$-smooth. Therefore, applying
    Lemma~\ref{lemma:contractive_mapping} twice gives the required result.
\end{proof}

Now we are able to prove the main result in Theorem~\ref{th:ADMM_privacy_SC}.
\begin{proof}[Proof of Theorem~\ref{th:ADMM_privacy_SC}]
The same proof as Theorem~\ref{th:privacy_amp} can be applied, according to
the established corresponding condition~(A) in Lemma~\ref{lemma:K_exp_SC}
and condition~(B) in Lemma~\ref{lemma:K_privacy_SC}. Note that $L$ in Theorem~\ref{th:privacy_amp}
corresponds to $\mathfrak{L}^2$.

However, observe that
\begin{align}
\|(x_0 - x'_0, \lambda_0-\lambda_0)\|^2_* &= \left(1-\frac{2 \eta \nu  {\mu} }{\nu+ {\mu} }\right)+\frac{1}{\eta}\left(\frac{2}{\nu+ {\mu} }-\eta\right)\|x_0 - x'_0 \|^2 + \eta \beta \|A (x_0 - x_0') \|^2 \notag
\\
&\leq \left[ \left(1-\frac{2 \eta \nu  {\mu} }{\nu+ {\mu} }\right)+\frac{1}{\eta}\left(\frac{2}{\nu+ {\mu} }-\eta\right) + \eta \beta \|A\|^2 \right] \|x_0 - x_0'\|^2,
\end{align}
which explains the factor $\left[(1-\frac{2 \eta \nu  {\mu} }{\nu+ {\mu} })+\frac{1}{\eta}(\frac{2}{\nu+ {\mu} }-\eta) + \eta \beta \|A\|^2\right]$ in the constant.

The desired bound follows given that $2T$ noisy ADMM iterations
correspond to the composition of $T$ Markov operators in the collection~$\mcal{K}$.
\end{proof}

\section{Privacy for All Users}
\label{sec:other_users}

As mentioned in the introduction,
given that we have analyzed the privacy guarantee 
from the perspective of the first user,
it is straightforward to apply the techniques in~\cite{DBLP:conf/focs/FeldmanMTT18}
to extend the privacy guarantees to all users.
For completeness, we explain the technical details.

Recall that we model the private data
of user~$t$ as some $f_t \in \mcal{D}$.
Given a distribution $Z_{t-1}$ on some sample space~$\Omega$,
we interpret iteration~$t$ using the private data $f_t$ as
a Markov operator $\Phi(f_t; \cdot) : \mcal{P}(\Omega) \rightarrow \mcal{P}(\Omega)$
that produces the distribution $Z_t := \Phi(f_t; Z_{t-1})$.

The conclusion in Theorem~\ref{th:ADMM_privacy}
and Corollary~\ref{cor:first_user} can be rephrased
in the following assumption, which states
that the privacy infringement for a user essentially
depends on how many iterations have passed after that user.
However, since \renyi divergence is only ``weakly convex'' 
(see~\cite[Lemma 25]{DBLP:conf/focs/FeldmanMTT18}),
we have to explicitly include the order~$\alpha$
of the divergence $\D_\alpha$ in the result.

\begin{assumption}[Privacy Amplification by Iteration]
\label{assume:amp}
Suppose $f_{t+1} \sim f_{t+1}' \in \mcal{D}$ are neighboring private data
for user~$t$, and users in subsequent iterations have data $f_{t+2}, f_{t+3}, \ldots, f_{t+T} \in \mcal{D}$.
Starting from the same distribution $Z_t$,
suppose $Z_{t+1} := \Phi(f_t; Z_t)$ and $Z'_{t+1} := \Phi(f'_t; Z_t)$.
For $j \geq 2$, the same $f_{t+j}$ is used in iteration $t+j$ to produce
the distributions
$Z_{t+j} := \Phi(f_{t+j}; Z_{t+j-1})$
and $Z'_{t+j} := \Phi(f_{t+j}; Z'_{t+j-1})$.

Then, privacy amplification by iteration states that for some $C > 0$,
for any $\alpha \geq 1$,

$\D_\alpha(Z_{t+T} \| Z'_{t+T}) \leq \frac{\alpha C}{T}$.
\end{assumption}

\noindent \textbf{Random Number of Iterations after a User.}
From Assumption~\ref{assume:amp}, the idea is to introduce extra randomness
such that for any user, the number~$L$ of iterations that depends on its private data (including its own iteration)
is not too small with constant probability.  Here are two ideas to introduce this
extra randomness when there are $N$ users.

\begin{itemize}
\item \emph{Random Permutation.} The $N$ users' data are processed
in a uniformly random permutation. In this case, for any fixed user, $L$ has
a uniform distribution on $\{1, 2, \ldots, N\}$
and $\E[\frac{1}{L}] = O(\frac{\log N}{N})$.

\item \emph{Random Stopping.} The $N$ users follow some deterministic 
arbitrary order. A random number $T$ is sampled from $\{\frac{N}{2} + 1,
\ldots, N\}$ and only the first $T$ users' data are used.

For each of the first $\frac{N}{2}$ users, the number
$L$ has a distribution that stochastically dominates 
the uniform distribution on $\{\frac{N}{2} + 1,
\ldots, N\}$.

For $t \geq \frac{N}{2} + 1$,
the $t$-th user's data will not be used with probability $\frac{t - N/2}{N/2}$,
in which case we use the convention $L = + \infty$;
given that its data is used, then $L$
is uniformly distributed in $\{1, 2, N - t\}$.

In any case, we can verify that $\E[\frac{1}{L}] \leq O(\frac{\log N}{N})$.
\end{itemize}

\noindent \textbf{Weak Convexity of \renyi Divergence.}
If $\D_\alpha(\cdot \| \cdot)$ were jointly convex
in both of its arguments,
then in Assumption~\ref{assume:amp},
we could simply replace the term $\frac{1}{T}$ with 
$\E[\frac{1}{L}] = O(\frac{\log N}{N})$.
However, as shown in~\cite[Lemma 25]{DBLP:conf/focs/FeldmanMTT18},
there is only a weak convexity result for $\D_\alpha$,
where we quote a special case as follows.

\begin{lemma}[Weak Convexity for $\D_\alpha$]
\label{lemma:weak_convex_D}
Fix $\alpha \geq 1$.
Suppose $\Lambda$ is an index set
such that for each $\ell \in \Lambda$,
$P_\ell$ and $Q_\ell$ are distributions
satisfying $\D_\alpha( P_\ell \| Q_\ell) \leq \frac{1}{\alpha - 1}$.

For any distribution $L$ on $\Lambda$, we have:
$\D_\alpha( P_L \| Q_L) \leq 2 \E_{\ell \gets L}[\D_\alpha( P_\ell \| Q_\ell)]$.
\end{lemma}

In order to apply Lemma~\ref{lemma:weak_convex_D}
to Assumption~\ref{assume:amp},
the constant $C$ actually needs to be chosen according to~$\alpha$.
In the context of our noisy ADMM, this can be achieved
by choosing a large enough magnitude~$\sigma$ 
for the noise added in each iteration.

\begin{corollary}[Extending Privacy Guarantees to All Users]
Suppose Assumption~\ref{assume:amp} holds
with $C \leq \frac{1}{\alpha(\alpha-1)}$.
For $N$ users, suppose (i) random permutation or (ii) random stopping
is employed for the iterative process.
Then, releasing the final variable is $O(\alpha C \cdot \frac{\log N}{N})$-$\mathsf{RDP}_\alpha$ with respect to each user.
\end{corollary}

\section{Privacy and Utility Trade-off}
\label{sec:utility}

\ignore{
In this section, we suppose $\mathcal{D}$ is a distribution on $f: \mathbb{R}^n \to \mathbb{R}$, and $\f(x)=\E_{f\ \gets \mathcal{D}}[f(x)]$.
Suppose there is an unbiased oracle $\mcal{S}$ access to the gradient of the function $\f(x)$ and $\forall x\in \mathbb{R}^n$, $\E_{s \gets \mcal{S}}[s(x)]=\nabla \f(x)$.
We consider the following stochastic ADMM Algorithm~\ref{alg:ADMM_oneOracle} in this section. Compared with Algorithm~\ref{alg:one_iteration}, in the $t$-th iteration, the algorithm samples an $s_t(x)$ from $\mcal{S}$ and updates $x_{t+1}$ using an operator $\mcal{F}^{\mathcal{S}}(x, y,{\lambda})$, which can be considered similar to $\mcal{F}^f(x,y,\lambda)$ and will be defined clearly later.
Some results on the convergence rate of Algorithm~\ref{alg:ADMM_oneOracle} have been given in \cite{DBLP:conf/icml/WangB12, DBLP:conf/icml/OuyangHTG13}. For the completeness of the section, we provide a brief convergence result of this algorithm and then extend the result to the following situation: in the $t$-th iteration, the algorithm first updates $x_{t+1}$ using $\mcal{F}^{\mathcal{S}}(x_{t}, y_{t},{\lambda}_{t})$ and then add $\mcal{N}(0,\sigma^2)$ noise to $x_{t+1}$. This captures the processes in Algorithm~\ref{alg:m1} and Algorithm~\ref{alg:m2} well.
}

In this section, we focus on privacy and utility trade-off of private ADMM
that we denote as Algorithm~\ref{alg:ADMM_oneOracle}, in which each iteration
uses a randomized oracle for the gradient.
This can be viewed as a randomized version of the original ADMM iteration shown in Algorithm~\ref{alg:one_iteration}.
We will follow the proof of \cite{DBLP:conf/icml/WangB12} to show that Algorithm~\ref{alg:ADMM_oneOracle} also achieves $O(\frac{1}{\sqrt{T}})$ convergence.

We use the same notations as in Algorithm~\ref{alg:one_iteration},
with additional definitions as follows.
Let $\mathcal{D}$ be a distribution on $f: \mathbb{R}^n \to \mathbb{R}$ and write $\f(x)=\E_{f\gets \mathcal{D}}[f(x)]$.
We assume there is an unbiased oracle $\mcal{S}$ access to the gradient of the function $\f(x)$ such that $\forall x\in \mathbb{R}^n$, $\E_{s \gets \mcal{S}}[s(x)]=\nabla \f(x)$.

\noindent \textbf{User Model.}  We assume that the user
at iteration~$t$ samples its private function~$f_t$ from the
distribution $\mathcal{D}$ independently.
This assumption is needed to ensure that
the private data from all users can reflect the
the function $\f(x)=\E_{f\gets \mathcal{D}}[f(x)]$ in the objective function.

\noindent \emph{Oracle Definition.}
Recall that $\sigma^2$ is the variance of the noise
used in Corollary~\ref{cor:first_user} .
We pick some $\rho > 0$ such that $\sigma = \eta \rho$.
Consequentially,
at iteration~$t$,
we define the randomized gradient oracle $\mcal{S}(x_t)$ as the following steps:

\begin{enumerate}

\item Sample $f_t$ from $\mathcal{D}$ and compute  $s_t := \nabla f_t(x_t)$.

\item Sample a vector $z_t$ from $\mcal{N}(0, \rho^2 \mathbb{I}_n)$.

\item Return $g_t := s_t + (\I + \eta \beta A^{\top} A) z_t$.

\end{enumerate}

The notation $\mcal{F}^{\mathcal{S}}(x_t, y_t,{\lambda}_t)$ means the following:

\begin{enumerate}
\item Access the oracle $\mcal{S}(x_t)$ to produce some sample $g_t$ (which is
supposed to approximate $\nabla \mathfrak{f}(x)$);

\item Return $(\I + \eta \beta A^{\top} A)^{-1} \{x_t - \eta \cdot [g_t + A^{\top} (\beta(B y_t - c) - \lambda_t )] \}$.
\end{enumerate}

\noindent \emph{Reparametrization.}
Observe that  $\mcal{F}^{\mathcal{S}}(x_t, y_t,{\lambda}_t)$
is equivalent to first sample $f_t$ from $\mcal{D}$ and
$N$ from $\mcal{N}(0,\eta^2\rho^2 \mathbb{I}_n)$
and then return
$\mcal{F}^{\nabla f_t}(x_t, y_t,{\lambda}_t) + N$.

\ignore{
We first introduce the notation $\mcal{F}^{\mathcal{S}}(x, y,{\lambda})$ as an operation consisting the following two steps:

\begin{enumerate}
\item Sample an estimate $s(x)$ of the gradient $\nabla \f(x)$ from $\mcal{S}$.
\item Based on this  $s(x)$, return
$$
(\I + \eta \beta A^{\top} A)^{-1}
\{x - \eta \cdot [s(x) + A^{\top} (\beta(B y - c) - \lambda )] \}.
$$
\end{enumerate}
}

\begin{algorithm}
	\caption{ADMM with Oracle}
	\label{alg:ADMM_oneOracle}
	
	\KwIn{Parameters $\beta$ and $\eta$ to define $\mcal{L}$ and $\mcal{H}$, and number $T$ of iterations.}
	
	Initialize arbitrary $(x_0, y_0) \in \R^n \times \R^\ell $, $\lambda_0=0\in \mathbb{R}^m$,  an unbiased orable $\mcal{S}$ for $\nabla \f(x)$
    such that $\forall x\in \mathbb{R}^n$, $\E_{s\gets \mcal{S}}[s(x)]=\nabla \f(x)$.
	
	\For{$t=0, \ldots, T-1$}{
		
        $x_{t+1} \gets \mcal{F}^{\mcal{S}}(x_{t}, y_{t},{\lambda}_{t})$
		\label{ln:step1_oracle}

        $y_{t+1} \gets \mcal{G}({\lambda}_{t} - \beta A x_{t+1} )$		

        $\lambda_{t+1} \gets {\lambda}_{t} - \beta (A x_{t+1} + B y_{t+1} - c)$
		\label{ln:step3}

		\label{ln:step2}
		
	}	
	\KwRet $(x_T, y_T, \lambda_T)$

\end{algorithm}

The following theorem provides a $O(\frac{1}{\sqrt{T}})$ convergence rate for Algorithm~\ref{alg:ADMM_oneOracle}.
We need some upperbound on
$\E_{f \gets \mcal{D}}[\|\nabla f(x)\|^2]$
that depends on the given distribution $\mcal{D}$.

\begin{theorem}[Convergence Rate of Algorithm~\ref{alg:ADMM_oneOracle}]
\label{thm:utility1}

\ignore{
Let the sequences $\{(x_t,y_t,\lambda_t)\}_{t\in [T]}$ be generated by Algorithm~\ref{alg:ADMM_oneOracle}.
Let $\overline{x}_{[a,b]}=\frac{1}{b-a+1} \sum_{i=a}^{b} x_{i}$, $\overline{y}_{[a,b]}=\frac{1}{b-a+1} \sum_{i=a}^{b} y_{i}$.
Assume that $\exists G\in \mathbb{R}$ such that $\forall x\in \mathbb{R}^n$, $\forall s\in \mcal{S}$, $\E[\|s(x)\|^2]\leq G^2$.
Let $\eta=\frac{1}{\sqrt{T}}$, for any $(x^*,y^*)$ such that $Ax^*+By^*=c$, we have
}
Suppose there exists $G \ge 0$ such that $\forall x\in \mathbb{R}^n$, $\E_{f \gets \mcal{D}}[\|\nabla f(x)\|^2]\leq G^2$.
If $\eta=\frac{1}{\sqrt{T}}$ in Algorithm~\ref{alg:ADMM_oneOracle},
and let the sequences $\{(x_t,y_t,\lambda_t)\}_{t\in [T]}$ be generated by Algorithm~\ref{alg:ADMM_oneOracle},
and $\overline{x}_{[a,b]}:=\frac{1}{b-a+1} \sum_{i=a}^{b} x_{i}$, $\overline{y}_{[a,b]}:=\frac{1}{b-a+1} \sum_{i=a}^{b} y_{i}$,
then for any $(x^*,y^*)$ such that $Ax^*+By^*=c$, we have

\begin{align}
&\E\left[ \f(\overline{x}_{[0,T-1]})-\f(x^*)+g(\overline{y}_{[1,T]})-g(y^*)+\frac{\beta}{2}\|A\overline{x}_{[1,T]}+B\overline{y}_{[1,T]}-c\|^2\right]
\notag
\\
&\leq \frac{\beta \|By_0-By^*\|^2}{2T}+\frac{\|x_0-x^*\|^2}{2\sqrt{T}}
+\frac{G^2+n \rho^2}{2\sqrt{T}}+\frac{n\beta^2 \rho^2  \|A\|^4}{2T^{1.5}}+\frac{n\beta \rho^2 \|A\|^2}{T}.
\end{align}
\end{theorem}

\begin{lemma}
\label{lemma:utility_onestep}

\ignore{
Let the sequences $\{(x_t,y_t,\lambda_t)\}_{t\in [T]}$ be generated by Algorithm~\ref{alg:ADMM_oneOracle}.
Suppose in the $t$-th iteration and in line~\ref{ln:step1_oracle}, the algorithm samples $s_t(x)$ from $\mcal{S}$ and add $N_t$ gaussian noise, then
}
Using the same assumptions in Theorem~\ref{thm:utility1},
suppose $f_t$ is sampled from $\mcal{D}$ and
$s_t(x) = \nabla f_t(x)$ and $N_t$ is sampled from
$\mcal{N}(0,\eta^2\rho^2 \mathbb{I}_n)$.
Then, we have:


\begin{align}
\f(x_t)-\f(x^*)+g(y_{t+1})-g(y^*)\leq &\frac{1}{2\beta}(\|\lambda_{t}\|^2-\|\lambda_{t+1}^2\|)-\frac{\beta}{2}\|Ax_{t+1}+By_{t}-c\|^2
\notag
\\
&+\frac{\beta}{2}[\|By^*-By_{t}\|^2-\|By^*-By_{t+1}\|^2]
\notag
\\
&+\frac{\eta}{2}\|s_t(x_t)+(\beta A^\top A +\frac{1}{\eta})N_t\|^2+\frac{1}{2\eta}(\|x_{t}-x^*\|^2-\|x_{t+1}-x^*\|^2)
\notag
\\
&+\langle \nabla \f(x_t)-(s_t(x_t)+(\beta A^\top A +\frac{1}{\eta})N_t),x_{t}-x^*\rangle.
\notag
\end{align}
\end{lemma}

\begin{proof}
By Lemma~\ref{lemma:step1}, adding noise $N_t$ to $x_{t+1}$ is equivalent to adding noise
$\frac{1}{\eta} \cdot (\eta\beta A^\top A + \I) N_t$ to $s_t(x_t)$, we have
$$
s_t(x_t)+(\beta A^\top A +\frac{1}{\eta})N_t=A^\top (\lambda_t-\beta(Ax_{t+1}+By_{t}-c))-\frac{1}{\eta}(x_{t+1}-x_{t})
=A^\top (\lambda_{t+1}+\beta(By_{t+1}-By_{t}))+\frac{1}{\eta}(x_{t}-x_{t+1}).
$$

Also, Lemma~\ref{lemma:step2} gives that
$$
B^\top(\lambda_t-(A x_{t+1}+By_{t+1}-c))=B^\top \lambda_{t+1}\in \partial g(y_{t+1}).
$$

By convexity of $\f(x)$, we have

\begin{align}
\f(x_t)-\f(x^*)+g(y_{t+1})-g(y^*)
\leq &\langle s_t(x_t)+(\beta A^\top A +\frac{1}{\eta})N_t,x_{t+1}-x^*\rangle
+\langle\lambda_{t+1},B(y_{t+1}-y^*)\rangle
\notag
\\
&+\langle s_t(x_t)+(\beta A^\top A +\frac{1}{\eta})N_t,x_{t}-x_{t+1}\rangle
\notag
\\
&+\langle \nabla \f(x_t)-(s_t(x_t)+(\beta A^\top A
+\frac{1}{\eta})N_t),x_{t}-x^*\rangle.
\label{eq:utility1}
\end{align}

For the first term at the right hand side of (\ref{eq:utility1}), we have
\begin{align}
&\langle s_t(x_t)+(\beta A^\top A +\frac{1}{\eta})N_t,x_{t+1}-x^*\rangle=
\langle A^\top (\lambda_{t+1}+\beta(By_{t+1}-By_{t})),x_{t+1}-x^*\rangle +\frac{1}{\eta}\langle x_t-x_{t+1},x_{t+1}-x^*\rangle
\notag
\\
=&\langle \lambda_{t+1}+\beta (By_{t+1}-By_{t}),A(x_{t+1}-x^*)\rangle+\frac{1}{\eta}\langle x_t-x_{t+1},x_{t+1}-x^*\rangle
\notag
\\
=&\langle \lambda_{t+1},Ax_{t+1}-c+By^*\rangle+\beta\langle B(y_{t+1}-y_t),Ax_{t+1}-c+By^*\rangle+\frac{1}{\eta}\langle x_t-x_{t+1},x_{t+1}-x^*\rangle
\notag
\\
=&\langle \lambda_{t+1},Ax_{t+1}-c+By^*\rangle+\frac{\beta}{2}\|Ax_{t+1}+By_{t+1}-c\|^2-\frac{\beta}{2}\|Ax_{t+1}+By_{t}-c\|^2
\notag
\\
&+\frac{\beta}{2}[\|By^*-By_{t}\|^2-\|By^*-By_{t+1}\|^2]+\frac{1}{\eta}\langle x_t-x_{t+1},x_{t+1}-x^*\rangle.
\notag
\end{align}

So the first two terms at the right hand side of (\ref{eq:utility1}) can be written as
\begin{align}
& \langle s_t(x_t)+(\beta A^\top A +\frac{1}{\eta})N_t,x_{t+1}-x^*\rangle+\langle\lambda_{t+1},B(y_{t+1}-y^*)\rangle \notag\\
=&\langle \lambda_{t+1},Ax_{t+1}+By_{t+1}-c\rangle+\frac{\beta}{2}\|Ax_{t+1}+By_{t+1}-c\|^2-\frac{\beta}{2}\|Ax_{t+1}+By_{t}-c\|^2
\notag
\\
&+\frac{\beta}{2}[\|By^*-By_{t}\|^2-\|By^*-By_{t+1}\|^2]+\frac{1}{\eta}\langle x_t-x_{t+1},x_{t+1}-x^*\rangle
\notag
\\
=&\frac{1}{\beta}\langle \lambda_{t+1},\lambda_{t}-\lambda_{t+1}\rangle+\frac{1}{2\beta}\|\lambda_{t}-\lambda_{t+1}\|^2-\frac{\beta}{2}\|Ax_{t+1}+By_{t}-c\|^2
\notag
\\
&+\frac{\beta}{2}[\|By^*-By_{t}\|^2-\|By^*-By_{t+1}\|^2]+\frac{1}{\eta}\langle x_t-x_{t+1},x_{t+1}-x^*\rangle
\notag
\\
=&\frac{1}{2\beta}(\|\lambda_{t}\|^2-\|\lambda_{t+1}\|^2)-\frac{\beta}{2}\|Ax_{t+1}+By_{t}-c\|^2
\notag
\\
&+\frac{\beta}{2}[\|By^*-By_{t}\|^2-\|By^*-By_{t+1}\|^2]+\frac{1}{\eta}\langle x_t-x_{t+1},x_{t+1}-x^*\rangle.
\end{align}

For the third term at the right hand side of (\ref{eq:utility1}), by Fenchel-Young's inequality,
$$
\langle s_t(x_t)+(\beta A^\top A +\frac{1}{\eta})N_t,x_{t}-x_{t+1}\rangle\leq \frac{\eta}{2}\| s_t(x_t)+(\beta A^\top A +\frac{1}{\eta})N_t\|^2+\frac{1}{2\eta}\|x_t-x_{t+1}\|^2.
$$

Finally, by equality
$$
\frac{1}{2\eta}\|x_t-x_{t+1}\|^2+\frac{1}{\eta}\langle x_t-x_{t+1},x_{t+1}-x^*\rangle=\frac{1}{2\eta}(\|x_{t}-x^*\|^2-\|x_{t+1}-x^*\|^2),
$$
formula (\ref{eq:utility1}) can be upperbounded by
\begin{align}
\f(x_t)-\f(x^*)+g(y_{t+1})-g(y^*)\leq &\frac{1}{2\beta}(\|\lambda_{t}\|^2-\|\lambda_{t+1}^2\|)-\frac{\beta}{2}\|Ax_{t+1}+By_{t}-c\|^2
\notag
\\
&+\frac{\beta}{2}[\|By^*-By_{t}\|^2-\|By^*-By_{t+1}\|^2]+\frac{\eta}{2}\| s_t(x_t)+(\beta A^\top A +\frac{1}{\eta})N_t\|^2
\notag
\\
&+\frac{1}{2\eta}(\|x_{t}-x^*\|^2-\|x_{t+1}-x^*\|^2)
\notag
\\
&+\langle \nabla \f(x_t)-(s_t(x_t)+(\beta A^\top A +\frac{1}{\eta})N_t),x_{t}-x^*\rangle.
\notag
\end{align}
\end{proof}

\begin{lemma}
\label{lemma:utility_regulizer}
Let the sequences $\{(x_t,y_t,\lambda_t)\}_{t\in [T]}$ be generated by Algorithm~\ref{alg:ADMM_oneOracle}.
Then
$$
\|Ax_{t+1}+By_{t+1}-c\|^2\leq \|Ax_{t+1}+By_{t}-c\|^2
$$
\end{lemma}

\begin{proof}
By the optimal condition of $y_{t+1}$ in Lemma~\ref{lemma:step2},
$$
B^\top(\lambda_t-\beta (Ax_{t+1}+By_{t+1}-c))=B^\top \lambda_{t+1}\in \partial g(y_{t+1}).
$$

By convexity of $g(y)$,
\begin{align}
& g(y_{t+1})-g(y_{t})\leq \langle B^\top\lambda_{t+1},y_{t+1}-y_{t}\rangle
\notag
\\
& g(y_{t})-g(y_{t+1})\leq \langle B^\top\lambda_{t},y_{t}-y_{t+1}\rangle
\notag
\end{align}

Summing the above formulas together we have
\begin{align}
0 &\leq \langle \lambda_{t}-\lambda_{t+1},By_{t}-By_{t+1}\rangle
\notag
\\
&=\beta\langle Ax_{t+1}+By_{t+1}-c,By_{t}-By_{t+1}\rangle
\notag
\\
&=\frac{\beta}{2}(\|Ax_{t+1}+By_{t}-c\|^2-\|By_{t}-By_{t+1}\|^2-\|Ax_{t+1}+By_{t+1}-c\|^2)
\notag
\\
&\leq \frac{\beta}{2}(\|Ax_{t+1}+By_{t}-c\|^2-\|Ax_{t+1}+By_{t+1}-c\|^2).
\notag
\end{align}

\end{proof}

\begin{proof}[Proof of Theorem~\ref{thm:utility1}]
By Lemma~\ref{lemma:utility_onestep} and \ref{lemma:utility_regulizer}, in each iteration we have
\begin{align}
&\f(x_t)-\f(x^*)+g(y_{t+1})-g(y^*)+\frac{\beta}{2}\|Ax_{t+1}+By_{t+1}-c\|^2
\notag
\\
&\leq \frac{1}{2\beta}(\|\lambda_{t}\|^2-\|\lambda_{t+1}^2\|)
+\frac{\beta}{2}[\|By^*-By_{t}\|^2-\|By^*-By_{t+1}\|^2]+\frac{\eta}{2}\| s_t(x_t)+(\beta A^\top A +\frac{1}{\eta})N_t\|^2
\notag
\\
&+\frac{1}{2\eta}(\|x_{t}-x^*\|^2-\|x_{t+1}-x^*\|^2)+\langle \nabla \f(x_t)-(s_t(x_t)+(\beta A^\top A +\frac{1}{\eta})N_t),x_{t}-x^*\rangle.
\notag
\end{align}

Summing over $t$ from 0 to $T-1$, we have

\begin{align}
\sum_{t=0}^{T-1} &\left[ \f(x_t)-\f(x^*)+g(y_{t+1})-g(y^*)+\frac{\beta}{2}\|Ax_{t+1}+By_{t+1}-c\|^2\right]
\notag
\\
\leq &\frac{\beta \|By_0-By^*\|^2}{2}+\frac{\|x_0-x^*\|^2}{2\eta}
+\frac{\eta\sum_{t=0}^{T-1}\|s_t(x_t)+(\beta A^\top A +\frac{1}{\eta})N_t\|^2}{2}
\notag
\\
&+\sum_{t=0}^{T-1}\langle \nabla \f(x_t)-(s_t(x_t)+(\beta A^\top A +\frac{1}{\eta})N_t),x_{t}-x^*\rangle.
\notag
\end{align}

Dividing both sides by $T$, and applying the Jensen's inequality,

\begin{align}
& \f(\overline{x}_{[0,T-1]})-\f(x^*)+g(\overline{y}_{[1,T]})-g(y^*)+\frac{\beta}{2}\|A\overline{x}_{[1,T]}+B\overline{y}_{[1,T]}-c\|^2
\notag
\\
\leq &\frac{\beta \|By_0-By^*\|^2}{2T}+\frac{\|x_0-x^*\|^2}{2\eta T}
+\frac{\eta\sum_{t=0}^{T-1}\|s_t(x_t)+(\beta A^\top A +\frac{1}{\eta})N_t\|^2}{2T}
\notag
\\
&+\frac{1}{T}\sum_{t=0}^{T-1}\langle \nabla \f(x_t)-(s_t(x_t)+(\beta A^\top A +\frac{1}{\eta})N_t),x_{t}-x^*\rangle.
\notag
\end{align}

By the assumption of the oracle $\mcal{S}$, $\E[\| s_t(x)\|^2]\leq G^2$ for all $t$. And for all $t$ we have

\begin{align}
\E &[\|s_t(x_t)+(\beta A^\top A +\frac{1}{\eta})N_t\|^2]
= \E[\|s_t(x_t)\|^2] +\frac{1}{\eta^2} \E[N_t^\top(\eta\beta A^\top A+ \mathbb{I})^2 N_t]
\notag
\\
 \leq&
 G^2+ \frac{n \eta^2 \rho^2}{\eta^2}( \eta \beta\|A\|^2+1)^2
 =G^2+ n \rho^2( \eta \beta\|A\|^2+1)^2.
\end{align}

Setting $\eta=\frac{1}{\sqrt{T}}$ and taking expectation, we have
\begin{align}
&\E\left[ \f(\overline{x}_{[0,T-1]})-\f(x^*)+g(\overline{y}_{[1,T]})-g(y^*)+\frac{\beta}{2}\|A\overline{x}_{[1,T]}+B\overline{y}_{[1,T]}-c\|^2\right]
\notag
\\
&\leq \frac{\beta \|By_0-By^*\|^2}{2T}+\frac{\|x_0-x^*\|^2}{2\sqrt{T}}
+\frac{G^2+n \rho^2}{2\sqrt{T}}+\frac{n\beta^2 \rho^2  \|A\|^4}{2T^{1.5}}+\frac{n\beta \rho^2 \|A\|^2}{T}.
\end{align}

\end{proof}

\ignore{
Now we're ready to give the privacy and utility trade-off of the following situation.
Suppose $\mcal{D}$ is a distribution on a class of functions $f:\mathbb{R}^n\to \mathbb{R}$,
and $\mcal{S}$ is an unbiased oracle access to the gradient of function $\f(x)$.
We consider an operation $\widehat{\mcal{F}}^\mcal{S}(x,y,\lambda)$ that returns $\mcal{F}^\mcal{S}(x,y,\lambda)+\mcal{N}(0,\eta^2\rho^2)$.
}

\ignore{
The ADMM with oracle is provided in Algorithm~\ref{alg:ADMM_oneOracle}. Note that the iteration fully captures the behavior of Mechanism $\M_2$. The difference lies in the fact that, in $\M_2$, $\mcal{N}(0,\sigma^2 \mathbb{I}_n)$ noises are added, whereas here we write the noises as $N \sim \mcal{N}(0,\eta^2\rho^2 \mathbb{I}_n)$.
}

In the following corollary, we provide the privacy guarantee for Algorithm~\ref{alg:ADMM_oneOracle}. The proof follows the same reasoning as presented in Corollary~\ref{cor:first_user}, with the modification of setting $\sigma^2$ in Corollary~\ref{cor:first_user} to $\eta^2\rho^2$.

\begin{corollary}
Consider two scenarios for running Algorithm~\ref{alg:ADMM_oneOracle}
$2T+1$ iterations,
where the only difference
in the two scenarios
is that the corresponding sampled estimate of the gradient in the first iteration
can be different neighboring functions $s_0 \sim_\Delta s'_0$ (as
in Definition~\ref{defn:neighbor_functions}).
In other words, the initial $(x_0, \lambda_0)$
and the estimate of gradient $s_t$ sampled from the oracle $\mcal{S}$ in subsequent iterations $t \geq 1$ are identical
in the two scenarios.

Then, the solutions from the two scenarios
satisfy the following.

\begin{compactitem}

\item \emph{Local Privacy.} After the first iteration, we have

$$\Dz(x'_1 \| x_1) \leq \frac{ \Delta^2}{2 \rho^2}.$$

\item \emph{Privacy Amplification.}  After the final iteration $2T+1$, we have

$$\Dz(({x}_{2T+1}, \lambda_{2T+1}) \| ({x}'_{2T+1}, \lambda'_{2T+1})) \leq \frac{C}{T} \cdot\frac{ \Delta^2}{2 \rho^2},$$

where $C := \max\{2, \frac{3}{\beta \eta} \} \cdot (1+  \beta \eta \cdot \|A\|^2)$.
\end{compactitem}
\end{corollary}

\section{Experiments}
\label{sec:experiment}

In this section, we conduct experiments to explore how the utility is affected by contraction factor $\mathfrak{L}$ and variance of the noise $\sigma$. We implemented the experiments using Python. 
And to ensure reproducibility, we fixed the random seeds for all experiments. 
Our focus is on achieving accurate and reproducible results and do not prioritize running time. 
As such, the experiments can be run on any computer without concern for hardware specifications and will finish within a few minutes. The results will always be the same regardless of the computing environment.
The source code is available at: \url{https://github.com/KawaiiMengshi/Privacy-Amplification-by-Iteration-for-ADMM}.

\noindent\textbf{LASSO Problem.}
We consider a generalised LASSO problem, a.k.a. Elastic Net Regularization~\cite{zou2005regularization}. Given a dataset $U = \{(a_i, b_i)\in \R^n \times \R\}_{i\in [N]}$ and parameters $c_1, c_2$, the generalised LASSO aims at solving the following regression problem:

\begin{align*}
\min_{x\in \R^n} \quad& \frac{1}{N}\sum_{i\in[N]} (\langle a_i,x\rangle-b_i)^2+c_1\|x\|_1+c_2\|x\|_2^2
\end{align*}

Let $f_{U_i}(x) = (\langle a_i,x\rangle-b_i)^2$, $\f(x)=\frac{1}{N}\sum_{i\in[N]} f_{U_i}(x)$, and $g(y)=c_1\|y\|_1+c_2\|y\|_2^2$, the problem could be decomposed into the following program that fits into the ADMM framework:

\begin{subequations}
  \label{prob:lasso_program}
\begin{align}
  \label{prob:Lasso_obj}
  \min \quad& \f(x)+g(y)
  \\ \text{s.t.} \quad &x-y=0.
\end{align}
\end{subequations}

\noindent\textbf{Data Generation.}
Using ADMM to solve Elastic Net Regularizetion problem has been considered in previous work \cite{DBLP:journals/siamis/GoldsteinOSB14}. 
They use one example proposed by \cite{zou2005regularization} to test the behaviour of the ADMM.
So here we adopt a unified way similar to it for generating the synthetic data:
\ignore{
We assume that there is some database $D=(a_i,b_i):i=1,\dots,N$, where each pair $(a_i,b_i)$ consists of a vector $a_i$ and a scalar $b_i$. The goal is to find some vector $w$ such that $\langle a_i, w \rangle$ can predict $b_i$. As in
\cite{DBLP:conf/icml/CyffersBB23}, we will first consider a synthetic database $D$.

We consider a general form of the Lasso problem, or also known as the Elastic Net Regularization\cite{zou2005regularization}, which correspond to the following regression problem:

We formulate the above optimization problem into the following ADMM,

Suppose the number of the data points in $D$ is $N$ and each $a_i$ is in $m$-dimensional space.
The data is generated as the following step:

\begin{enumerate}
  \item For each $a_i$, let $a_{i,j}=51z_i$ for $j=1,2,\dots,[N/5]$
  and $a_{i,j}=z_i$ for $j=[N/5]+1,[N/5]+2,\dots,N$, where each $z_i$ is sampled from the standard $\mcal{N}(0,1)$ gaussian distribution.
  \item Scaling all $a_i$ such that each of them lies on the $m$-dimensional sphere.
  \item Let $x\in \mathbb{R}^m$, $x_j=3$ for $j=1,2,\dots,[N/5]$, and $x_j=0$ for other entries.
  \item Let $b_i=\langle a_i,x\rangle + \mathfrak{n}_{b_i}$, where
  $\mathfrak{n}_{b_i}$ is sampled from $\mcal{N}(0, \sigma_b^2)$ gaussian distribution.
  \end{enumerate}

}

\begin{enumerate}

  \item For each $a_i \in \R^n$, we first set $a'_i \in \R^n$ such that
    $$a'_{ij} = \begin{cases}
    50z_{ij}, & \text{if } j \le \lfloor n/5 \rfloor, \\
    z_{ij}, & \text{otherwise}
    \end{cases}$$
    where $z_{ij}$ are sampled independently from the standard gaussian distribution $\mcal{N}(0,1)$.
    Further, we let $a_i = \frac{\sqrt{\mu}}{\|a'_i\|_2} a'_i$, where $\mu$ is a parameter such that $f_{\mcal{D}_i}$ and $\f$ are $2\mu$-strongly convex.

  \item Let $x'\in \R^n$ be such that $x'_j=3$ for $j \le \lfloor n/5 \rfloor$ and $0$ otherwise. We set $b_i=\langle a_i, x'\rangle + \mathfrak{n}_{b_i}$ where $\mathfrak{n}_{b_i} \sim \mcal{N}(0, \sigma_b^2)$ independently.
\end{enumerate}

This generation is motivated by that, suppose we set $g(y)=0$, the task of the program will be to reveal the secret $x'$ when $\sigma_b^2$ is relatively small compared to $\mu$.

In the experiments we set $n=64$, $N=1000$ and $\sigma_b=0.01$.

\ignore{
Next we illustrate each step of running Mechanism~$M_2$.
Let $f_i(x)=(\langle a_i,x\rangle -b_i)^2$.
Each time one uniformly samples an $f_t$ from $\{f_i(x)\}_{i\in [N]}$.
Suppose $f_t(x)=(\langle a,x\rangle -b)^2$ for some $(a,b)$, the mechanism does the following calculations:
$$
\nabla f_t(x_t)=2(\langle a,x_t\rangle-b)a,
$$

$$
x_{t+1}=\frac{1}{1+\eta\beta}\left[x_t-\eta(\nabla f_{t}(x_t)+\beta y_t -\lambda_t)\right].
$$

$y_t$ and $\lambda_t$ are updated based on the Mechanism~$M_2$.
}

\noindent\textbf{Behaviour of the Algorithm.}
We consider a random sampling setting.
That is, given a designed total number of steps $T$,
the sequence $(f_1, f_2,\cdots, f_T)$ will be constructed by uniformly sampling $T$ functions from $\{f_{U_i}\}_{i\in [N]}$ independently.

Note that mechanisms~$\M_1$ and $\M_2$ are proposed for simplifying the analysis of privacy amplification by iteration,
we focus on the algorithm based on a sequential running of Algorithm~\ref{alg:one_iteration} for $T$ times,
i.e., given initial $(x_0, \lambda_0)$, at step $t\in\{0, 1, \cdots, T-1\}$, we let $(x_t, \lambda_t)$ be the input of Algorithm~\ref{alg:one_iteration},
which outputs $(x_{t+1}, \lambda_{t+1})$ as the input of the next iteration. For the initial value, each dimension of $x_0$ is set to be $3$, and $\lambda_0$ is set to be $0$.

Further, in each step, we assume there is an additional gaussian noise added into $x_{t+1}$ in Line~\ref{ln:x} of Algorithm~\ref{alg:one_iteration}, i.e.
$$
x_{t+1} \gets \mcal{F}^{\nabla f_{t+1}}(x_{t}, y_{t},{\lambda}_{t+1}) + z_{t+1}
$$
where $z_{t+1}\sim \mcal{N}(0, \sigma^2\I_n)$.

In the end, we evaluate $\E[\f(x_{t})+g(y_{t})]-[\f(x^*)+g(y^*)]$ where $(x^*,y^*)$ represent the optimal solution of (\ref{prob:lasso_program}). 
Due to the randomized setting, we denote the optimality gap in the $t$-th iteration as $\f(x_{t})+g(y_{t})-[\f(x^*)+g(y^*)]$,
and will run the same experiment for $100$ times and use the \textit{average optimality gap} as the estimate of the expectation.

\ignore{

Our experiments will mainly consider the effects of two varying parameters:

\begin{enumerate}
  \item the contraction parameter $\mathfrak{L}$,
  \item the standard deviation $\sigma$ of the noise added in each iteration.
\end{enumerate}

}

\ignore{
To investigate the performances under different values of $\mathfrak{L}$, we can change the parameters $c_1$, $c_2$, and the radius of $\{a_i\}_{i\in [N]
}$.
Once we have fixed all parameters, the value of $\mathfrak{L}$ can be explicitly derived from the proof of Lemma~\ref{lemma:contractive_mapping}.
The parameter can be calculated as follows.

\begin{align}
 &\nu=\max_i (2a_i^Ta_i) \\ &\mu = \min_i (2a_i^Ta_i)\\ &\mu_g=2c_2 \\
 &\eta \in \left( \max \left\{\frac{4}{\nu+ {\mu} +\sqrt{\left(\nu+ {\mu} \right)^2+8 \nu  {\mu} }}, \frac{2}{\nu+ {\mu} }-\frac{2 \mu_g}{\beta^2\|A^{\top}B\|^2}\right\}, \frac{2}{\nu+  {\mu} }\right)
\end{align}

And $\mathfrak{L}$ is obtained by
\begin{align}
 &R := \left(1-\frac{2 \eta \nu {\mu} }{\nu+ {\mu} }\right)+\frac{1}{\eta}\left(\frac{2}{\nu+ {\mu} }-\eta\right) \\
 &S :=\frac{\eta}{\beta} \\
 &Q := \frac{\eta}{\beta}+\frac{\eta}{4}\left(\frac{2}{\nu+ {\mu} }-\eta\right) \\
 &P := 1-\frac{1}{\eta}\left(\frac{2}{\nu+ {\mu} }-\eta\right)
 \\
 &\mathfrak{L}:=\max \{\frac{R}{P},\frac{S}{Q}\} .
\end{align}
}

\noindent\textbf{Utility under Different Choices of $\mathfrak{L}$'s.}
We first show how the contraction factor $\mathfrak{L}$ will affect the utility of the algorithm.

Note that different $\mathfrak{L}$'s are resulted from a completely different choice of $\f$ and $g$,
and that $\f$ is $2\mu$-smooth $2\mu$-strongly convex and $g$ is $2c_2$-strongly convex.
With $\nu=\mu, \mu_g=2c_2$, $\mathfrak{L}$ could be computed using Lemma~\ref{lemma:contractive_mapping}.
For parameters, $\mu, \beta, c_2, c_1$ and $\mathfrak{L}$ in the experiments are listed in Table~\ref{table:L}.
Given the limitation of $\eta$ as shown in (\ref{ineq:eta_range_for_SC}), we let $\eta$ to be
$$
\eta = \frac{1}{2}\left[ \max \left\{\frac{4}{\nu+ {\mu} +\sqrt{\left(\nu+ {\mu} \right)^2+8 \nu  {\mu} }}, \frac{2}{\nu+ {\mu} }-\frac{2 \mu_g}{\beta^2\|A^{\top}B\|^2}\right\}+ \frac{2}{\nu+  {\mu} }\right].
$$

The result is shown in Figure~\ref{fig:utility_L}.
From the figure, we could see that, under the same noise level, larger $\mathfrak{L}$ leads to slower convergence.
However, we have to highlight that, due to the different choices of the problems, i.e., different $\f$'s, to generate different $\mathfrak{L}$'s,
the results may unavoidably be affected by the different choices of other parameters listed in Table~\ref{table:L}.

\noindent \textbf{Statistical Test.}
We present a comprehensive set of statistical tests to establish significant differences in the convergence rate across various levels of $\mathfrak{L}$. 
For each iteration $t$ at a fixed level of $\mathfrak{L}$, we have collected 100 samples of optimality gaps. 
Using a standard two-sample t-test, our analysis involves comparing the mean value of the optimality gap in iteration $t$ with the gap in the fifth iteration after it to determine if there is a statistically significant difference.
We record the smallest number of iterations where this gap fails to reach statistical significance, which serves as a measure of the convergence rate.
The confidence level is set at $0.95$, and the results are presented in Table~\ref{table:L}, labeled as ``convergence iterations". 
These results are consistent with the observations from the figure.

\ignore{
We list the parameters in the experiment in Table~\ref{table:L}.
For each setting, we focus on $\E[\f(x_{t})+g(y_{t})]$.
Due to the randomized setting, we run the same experiment for $100$ times and use the average score as the estimate of the expectation.
Figure~\ref{fig:utility_L} shows the curve under different $\mathfrak{L}$'s.
}

\begin{table}[htbp]
\caption{Parameters for different $\mathfrak{L}$}
\label{table:L}
\centering
\begin{tabular}{clccccc}
\hline
$\mu$ & $\beta$ & $c_2$ & $c_1$ & $\eta$ & $\mathfrak{L}$  &convergence iterations\\ \hline
0.25    & 0.9     & 0.1   & 0.01  & 1.95   & 0.95     &26      \\
0.09    & 0.5     & 0.1   & 0.01  & 4.81   & 0.91     &13      \\
0.0225   & 0.3     & 0.1   & 0.01  & 20.00  & 0.85    &7       \\
0.01    & 0.15    & 0.1   & 0.01  & 43.30  & 0.80     &6      \\ \hline
\end{tabular}
\end{table}

\ignore{
To investigate the performances under different value of $\sigma$,
we'll fix all other parameters as follows: $c_1=0.01,c_2=0.1, \text{Radius}=0.5, \beta=0.9, \mathfrak{L}=0.97$.
Similarly, we run the same experiment for $100$ times and calculate the average score.
Figure~\ref{fig:utility_sigma} shows the curve under different $\sigma$'s.
}

\begin{table}[]
  \caption{$p$ value for mean value test among different $\sigma$'s}
  \label{table:pvalue}
  \centering
  \begin{tabular}{c|ccccc}
  \hline
  \diagbox[dir=SE]{$\sigma$}{$\sigma$} & 0.05     & 0.1      & 0.2      & 0.5     & 0.7     \\ \hline
  0.05     & 1.00        & $4.20\times 10^{-27}$  & $1.87\times 10^{-71}$ & $4.57\times 10^{-78}$ & $2.30\times 10^{-79}$ \\
  0.1      & $4.20\times 10^{-27}$  & 1.00        & $3.29\times 10^{-62}$ & $4.12\times 10^{-77}$ & $1.11\times 10^{-78}$ \\
  0.2      & $1.87\times 10^{-71}$ & $3.29\times 10^{-62}$ & 1.00        & $6.70\times 10^{-74}$ & $2.22\times 10^{-76}$ \\
  0.5      & $4.57\times 10^{-78}$  & $4.12\times 10^{-77}$  & $6.70\times 10^{-74}$  & 1.00       & $1.83\times 10^{-36}$ \\
  0.7      & $2.30\times 10^{-79}$  & $1.11\times 10^{-78}$  & $2.22\times 10^{-76}$  & $1.83\times 10^{-36}$ & 1.00       \\ \hline
  \end{tabular}
  \end{table}

\noindent\textbf{Utility under Different Choices of $\sigma$'s.}
To investigate how $\sigma$ will affect the utility,
we fix $\mu=0.25, \beta=0.9, c_2=0.1, c_1=0.01$ and hence $\mathfrak{L}=0.95$.
Setting $\sigma=0.05,0.10,0.20,0.50,0.70$, the result is shown in Figure~\ref{fig:utility_sigma}.
The figure indicates that for small values of $\sigma\leq 0.1$, the average optimality gap is small, whereas larger values of $\sigma$, such as $\sigma=0.7$, result in a larger optimality gap. This observation can be explained by the fact that as the iteration number $t$ increases, the optimality gap gradually decreases and the pair $(x_t, y_t)$ attains a state of stability. However, the introduction of noise to $x_t$ hinders the convergence process, leading to its stagnation at a specific level and making further convergence unachievable. Consequently, as the value of $\sigma^2$ increases, the convergence of $(x_t, y_t)$ becomes progressively difficult, resulting in a larger optimality gap.

\noindent \textbf{Statistical Test.}
We also present the results of a statistical test to demonstrate that the optimality gaps vary across different levels of noise, using a standard two-sample t-test for mean values. 
We perform the test on the optimality gaps of the $100$th iteration in different noise settings. 
The corresponding $p$ values are reported in Table~\ref{table:pvalue}, where the $(i,j)$th entry represents the $p$ value of the test between the noise levels referred to by the header of the table. A $p$ value of $p\leq 0.05$ indicates that we can assert, with at least $95\%$ confidence, that the optimality gaps are indeed different.
This also confirms that noise level will affect the utility of the algorithm. 
\begin{figure}[htbp]
\centering
\subfloat[Convergence rate under different $\mathfrak{L}$'s]
{
\begin{minipage}{7cm}
\centering
\includegraphics[scale=0.5]{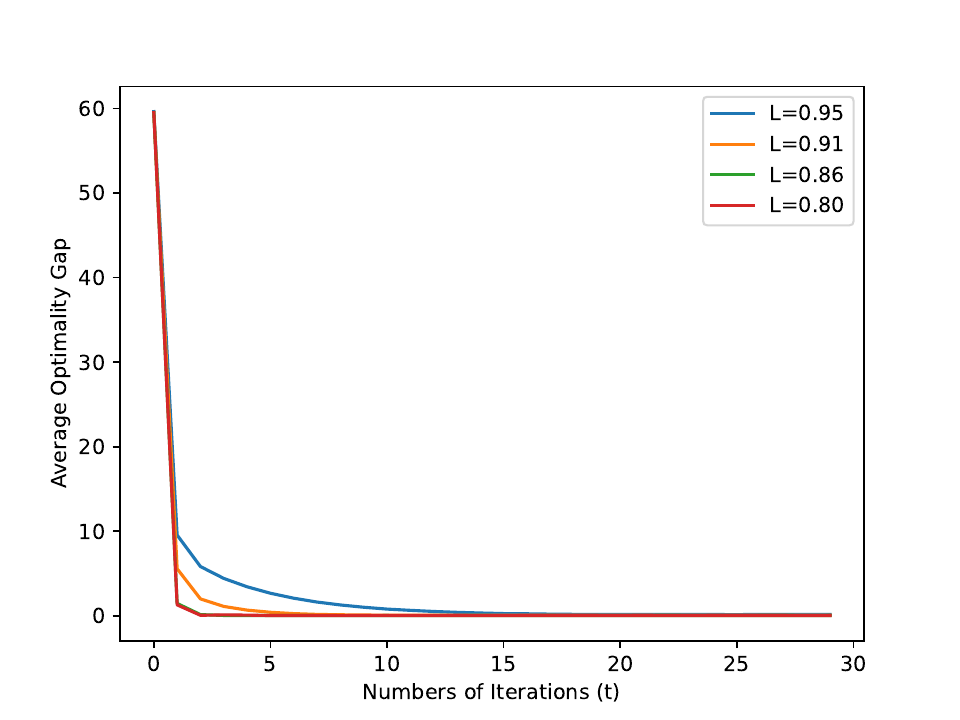}
\end{minipage}
\label{fig:utility_L}
}
\subfloat[Convergence rate under different $\sigma$'s]
{
\begin{minipage}{7cm}
\centering
\includegraphics[scale=0.5]{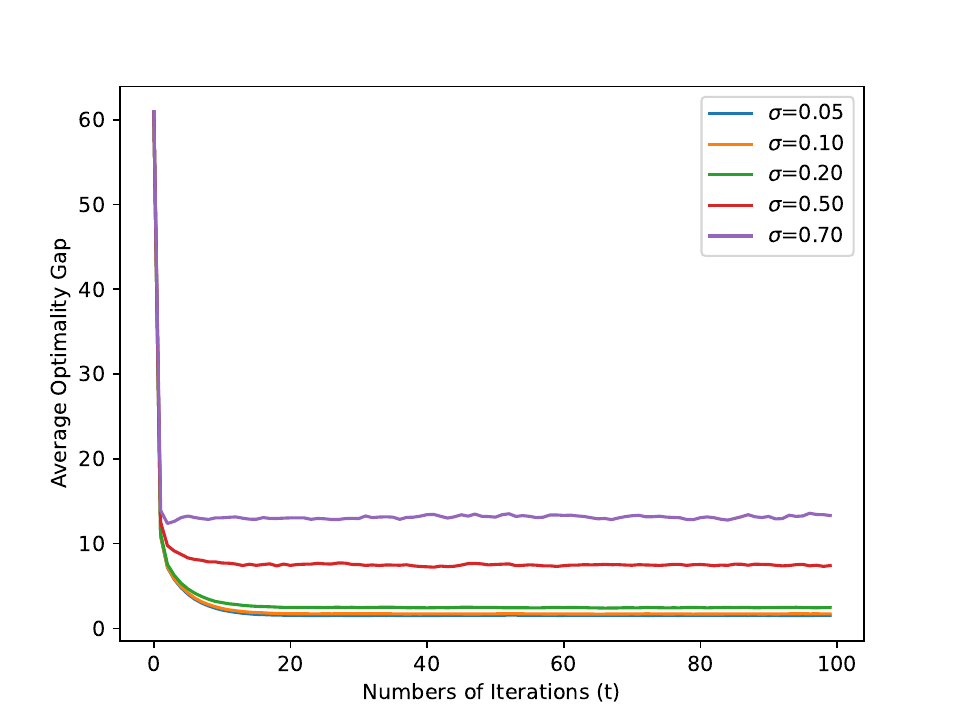}
\end{minipage}
\label{fig:utility_sigma}
}
\caption{}
\end{figure}

\section{Potential Improvements}
\label{sec:improvement}

We discuss whether the factor $\frac{1}{\beta \eta}$
in Theorem~\ref{th:ADMM_privacy}
may be improved.

\begin{compactitem}
\item \emph{Dependency on $\eta$.}  Observe that as $\eta$ tends to 0,
each ADMM iteration in Algorithm~\ref{alg:one_iteration}
essentially makes no update to the $x$ variable.  However, for some specific
case like $A = \I$, the $\lambda$
variable acts as a prefix sum for the $x$'s over all iterations.
Hence, given two different initial solutions,
after $T$ iterations, the magnitude of the difference in the $\lambda$ variables has a factor of~$T$.
Observe that the sum of $T$ independent copies of Gaussian noises only has variance proportional to~$T$, which means that the corresponding $\Dz$-divergence will have a factor
of $\Omega(T)$.  Hence, some inverse proportional dependence on $\eta$
in the divergence seems necessary if we insist on releasing the final~$\lambda$ variable.

\item \emph{Dependency on $\beta$.} Observe that as $\beta$
tends to 0, each ADMM iteration in Algorithm~\ref{alg:one_iteration}
degenerates exactly to the gradient descent update 
as in~\cite{DBLP:conf/focs/FeldmanMTT18}, which has no dependence on any~$\beta$ parameter.
This apparent paradox comes from the intermediate variable~$w$ in the first step of
our composition proof in Lemma~\ref{lemma:m1_privacy}, which has a positive divergence
even as $\beta$ tends to 0.

The best analogy we have is the following situation.
Consider two vectors $w$ and $w'$ such that $\|w - w'\| = 1$.
Let $\beta > 0$ and consider Gaussian noise $N$ sampled from
$\mcal{N}(0, \I)$ to define $\widetilde{w}_\beta = \beta w + \beta N$
and $\widetilde{w}' = \beta w' + \beta N$.
Then, as $\beta$ tends to 0, both random vectors
$\widetilde{w}_\beta$ and $\widetilde{w}'_\beta$
tend to 0, but their divergence $\Dz(\widetilde{w}_\beta \| \widetilde{w}'_\beta)$
stays 1, as long as $\beta > 0$.  However, observe that 
the divergence of the two limiting vectors (which are both 0) is 0.
Hence, the ratio of the limit of the divergence
to the divergence of the limits is $+\infty$,
which is somehow reflected in the $\frac{1}{\beta}$ factor in our bound.

This suggests that our analysis may not be tight for small~$\beta > 0$,
but it will be challenging to find an alternative method
to analyze the privacy guarantee of two consecutive noisy ADMM iterations,
or even more non-trivial to find a totally different framework
to analyze privacy amplification by iteration for ADMM.
\end{compactitem}

\section{Conclusion}
\label{sec:conclusion}

We have applied the coupling framework~\cite{BalleBGG19} 
to achieve privacy amplification by iteration for ADMM.
Specifically, we have recovered the  factor of $\frac{1}{T}$
(or $\frac{L^T}{T}$ for some $0 < L <1$ in
the strongly convex case)
in the $\Dz$-divergence as the number~$T$ of iteration increases.
We have performed experiments to evaluate the empirical
performance of our methods in
Section~\ref{sec:experiment}.
Discussion of whether the factor $\frac{1}{\beta \eta}$
in Theorem~\ref{th:ADMM_privacy}
may be improved is given in Section~\ref{sec:improvement}.

\bibliographystyle{alpha}
\bibliography{admm,ref}

\end{document}